\documentclass[twoside,11pt]{article}

\usepackage{helvet}
\usepackage{graphicx, subfigure}
\usepackage[top=1in,bottom=1in,left=1in,right=1in]{geometry}
\usepackage[sort&compress]{natbib} 
\usepackage{amsfonts, amsmath, amssymb, amsthm}

\usepackage{algorithm}
\usepackage{algorithmic}
\usepackage{xcolor,colortbl}

\usepackage{color}
\definecolor{darkred}{RGB}{100,0,0}
\definecolor{darkgreen}{RGB}{0,100,0}
\definecolor{darkblue}{RGB}{0,0,150}

\usepackage{hyperref}
\hypersetup{colorlinks=true, linkcolor=darkred, citecolor=darkgreen, urlcolor=darkblue}
\usepackage{url}
\urldef\eacurl\url{http://www.math.ucsd.edu/~eariasca/}

\newtheorem{thm}{Theorem}

\newtheorem{lem}{Lemma}

\def\beq{\begin{equation}}
\def\eeq{\end{equation}}
\def\beqn{\begin{eqnarray*}}
\def\eeqn{\end{eqnarray*}}
\def\bitem{\begin{itemize}}
\def\eitem{\end{itemize}}
\def\benum{\begin{enumerate}}
\def\eenum{\end{enumerate}}
\def\bmult{\begin{multline*}}
\def\emult{\end{multline*}}
\def\bcenter{\begin{center}}
\def\ecenter{\end{center}}

\newcommand{\thmref}[1]{Theorem~\ref{thm:#1}}

\newcommand{\lemref}[1]{Lemma~\ref{lem:#1}}
\newcommand{\secref}[1]{Section~\ref{sec:#1}}
\newcommand{\figref}[1]{Figure~\ref{fig:#1}}
\newcommand{\algref}[1]{Algorithm~\ref{alg:#1}}
\newcommand{\Clref}[1]{C_{\ref*{lem:#1}}}

\DeclareMathOperator{\diam}{diam}
\DeclareMathOperator{\dist}{dist}

\DeclareMathOperator{\asin}{asin}

\def\cN{\mathcal{N}}

\def\cR{\mathcal{R}}
\def\cS{\mathcal{S}}

\def\bA{{\boldsymbol A}}

\def\bC{{\boldsymbol C}}
\def\bD{{\boldsymbol D}}

\def\bM{{\boldsymbol M}}
\def\bN{{\boldsymbol N}}

\def\bP{{\boldsymbol P}}
\def\bQ{{\boldsymbol Q}}

\def\bW{{\boldsymbol W}}

\def\bZ{{\boldsymbol Z}}

\def\ba{{\boldsymbol a}}
\def\bb{{\boldsymbol b}}

\def\bs{{\boldsymbol s}}
\def\bt{{\boldsymbol t}}
\def\bu{{\boldsymbol u}}
\def\bv{{\boldsymbol v}}
\def\bw{{\boldsymbol w}}
\def\bx{{\boldsymbol x}}
\def\by{{\boldsymbol y}}
\def\bz{{\boldsymbol z}}

\newcommand{\bmu}{{\boldsymbol\mu}}
\newcommand{\bxi}{{\boldsymbol\xi}}
\newcommand{\bzeta}{{\boldsymbol\zeta}}

\newcommand\bSigma{{\boldsymbol\Sigma}}

\def\bbN{\mathbb{N}}

\def\bbR{\mathbb{R}}

\newcommand{\E}{\operatorname{\mathbb{E}}}
\renewcommand{\P}{\operatorname{\mathbb{P}}}
\newcommand{\Var}{\operatorname{Var}}
\newcommand{\Cov}{\operatorname{Cov}}

\newcommand{\pr}[1]{\mathbb{P}\left(#1\right)}

\def\Bin{\text{Bin}}

\def\eps{\varepsilon}

\newcommand{\vol}{\operatorname{vol}}
\def\R{\mathbb{R}}

\def\reach{{\rm reach}}
\def\rad{r}
\def\thetamin{\theta_{\rm min}}
\def\thetamax{\theta_{\rm max}}
\newcommand{\symd}{\, \triangle \, }

\definecolor{darkgreen}{rgb}{.0,.638,.035}
\definecolor{darkred}{rgb}{.638,.0,.035}
\definecolor{purple}{rgb}{0.4,.1,.9}

\newcommand{\1}{{\rm 1}\kern-0.24em{\rm I}}

\begin{document}

\noindent {\huge Spectral Clustering Based on Local PCA}

\medskip
\renewcommand*{\thefootnote}{\fnsymbol{footnote}}
\noindent {\large
Ery Arias-Castro\footnote{Corresponding author: \url{math.ucsd.edu/~eariasca}}\renewcommand{\thefootnote}{\arabic{footnote}}\setcounter{footnote}{0}\footnote{University of California, San Diego},
Gilad Lerman%
\footnote{University of Minnesota, Twin Cities}
and
Teng Zhang%
\footnote{Institute for Mathematics and its Applications (IMA)}
}

\bigskip
\noindent
We propose a spectral clustering method based on local principal components analysis (PCA).  After performing local PCA in selected neighborhoods, the algorithm builds a nearest neighbor graph weighted according to a discrepancy between the principal subspaces in the neighborhoods, and then applies spectral clustering.  As opposed to standard spectral methods based solely on pairwise distances between points, our algorithm is able to resolve intersections.  We establish theoretical guarantees for  simpler variants within a prototypical mathematical framework for multi-manifold clustering, and evaluate our algorithm on various simulated data sets.

\medskip




\section{Introduction}
\label{sec:intro}

The task of multi-manifold clustering, where the data are assumed to be located near surfaces embedded in Euclidean space, is relevant in a variety of applications.  In cosmology, it arises as the extraction of galaxy clusters in the form of filaments (curves) and walls (surfaces)~\citep{galaxy-nonrandom,MarSaa}; in motion segmentation, moving objects tracked along different views form affine or algebraic surfaces~\citep{Ma07,1530127,vidal2006unified,AtevKSCC}; this is also true in face recognition, in the context of images of faces in fixed pose under varying illumination conditions~\citep{Ho03,Basri03,Epstein95}.

We consider a stylized setting where the underlying surfaces are nonparametric in nature, with a particular emphasis on situations where the surfaces intersect.  Specifically, we assume the surfaces are smooth, for otherwise the notion of continuation is potentially ill-posed.  For example, without smoothness assumptions, an L-shaped cluster is indistinguishable from the union of two line-segments meeting at right angle.

Spectral methods \citep{1288832} are particularly suited for nonparametric settings, where the underlying clusters are usually far from convex, making standard methods like K-means irrelevant.  However, a drawback of standard spectral approaches such as the well-known variation of \citet*{Ng02} is their inability to separate intersecting clusters.  Indeed, consider the simplest situation where two straight clusters intersect at right angle, pictured in \figref{segments} below.
\begin{figure}[htbp]
\vspace{0.2in}
\begin{center}
\centering
\includegraphics[width=.27\columnwidth]{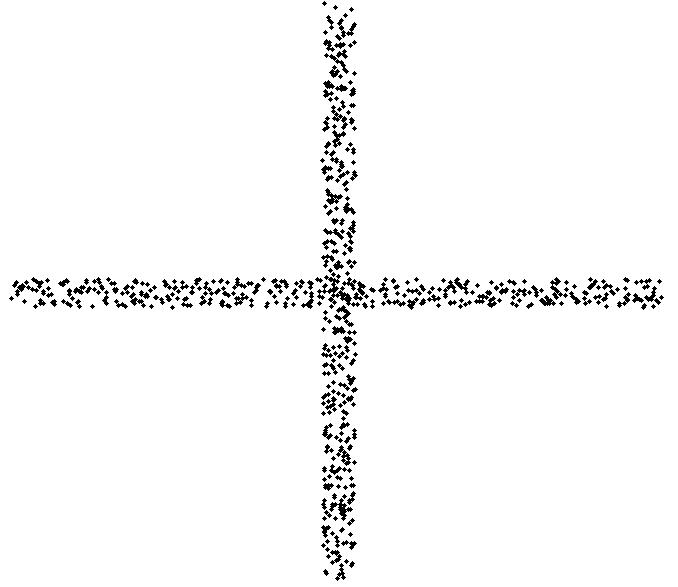} \qquad 
\includegraphics[width=.27\columnwidth]{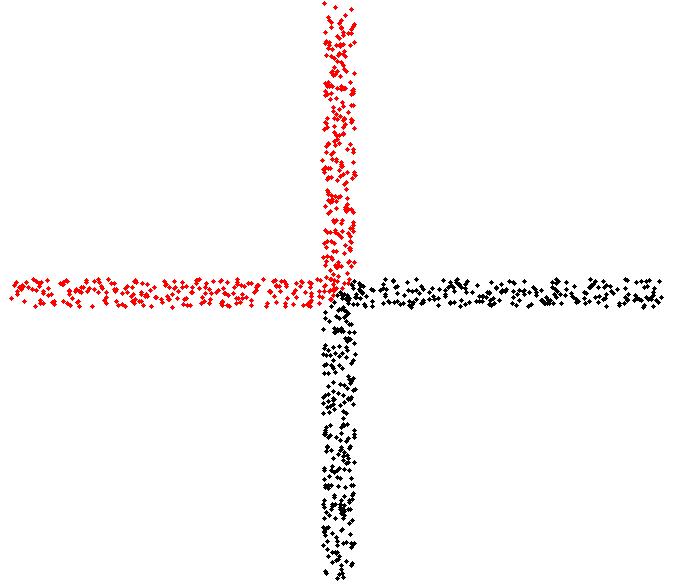} \qquad 
\includegraphics[width=.27\columnwidth]{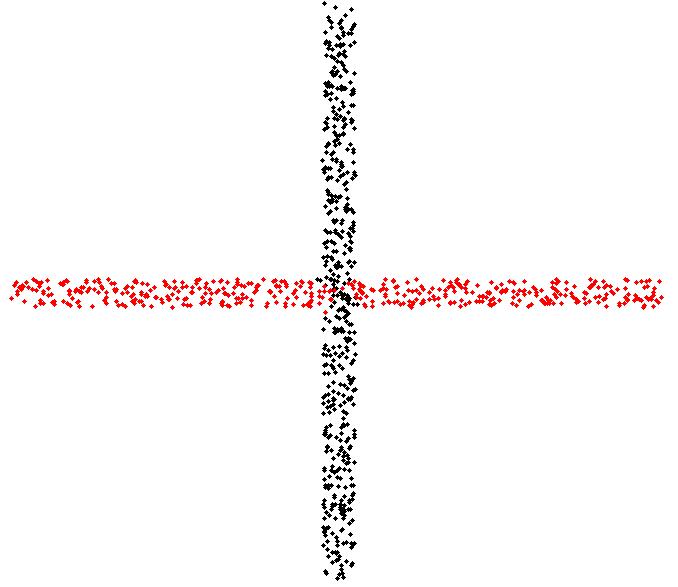}
\caption{Two rectangular clusters intersecting at right angle.  Left: the original data.  Center: a typical output of the standard spectral clustering method of \citet{Ng02}, which is generally unable to resolve intersections.  Right: our method.}
\label{fig:segments}
\end{center}
\end{figure}
The algorithm of \citet{Ng02} is based on pairwise affinities that are decreasing in the distances between data points, making it insensitive to smoothness and, therefore, intersections.  And indeed, this algorithm typically fails to separate intersecting clusters, even in the easiest setting of \figref{segments}.

As argued in \citep{Agarwal05,Agarwal06,Shashua06}, a multiway affinity is needed to capture complex structure in data (here, smoothness) beyond proximity attributes.  For example, \citet{spectral_applied} use a flatness affinity in the context of {\em hybrid linear modeling}, where the surfaces are assumed to be affine subspaces, and subsequently extended to algebraic surfaces via the `kernel trick' \citep*{AtevKSCC}.  Moving beyond parametric models, \citet*{higher-order} consider a localized measure of flatness; see also \citet{NIPS2011_0065}.  Continuing this line of work, we suggest a spectral clustering method based on the estimation of the local linear structure (tangent bundle) via local principal component analysis (PCA).

The idea of using local PCA combined with spectral clustering has precedents in the literature.  In particular, our method is inspired by the work of \citet*{goldberg2009multi}, where the authors develop a spectral clustering method within a semi-supervised learning framework.  Local PCA is also used in the multiscale, spectral-flavored algorithm of \citet*{kushnir}.  This approach is in the zeitgeist.  While writing this paper, 
we became aware of two very recent publications, by \citet*{wang2011spectral} and by \citet*{Gong2012}, both proposing approaches very similar to ours.
We comment on these spectral methods in more detail later on.

The basic proposition of local PCA combined with spectral clustering has two main stages. The first one forms an affinity between a pair of data points that takes into account both their Euclidean distance and a measure of discrepancy between their tangent spaces. Each tangent space is estimated by PCA in a local neighborhood around each point. The second stage applies standard spectral clustering with this affinity.  As a reality check, this relatively simple algorithm succeeds at separating the straight clusters in \figref{segments}.  We tested our algorithm in more elaborate settings, some of them described in \secref{numerics}.

Besides spectral-type approaches to multi-manifold clustering, other methods appear in the literature.  The methods we know of either assume that the different surfaces do not intersect \citep{polito2001grouping}, or that the intersecting surfaces have different intrinsic dimension or density \citep{gionis,Haro06}.  The few exceptions tend to propose very complex methods that promise to be challenging to analyze \citep{souvenir,energy}.

Our contribution is the design and detailed study of a prototypical spectral clustering algorithm based on local PCA, tailored to settings where the underlying clusters come from sampling in the vicinity of smooth surfaces that may intersect.  We endeavored to simplify the algorithm as much as possible without sacrificing performance.
We provide theoretical results for simpler variants within a standard mathematical framework for multi-manifold clustering.
To our knowledge, these are the first mathematically backed successes at the task of resolving intersections in the context of multi-manifold clustering, with the exception of \citep{higher-order}, where the corresponding algorithm is shown to succeed at identifying intersecting curves.
The salient features of that algorithm are illustrated via numerical experiments.

The rest of the paper is organized as follows.
In \secref{algo}, we introduce our methods.
In \secref{math}, we analyze our methods in a standard mathematical framework for multi-manifold learning.
In \secref{numerics}, we perform some numerical experiments illustrating several features of our algorithm.
In \secref{discussion}, we discuss possible extensions.

\section{The methodology}
\label{sec:algo}

We introduce our algorithm and simpler variants that are later analyzed in a mathematical framework.  We start with some review of the literature, zooming in on the most closely related publications.

\subsection{Some precedents}
Using local PCA within a spectral clustering algorithm was implemented in four other publications we know of \citep{goldberg2009multi,kushnir,Gong2012,wang2011spectral}.  As a first stage in their semi-supervised learning method, \cite*{goldberg2009multi} design a spectral clustering algorithm.  The method starts by subsampling the data points, obtaining `centers' in the following way.  Draw $\by_1$ at random from the data and remove its $\ell$-nearest neighbors from the data.  Then repeat with the remaining data, obtaining centers $\by_1, \by_2, \dots$.  Let $\bC_i$ denote the sample covariance in the neighborhood of $\by_i$ made of its $\ell$-nearest neighbors.  An $m$-nearest-neighbor graph is then defined on the centers in terms of the Mahalanobis distances.  Explicitly, the centers $\by_i$ and $\by_j$ are connected in the graph if $\by_j$ is among the $m$ nearest neighbors of $\by_i$ in Mahalanobis distance
\beq \label{maha}
\|\bC_{i}^{-1/2}(\by_i - \by_j)\|,
\eeq
or vice-versa.  The parameters $\ell$ and $m$ are both chosen of order $\log n$.
An existing edge between $\by_i$ and $\by_j$ is then weighted by $\exp(-H_{ij}^2/\eta^2)$, where $H_{ij}$ denotes the Hellinger distance between the probability distributions $\cN({\bf 0}, \bC_i)$ and $\cN({\bf 0}, \bC_j)$.  The spectral graph partitioning algorithm of \citet*{Ng02} --- detailed in \algref{njw} --- is then applied to the resulting affinity matrix, with some form of constrained K-means.
We note that \cite{goldberg2009multi} evaluate their method in the context of semi-supervised learning where the clustering routine is only required to return subclusters of actual clusters.  In particular, the data points other than the centers are discarded.
Note also that their evaluation is empirical.

\begin{algorithm}[htb]
\caption{\quad Spectral Graph Partitioning \citep*{Ng02}}
\label{alg:njw}
\vspace{.1in}
{\bf Input:} \\
\hspace*{.1in} Affinity matrix $\bW = (W_{ij})$, size of the partition $K$ \\[.1in]
{\bf Steps:}\\
\hspace*{.1in} {\bf 1:} Compute $\bZ = (Z_{ij})$ according to $Z_{ij} = W_{ij}/\sqrt{D_i D_j},$ with $D_i = \sum_{j=1}^n W_{ij}$. \\ 
\hspace*{.1in} {\bf 2:} Extract the top $K$ eigenvectors of $\bZ$. \\
\hspace*{.1in} {\bf 3:} Renormalize each row of the resulting $n \times K$ matrix. \\
\hspace*{.1in} {\bf 4:} Apply $K$-means to the row vectors. \\[-.05in]
\end{algorithm}

The algorithm proposed by \cite*{kushnir} is multiscale and works by coarsening the neighborhood graph and computing sampling density and geometric information inferred along the way such as obtained via PCA in local neighborhoods.  This bottom-up flow is then followed by a top-down pass, and the two are iterated a few times.  The algorithm is too complex to be described in detail here, and probably too complex to be analyzed mathematically.  The clustering methods of \cite{goldberg2009multi} and ours can be seen as simpler variants that only go bottom up and coarsen the graph only once.

In the last stages of writing this paper, we learned of the works of \cite*{wang2011spectral} and \cite*{Gong2012}, who propose algorithms very similar to our \algref{proj} detailed below.  Note that these publications do not provide any theoretical guarantees for their methods, which is one of our main contributions here.

\subsection{Our algorithms}
We now describe our method and propose several variants.
Our setting is standard: we observe data points $\bx_1, \dots, \bx_n \in \bbR^D$ that we assume were sampled in the vicinity of $K$ smooth surfaces embedded in $\bbR^D$.  The setting is formalized later in \secref{setting}.

\subsubsection{Connected component extraction: comparing local covariances}
\label{sec:cov}

We start with our simplest variant, which is also the most natural.
The method depends on a neighborhood radius $\rad > 0$, a spatial scale parameter $\eps > 0$ and a covariance (relative) scale $\eta > 0$.
For a vector $\bx$, $\|\bx\|$ denotes its Euclidean norm, and for a (square) matrix $\bA$, $\|\bA\|$ denotes its spectral norm. For $n \in \bbN$, we denote by $[n]$ the set $\{1,\ldots,n\}$.
Given a data set $\bx_1, \dots, \bx_n$, for any point $\bx \in \bbR^D$ and $\rad > 0$, define the neighborhood
\beq \label{Nr}
N_\rad(\bx) = \{\bx_j : \|\bx - \bx_j\| \le \rad\}.
\eeq

\begin{algorithm}[htb]
\caption{\quad Connected Component Extraction: Comparing Covariances}
\label{alg:cov}
\vspace{.1in}
{\bf Input:} \\
\hspace*{.1in} Data points $\bx_1, \dots, \bx_n$; neighborhood radius $\rad > 0$; spatial scale $\eps > 0$, covariance scale $\eta > 0$. \\[.1in]
{\bf Steps:}\\
\hspace*{.1in} {\bf 1:} For each $i \in [n]$, compute the sample covariance matrix $\bC_i$ of $N_\rad(\bx_i)$. \\
\hspace*{.1in} {\bf 2:} Compute the following affinities between data points:
\beq \label{cov-aff}
W_{ij} = \1_{\{\|\bx_i - \bx_j\| \le \eps\}} \cdot \1_{\{\|\bC_i - \bC_j\| \le \eta \rad^2\}}.
\eeq
\hspace*{.1in} {\bf 3:} Remove $\bx_i$ when there is $\bx_j$ such that $\|\bx_{j} - \bx_i\| \le \rad$ and $\|\bC_{j} - \bC_{i}\| > \eta \rad^2$. \\
\hspace*{.1in} {\bf 4:} Extract the connected components of the resulting graph. \\
\hspace*{.1in} {\bf 5:} Points removed in Step~3 are grouped with the closest point that survived Step~3. \\[-.05in]
\end{algorithm}

In summary, the algorithm first creates an unweighted graph: the nodes of this graph are the data points and edges are formed between two nodes if both the distance between these nodes and the distance between the local covariance structures at these nodes are sufficiently small.  After removing the points near the intersection at Step~3, the algorithm then extracts the connected components of the graph.


In principle, the neighborhood size $\rad$ is chosen just large enough that performing PCA in each neighborhood yields a reliable estimate of the local covariance structure.  For this, the number of points inside the neighborhood needs to be large enough, which depends on the sample size $n$, the sampling density, intrinsic dimension of the surfaces and their surface area (Hausdorff measure), how far the points are from the surfaces (i.e., noise level), and the regularity of the surfaces.
The spatial scale parameter $\eps$ depends on the sampling density and $\rad$.  It needs to be large enough that a point has plenty of points within distance $\eps$, including some across an intersection, so each cluster is strongly connected.  At the same time, $\eps$ needs to be small enough that a local linear approximation to the surfaces is a relevant feature of proximity.  Its choice is rather similar to the choice of the scale parameter in standard spectral clustering \citep{Ng02,Zelnik-Manor04}.
The orientation scale $\eta$ needs to be large enough that centers from the same cluster and within distance $\eps$ of each other have local covariance matrices within distance $\eta \rad^2$, but small enough that points from different clusters near their intersection have local covariance matrices separated by a distance substantially larger than $\eta \rad^2$.  This depends on the curvature of the surfaces and the incidence angle at the intersection of two (or more) surfaces.  Note that a typical covariance matrix over a ball of radius $\rad$ has norm of order $\rad^2$, which justifies using our choice of parametrization.
In the mathematical framework we introduce later on, these parameters can be chosen automatically as done in \citep{higher-order}, at least when the points are sampled exactly on the surfaces.  We will not elaborate on that since in practice this does not inform our choice of parameters.

The rationale behind Step~3 is as follows.  As we just discussed, the parameters need to be tuned so that points from the same cluster and within distance $\eps$ have local covariance matrices within distance $\eta \rad^2$.  Hence, $\bx_i$ and $\bx_j$ in Step~3 are necessarily from different clusters.  Since they are near each other, in our model this will imply that they are close to an intersection.  Therefore, roughly speaking, Step~3 removes points near an intersection.

Although this method works in simple situations like that of two intersecting segments (\figref{segments}), it is not meant to be practical.
Indeed, extracting connected components is known to be sensitive to spurious points and therefore unstable.
Furthermore, we found that comparing local covariance matrices as in affinity \eqref{cov-aff} tends to be less stable than comparing local projections as in affinity \eqref{proj-aff}, which brings us to our next variant.

\subsubsection{Connected component extraction: comparing local projections}
\label{sec:proj}

We present another variant also based on extracting the connected components of a neighborhood graph that compares orthogonal projections onto the largest principal directions.

\begin{algorithm}[htb]
\caption{\quad Connected Component Extraction: Comparing Projections}
\label{alg:proj}
\vspace{.1in}
{\bf Input:} \\
\hspace*{.1in} Data points $\bx_1, \dots, \bx_n$; neighborhood radius $\rad > 0$, spatial scale $\eps > 0$, projection scale $\eta > 0$. \\[.1in]
{\bf Steps:}\\
\hspace*{.1in} {\bf 1:} For each $i \in [n]$, compute the sample covariance matrix $\bC_i$ of $N_\rad(\bx_i)$.  \\
\hspace*{.1in} {\bf 2:} Compute the projection $\bQ_i$ onto the eigenvectors of $\bC_i$ with eigenvalue exceeding $\sqrt{\eta} \, \|\bC_i\|$. \\
\hspace*{.1in} {\bf 3:} Compute the following affinities between data points:
\beq \label{proj-aff}
W_{ij} = \1_{\{\|\bx_i - \bx_j\| \le \eps\}} \cdot \1_{\{\|\bQ_i - \bQ_j\| \le \eta\}}.
\eeq
\hspace*{.1in} {\bf 4:} Extract the connected components of the resulting graph. \\
\end{algorithm}

We note that the local intrinsic dimension is determined by thresholding the eigenvalues of the local covariance matrix, keeping the directions with eigenvalues within some range of the largest eigenvalue.
The same strategy is used by \cite{kushnir}, but with a different threshold.
The method is a hard version of what we implemented, which we describe next.

\subsubsection{Covariances or projections?}
\label{sec:compare}

In our numerical experiments, we tried working both directly with covariance matrices as in \eqref{cov-aff} and with projections as in \eqref{proj-aff}.
Note that in our experiments we used spectral graph partitioning with soft versions of these affinities, as described in \secref{stylized}.
We found working with projections to be more reliable.
The problem comes, in part, from boundaries.
When a surface has a boundary, local covariances over neighborhoods that overlap with the boundary are quite different from local covariances over nearby neighborhoods that do not touch the boundary.
Consider the example of two segments, $S_1$ and $S_2$, intersecting at an angle of $\theta \in (0, \pi/2)$ at their middle point, specifically
\[
S_1 = [-1,1] \times \{0\}, \qquad S_2 = \{(x, x \tan \theta): x \in [-\cos \theta,\cos \theta]\}.
\]
Assume there is no noise and that the sampling is uniform.  Assume $\rad \in (0, \frac12 \sin \theta)$ so that the disc centered at $\bx_1 := (1/2,0)$ does not intersect $S_2$, and the disc centered at $\bx_2 := (\frac12 \cos \theta, \frac12 \tan \theta)$ does not intersect $S_1$.  Let $\bx_0 = (1,0)$.  For $\bx \in S_1 \cup S_2$, let $\bC_\bx$ denote the local covariance at $\bx$ over a ball of radius $\rad$.  Simple calculations yield:
\[
\bC_{(1,0)} = \frac{\rad^2}{12} \begin{pmatrix} 1 & 0 \\ 0 & 0 \end{pmatrix}, \quad
\bC_{\bx_1} = \frac{\rad^2}{3} \begin{pmatrix} 1 & 0 \\ 0 & 0 \end{pmatrix}, \quad
\bC_{\bx_2} = \frac{\rad^2}{3} \begin{pmatrix} \cos^2 \theta & \sin(\theta) \cos(\theta) \\ \sin(\theta) \cos(\theta) & \sin^2 \theta \end{pmatrix},
\]
and therefore
\[
\|\bC_{\bx_0} - \bC_{\bx_1}\| = \frac{\rad^2}{4}, \quad \|\bC_{\bx_1} - \bC_{\bx_2}\| = \frac{\sqrt{2} \rad^2}3 \sin \theta.
\]
When $\sin \theta \le \frac3{4\sqrt{2}}$ (roughly, $\theta \le 32^o$), the difference in Frobenius norm between the local covariances at $\bx_0, \bx_1 \in S_1$ is larger than that at $\bx_1 \in S_1$ and $\bx_2 \in S_2$.  As for projections, however,
\[
\bQ_{\bx_0} = \bQ_{\bx_1} = \begin{pmatrix} 1 & 0 \\ 0 & 0 \end{pmatrix}, \quad
\bQ_{\bx_2} = \begin{pmatrix} \cos^2 \theta & \sin(\theta) \cos(\theta) \\ \sin(\theta) \cos(\theta) & \sin^2 \theta \end{pmatrix},
\]
so that
\[
\|\bQ_{\bx_0} - \bQ_{\bx_1}\| = 0, \quad \|\bQ_{\bx_1} - \bQ_{\bx_2}\| = \sqrt{2} \sin \theta.
\]

While in theory the boundary points account for a small portion of the sample, in practice this is not the case and we find that spectral graph partitioning is challenged by having points near the boundary that are far (in affinity) from nearby points from the same cluster.  This may explain why the (soft version of) affinity \eqref{proj-aff} yields better results than the (soft version of) affinity \eqref{cov-aff} in our experiments.

\subsubsection{Spectral Clustering Based on Local PCA}
\label{sec:stylized}

The following variant is more robust in practice and is the algorithm we actually implemented.
The method assumes that the surfaces are of same dimension $d$ and that they are $K$ surfaces, with both parameters $K$ and $d$ known.

\begin{algorithm}[htb]
\caption{\quad Spectral Clustering Based on Local PCA}
\label{alg:spectral}
\vspace{.1in}
{\bf Input:} \\
\hspace*{.1in} Data points $\bx_1, \dots, \bx_n$; neighborhood radius $\rad > 0$; spatial scale $\eps > 0$, projection scale $\eta > 0$; intrinsic dimension $d$; number of clusters $K$. \\[.1in]
{\bf Steps:}\\
\hspace*{.1in} {\bf 0:} Pick one point $\by_1$ at random from the data.  Pick another point $\by_2$ among the data points not included in $N_\rad(\by_1)$, and repeat the process, selecting centers $\by_1, \dots, \by_{n_0}$.\\
\hspace*{.1in} {\bf 1:} For each $i = 1, \dots, n_0$, compute the sample covariance matrix $\bC_i$ of $N_\rad(\by_i)$.  Let $\bQ_i$ denote the orthogonal projection onto the space spanned by the top $d$ eigenvectors of $\bC_i$.\\
\hspace*{.1in} {\bf 2:} Compute the following affinities between center pairs:
\beq \label{aff}
W_{ij} = \exp\left(-\frac{\|\by_i - \by_j\|^2}{\eps^2}\right) \cdot \exp\left(-\frac{\|\bQ_i - \bQ_j\|^2}{\eta^2} \right). 
\eeq
\hspace*{.1in} {\bf 3:} Apply spectral graph partitioning (\algref{njw}) to $\bW$.\\
\hspace*{.1in} {\bf 4:} The data points are clustered according to the closest center in Euclidean distance.\\[-.05in]
\end{algorithm}

We note that $\by_1, \dots, \by_{n_0}$ forms an $\rad$-packing of the data.
The underlying rationale for this coarsening is justified in \citep{goldberg2009multi} by the fact that the covariance matrices, and also the top principal directions, change smoothly with the location of the neighborhood, so that without subsampling these characteristics would not help detect the abrupt event of an intersection.
The affinity \eqref{aff} is of course a soft version of \eqref{proj-aff}.


\subsubsection{Comparison with closely related methods}
We highlight some differences with the other proposals in the literature.
We first compare our approach to that of \cite{goldberg2009multi}, which was our main inspiration.
\bitem
\item {\em Neighborhoods.}  Comparing with \cite{goldberg2009multi}, we define neighborhoods over $\rad$-balls instead of $\ell$-nearest neighbors, and connect points over $\eps$-balls instead of $m$-nearest neighbors.  This choice is for convenience, as these ways are in fact essentially equivalent when the sampling density is fairly uniform.  This is elaborated at length in \citep{1519716,brito,5714241}.
\item {\em Mahalanobis distances.}
\cite{goldberg2009multi} use Mahalanobis distances \eqref{maha} between centers.  In our version, we could for example replace the Euclidean distance $\|\bx_i -\bx_j\|$ in the affinity \eqref{cov-aff} with the average Mahalanobis distance
\beq \label{maha2}
\|\bC_{i}^{-1/2}(\bx_i - \bx_j)\| + \|\bC_{j}^{-1/2}(\bx_j - \bx_i)\|.
\eeq
We actually tried this and found that the algorithm was less stable, particularly under low noise.  Introducing a regularization in this distance --- which requires the introduction of another parameter --- solves this problem partially.

That said, using Mahalanobis distances makes the procedure less sensitive to the choice of $\eps$, in that neighborhoods may include points from different clusters.  Think of two parallel line segments separated by a distance of $\delta$, and assume there is no noise, so the points are sampled exactly from these segments.  Assuming an infinite sample size, the local covariance is the same everywhere so that points within distance $\eps$ are connected by the affinity \eqref{cov-aff}.  Hence, \algref{cov} requires that $\eps < \delta$.  In terms of Mahalanobis distances, points on different segments are infinitely separated, so a version based on these distances would work with any $\eps > 0$.  In the case of curved surfaces and/or noise, the situation is similar, though not as evident.  Even then, the gain in performance guarantees is not obvious, since we only require that $\eps$ be slightly larger in order of magnitude that $\rad$.

\item {\em Hellinger distances.}  As we mentioned earlier, \cite{goldberg2009multi} use Hellinger distances of the probability distributions $\cN({\bf 0}, \bC_i)$ and $\cN({\bf 0}, \bC_j)$ to compare  covariance matrices, specifically
\beq \label{hellinger}
\left(1 - 2^{D/2} \frac{\det(\bC_i \bC_j)^{1/4}}{\det(\bC_i + \bC_j)^{1/2}} \right)^{1/2},
\eeq
if $\bC_i$ and $\bC_j$ are full-rank.
While using these distances or the Frobenius distances makes little difference in practice, we find it easier to work with the latter when it comes to proving theoretical guarantees.  Moreover, it seems more natural to assume a uniform sampling distribution in each neighborhood rather than a normal distribution, so that using the more sophisticated similarity \eqref{hellinger} does not seem justified.

\item {\em K-means.}
We use K-means++ for a good initialization.  However, we found that the more sophisticated size-constrained K-means \citep{constrained-k-means} used in \citep{goldberg2009multi} did not improve the clustering results.

\eitem

%

As we mentioned above, our work was developed in parallel to that of \cite{wang2011spectral} and \cite{Gong2012}.
We highlight some differences.  They do not subsample, but estimate the local tangent space at each data point $\bx_i$.  \cite{wang2011spectral} fit a mixture of $d$-dimensional affine subspaces to the data using MPPCA \citep{tipping1999mixtures}, which is then used to estimate the tangent subspaces at each data point.
\cite{Gong2012} develop some sort of robust local PCA.
While \cite{wang2011spectral} assume all surfaces are of same dimension known to the user, \cite{Gong2012} estimate that locally by looking at the largest gap in the spectrum of estimated local covariance matrix.
This is similar in spirit to what is done in Step~2 of \algref{proj}, but we did not include this step in \algref{spectral} because we did not find it reliable in practice.
We also tried estimating the local dimensionality using the method of \cite{little2009multiscale}, but this failed in the most complex cases.

\cite{wang2011spectral} use a nearest-neighbor graph and their affinity is defined as
\beq \label{wang}
W_{ij} = \Delta_{ij} \cdot \left(\prod_{s = 1}^d \cos \theta_s(i,j)\right)^\alpha,
\eeq
where $\Delta_{ij} = 1$ if $\bx_i$ is among the $\ell$-nearest neighbors of $\bx_j$, or vice versa, while $\Delta_{ij} = 0$ otherwise; $\theta_1(i,j) \ge \cdots \ge \theta_d(i,j)$ are the principal (aka, canonical) angles \citep{MR1061154} between the estimated tangent subspaces at $\bx_i$ and $\bx_j$.  $\ell$ and $\alpha$ are parameters of the method.
\cite{Gong2012} define an affinity that incorporates the self-tuning method of \cite{Zelnik-Manor04}; in our notation, their affinity is
\beq \label{gong}
\exp\left(-\frac{\|\bx_i - \bx_j\|^2}{\eps_i \eps_j}\right) \cdot \exp\left(-\frac{\asin^2 \|\bQ_i - \bQ_j\|}{\eta^2 \|\bx_i - \bx_j\|^2/(\eps_i \eps_j)} \right). 
\eeq
where $\eps_i$ is the distance from $\bx_i$ to its $\ell$-nearest neighbor.  $\ell$ is a parameter.

Although we do not analyze their respective ways of estimating the tangent subspaces, our analysis provides essential insights into their methods, and for that matter, any other method built on spectral clustering based on tangent subspace comparisons.

\section{Mathematical Analysis}
\label{sec:math}

While the analysis of \algref{spectral} seems within reach, there are some complications due to the fact that points near the intersection may form a cluster of their own --- we were not able to discard this possibility.  Instead, we study the simpler variants described in \algref{cov} and \algref{proj}.
Even then, the arguments are rather complex and interestingly involved.  The theoretical guarantees that we thus obtain for these variants are stated in \thmref{main} and proved in \secref{proofs}.  We comment on the analysis of \algref{spectral} right after that.  We note that there are very few theoretical results on resolving intersecting clusters.  In fact, we are only aware of \citep{spectral_theory} in the context of affine surfaces, \citep{Soltanolkotabi_Candes2011} in the context of affine surfaces without noise and \citep{higher-order} in the context of curves.

The generative model we assume is a natural mathematical framework for multi-manifold learning where points are sampled in the vicinity of smooth surfaces embedded in Euclidean space.  For concreteness and ease of exposition, we focus on the situation where two surfaces (i.e., $K = 2$) of same dimension $1 \le d \le D$ intersect.  This special situation already contains all the geometric intricacies of separating intersecting clusters.  On the one hand, clusters of different intrinsic dimension may be separated with an accurate estimation of the local intrinsic dimension without further geometry involved \citep{Haro06}.  On the other hand, more complex intersections (3-way and higher) complicate the situation without offering truly new challenges.
For simplicity of exposition, we assume that the surfaces are submanifolds without boundary, though it will be clear from the analysis (and the experiments) that the method can handle surfaces with (smooth) boundaries that may self-intersect.  We discuss other possible extensions in \secref{discussion}.

Within that framework, we show that \algref{cov} and \algref{proj} are able to identify the clusters accurately except for points near the intersection.  Specifically, with high probability with respect to the sampling distribution, \algref{cov} divides the data points into two groups such that, except for points within distance $C \eps$ of the intersection, all points from the first cluster are in one group and all points from the second cluster are in the other group.  The constant $C$ depends on the surfaces, including their curvatures, separation between them and intersection angle.
The situation for \algref{proj} is more complex, as it may return more than two clusters, but the main feature is that most of two clusters (again, away from the intersection) are in separate connected components.

\subsection{Generative model}
\label{sec:setting}

Each surface we consider is a connected, $C^2$ and compact submanifold without boundary and of dimension $d$ embedded in $\R^D$.  Any such surface has a positive reach, which is what we use to quantify smoothness.  The notion of reach was introduced by \cite{MR0110078}.  Intuitively, a surface has reach exceeding $r$ if, and only if, one can roll a ball of radius $r$ on the surface without obstruction~\citep{walther}.
Formally, for $\bx \in \bbR^D$ and $S \subset \bbR^D$, let
\[
\dist(\bx, S) = \inf_{\bs \in S} \|\bx - \bs\|,
\]
and
\[
B(S, r) = \{\bx : \dist(\bx, S) < r\},
\]
which is often called the $r$-tubular neighborhood (or $r$-neighborhood) of $S$.
The reach of $S$ is the supremum over $r> 0$ such that, for each
$\bx \in B(S, r)$, there is a unique point in $S$ nearest~$\bx$.
It is well-known that, for $C^2$ submanifolds, the reach bounds the radius of curvature from below~\citep[Lem.~4.17]{MR0110078}.
For submanifolds without boundaries, the reach coincides with the condition number introduced in~\citep{1349695}.

When two surfaces $S_1$ and $S_2$ intersect, meaning $S_1 \cap S_2 \neq \emptyset$, we define their incidence angle as
\beq \label{theta}
\theta(S_1,S_2) := \inf \left( \thetamin(T_{S_1}(\bs), T_{S_2}(\bs)) : \bs \in S_1 \cap S_2\right),
\eeq
where $T_S(\bs)$ denote the tangent subspace of submanifold $S$ at point $\bs \in S$, and $\thetamin(T_1, T_2)$ is the smallest {\em nonzero} principal (aka, canonical) angle between subspaces $T_1$ and $T_2$~\citep{MR1061154}.

The clusters are generated as follows.  Each data point $\bx_i$ is drawn according to
\beq \label{data-point}
\bx_i = \bs_i + \bz_i,
\eeq
where $\bs_i$ is drawn from the uniform distribution over $S_1 \cup S_2$ and $\bz_i$ is an additive noise term satisfying $\|\bz_i\| \leq \tau$ --- thus $\tau$ represents the noise or jitter level, and $\tau = 0$ means that the points are sampled on the surfaces.  We assume the points are sampled independently of each other.  We let
\beq \label{I}
I_k = \{i: \bs_i \in S_k\},
\eeq
and the goal is to recover the groups $I_1$ and $I_2$, up to some errors.

\subsection{Performance guarantees}
\label{sec:clustering}

We state some performance guarantees for \algref{cov} and \algref{proj}.

\begin{thm} \label{thm:main}
Consider two connected, compact, twice continuously differentiable submanifolds without boundary, of same dimension, intersecting at a strictly positive angle, with the intersection set having strictly positive reach.
Assume the parameters are set so that
\beq \label{main}
\tau \le \rad \eta/C, \quad \rad \le \eps/C, \quad \eps \le \eta/C, \quad \eta \le 1/C,
\eeq
and $C > 0$ is large enough.
Then with probability at least $1 - C n \exp\big[- n \rad^d \eta^2/C\big]$:
\bitem
\item \algref{cov} returns exactly two groups such that two points from different clusters are not grouped together unless one of them is within distance $C \rad$ from the intersection.
\item \algref{proj} returns at least two groups, and such that two points from different clusters are not grouped together unless one of them is within distance $C \rad$ from the intersection.
\eitem
\end{thm}

We note that the constant $C>0$ depends on what configuration the surfaces are in, in particular their reach and intersection angle, but also aspects that are harder to quantify, like their separation away from their intersection.

%
%

We now comment on the challenge of proving a similar result for \algref{spectral}.
This algorithm relies on knowledge of the intrinsic dimension of the surfaces $d$ and the number of clusters (here $K=2$), but these may be estimated as in \citep{higher-order}, at least in theory, so we assume these parameters are known.  The subsampling done in Step~0 does not pose any problem whatsoever, since the centers are well-spread when the points themselves are.  The difficulty resides in the application of the spectral graph partitioning, \algref{njw}.  If we were to include the intersection-removal step (Step~3 of \algref{cov}) before applying spectral graph partitioning, then a simple adaptation of arguments in \citep{5714241} would suffice.
The real difficulty, and potential pitfall of the method in this framework (without the intersection-removal step), is that the points near the intersection may form their own cluster.
For example, in the simplest case of two affine surfaces intersecting at a positive angle and no sampling noise, the projection matrix at a point near the intersection --- meaning a point whose $r$-ball contains a substantial piece of both surfaces --- would be the projection matrix onto $S_1 + S_2$ seen as a linear subspace.
We were not able to discard this possibility, although we do not observe this happening in practice.
A possible remedy is to constrain the K-means part to only return large-enough clusters.  However, a proper analysis of this would require a substantial amount of additional work and we did not engage seriously in this pursuit.

\section{Numerical Experiments}
\label{sec:numerics}

We tried our code\footnote{The code is available online at \url{http://www.ima.umn.edu/~zhang620/}.} on a few artificial examples. 
Very few algorithms were designed to work in the general situation we consider here and we did not compare our method with any other.
As we argued earlier, the methods of \cite{wang2011spectral} and \cite{Gong2012} are quite similar to ours, and we encourage the reader to also look at the numerical experiments they performed.
Our numerical experiments should be regarded as a proof of concept, only here to show that our method can be implemented and works on some toy examples.

In all experiments, the number of clusters $K$ and the dimension of the manifolds $d$ are assumed known.
We choose spatial scale $\eps$ and the projection scale $\eta$ automatically as follows: we let
\begin{equation}\eps=\max_{1\leq i\leq n_0}\min_{j\neq i} \|\by_i -\by_j \|,\label{eq:eps}\end{equation} and
\begin{equation}\eta=\operatorname*{median}_{(i,j): \|\by_i -\by_j \|<\eps}\|\bP_i-\bP_j\|.\label{eq:eta}\end{equation}
Here, we implicitly assume that the union of all the underlying surfaces forms a connected set.
In that case, the idea behind choosing $\eps$ as in \eqref{eq:eps} is that we want the $\eps$-graph on the centers $\by_1, \dots, \by_n$ to be connected.
Then $\eta$ is chosen so that a center $\by_i$ remains connected in the $(\eps,\eta)$-graph to most of its neighbors in the $\eps$-graph.

The neighborhood radius $\rad$ is chosen by hand for each situation. 
Although we do not know how to choose $r$ automatically, there are some general ad hoc guidelines.
When $\rad$ is too large, the local linear approximation to the underlying surfaces may not hold in neighborhoods of radius $r$, resulting in local PCA becoming inappropriate.
When $\rad$ is too small, there might not be enough points in a neighborhood of radius $r$ to accurately estimate the local tangent subspace to a given surface at that location, resulting in local PCA becoming inaccurate.
From a computational point of view, the smaller $\rad$, the larger the number of neighborhoods and the heavier the computations, particularly at the level of spectral graph partitioning.
In our numerical experiments, we find that our algorithm is more sensitive to the choice of $\rad$ when the clustering problem is more difficult.
We note that automatic choice of tuning parameters remains a challenge in clustering, and machine learning at large, especially when no labels are available whatsoever.
See \citep{Zelnik-Manor04,LBF_journal10,little2009multiscale,Kaslovsky2011}.

Since the algorithm is randomized (see Step 0 in Algorithm~\ref{alg:spectral}) we repeat each simulation $100$ times and report the median misclustering rate and number of times where the misclustering rate is smaller than $5\%$, $10\%$, and $15\%$.
%

We first run Algorithm~\ref{alg:spectral} on several artificial data sets, which are demonstrated in the LHS of Figures~\ref{fig:simulation_2D} and~\ref{fig:simulation_3D}. Table~\ref{table:simulation} reports the local radius $\rad$ used for each data set ($R$ is the global radius of each data set), and the statistics for misclustering rates.
Typical clustering results are demonstrated in the RHS of Figures~\ref{fig:simulation_2D} and~\ref{fig:simulation_3D}.
It is evident that Algorithm~\ref{alg:spectral} performs well in these simulations.
\begin{table}[h!]
\centering
\begin{tabular}{ l c c c c c}
\hline
dataset &$\rad$ & median misclustering rate & 5\% &10\% & 15\%\\
\hline
Three curves& 0.02 (0.034$R$) &4.16\%& 76 & 89 & 89 \\
Self-intersecting curves& 0.1 (0.017$R$)& 1.16\% & 85 & 85 &86\\
Two spheres& 0.2 (0.059$R$)&3.98\% &100&100&100\\
Mobius strips& 0.1 (0.028$R$)&2.22\% &85&86&88\\
Monkey saddle& 0.1 (0.069$R$)&9.73\% &0&67&97\\
Paraboloids& 0.07 (0.048$R$)& 10.42\% &0&12&91\\
\hline
\end{tabular}
\caption{Choices for $r$ and misclustering statistics for the artificial data sets demonstrated in Figures~\ref{fig:simulation_2D} and~\ref{fig:simulation_3D}.
The statistics are based on $100$ repeats and include the median misclustering rate and number of repeats where the misclustering rate is smaller than $5\%$, $10\%$ and $15\%$.
\label{table:simulation}}
\end{table}

\begin{figure}[htbp]
\begin{center}
\includegraphics[width=.49\columnwidth]{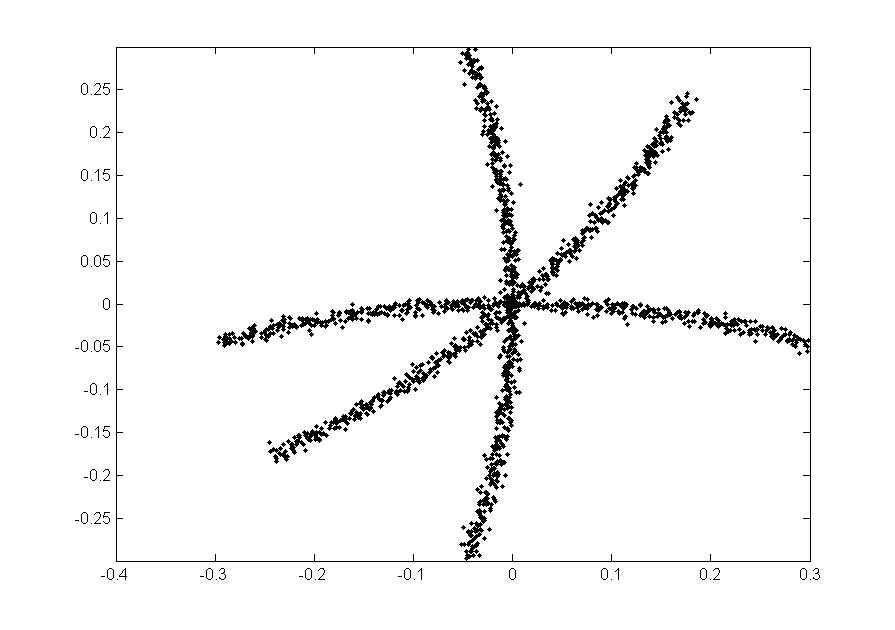}
\includegraphics[width=.49\columnwidth]{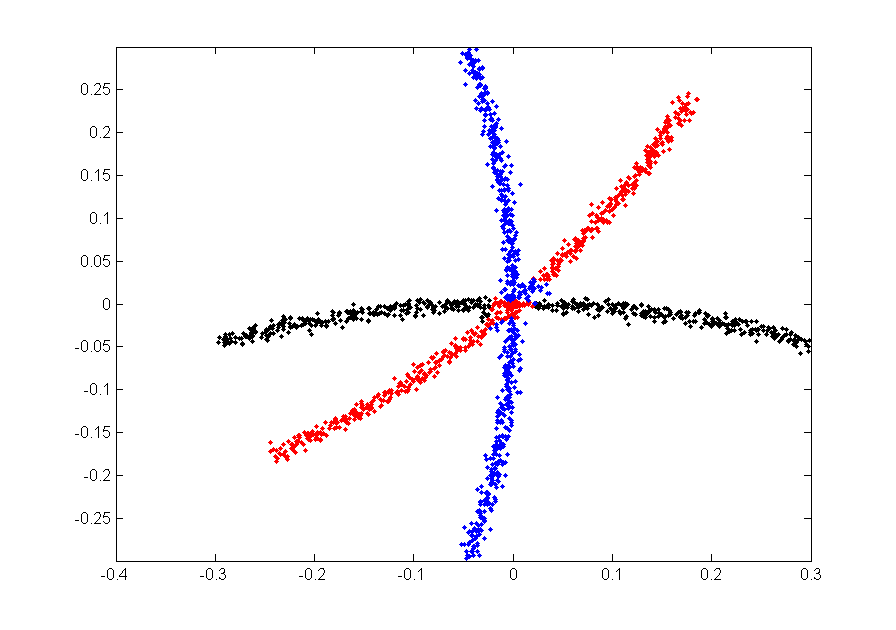}
\includegraphics[width=.49\columnwidth]{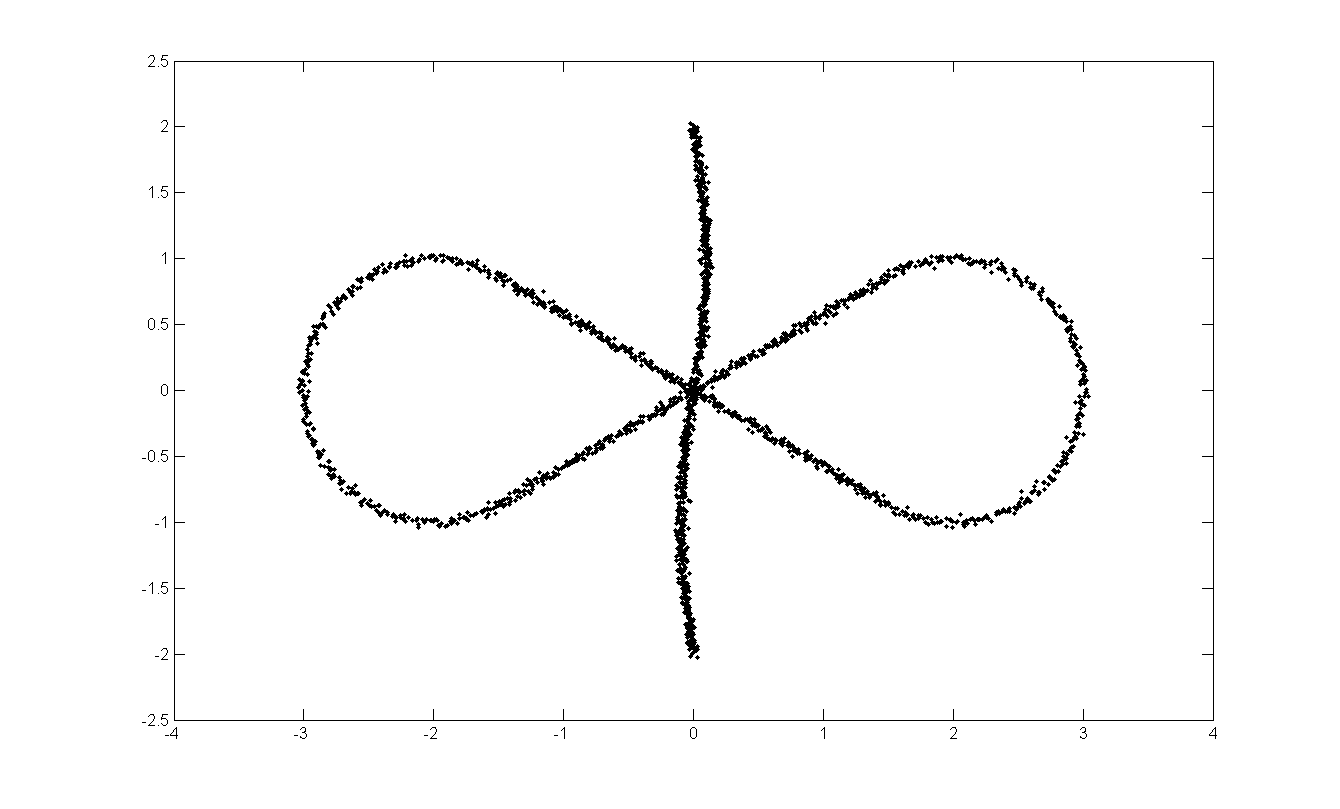}
\includegraphics[width=.49\columnwidth]{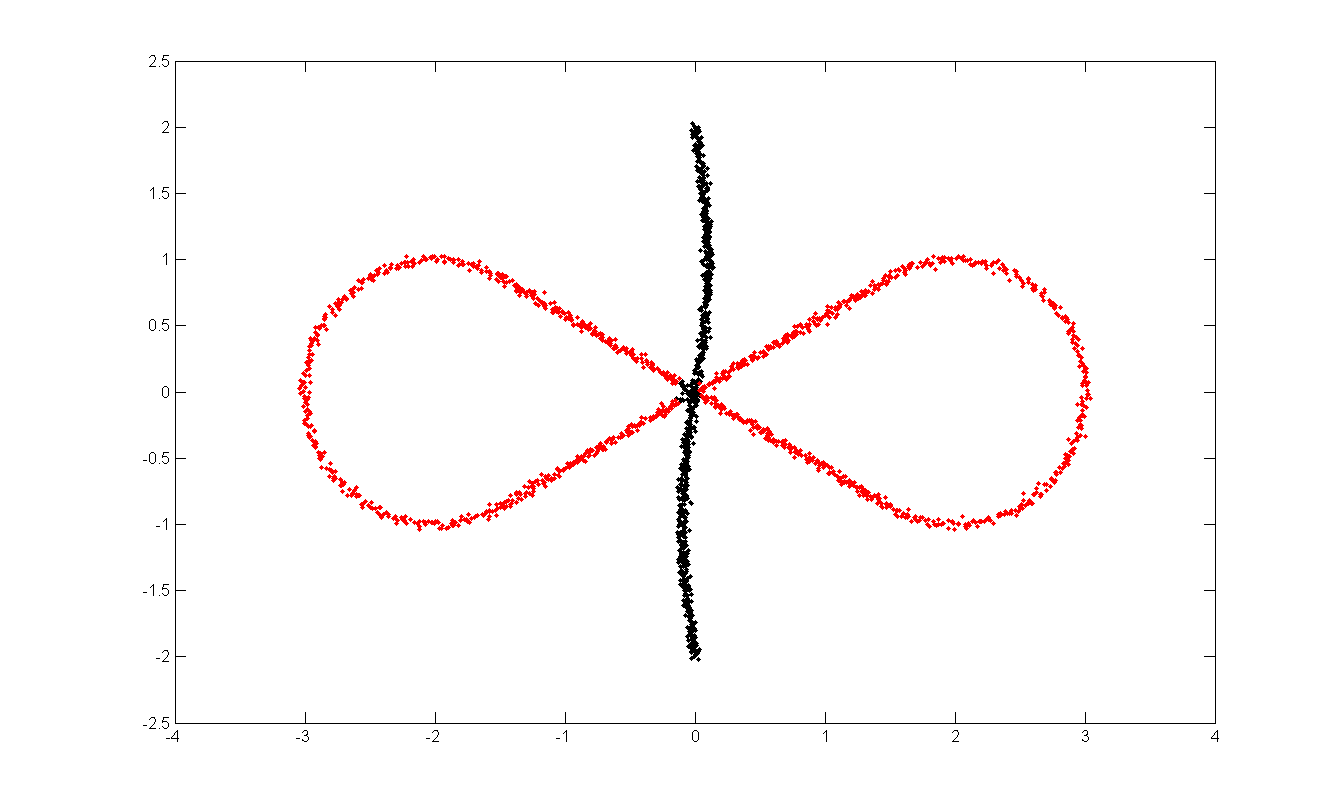} 
\caption{Performance of Algorithm~\ref{alg:spectral} on data sets ``Three curves'' and ``Self-intersecting curves''. Left column is the input data sets, and right column demonstrates the typical clustering.
\label{fig:simulation_2D}}
\end{center}
\end{figure}

\begin{figure}[htbp]
\begin{center}
\centering
\includegraphics[width=.45\columnwidth]{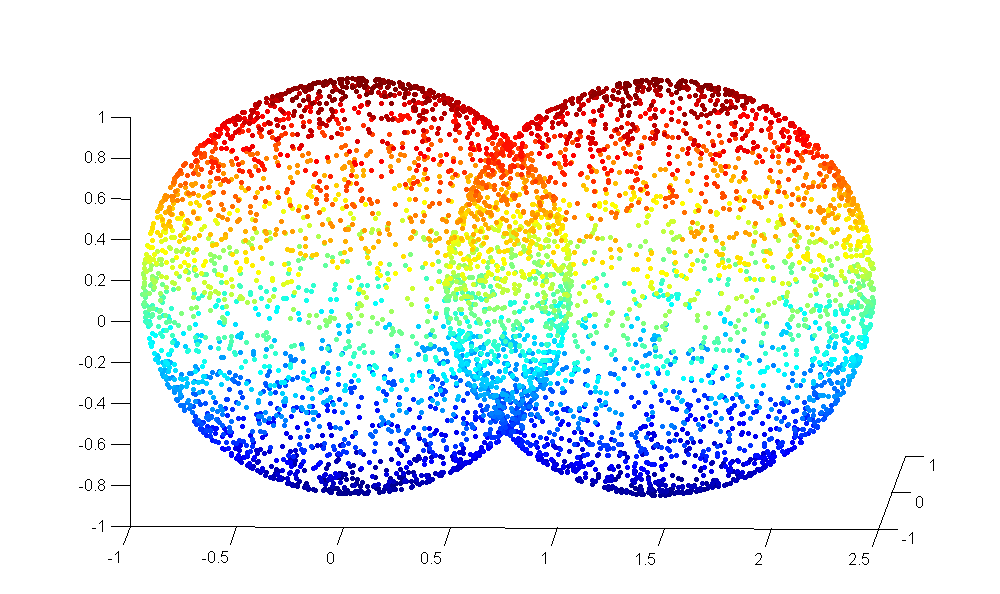}
\includegraphics[width=.45\columnwidth]{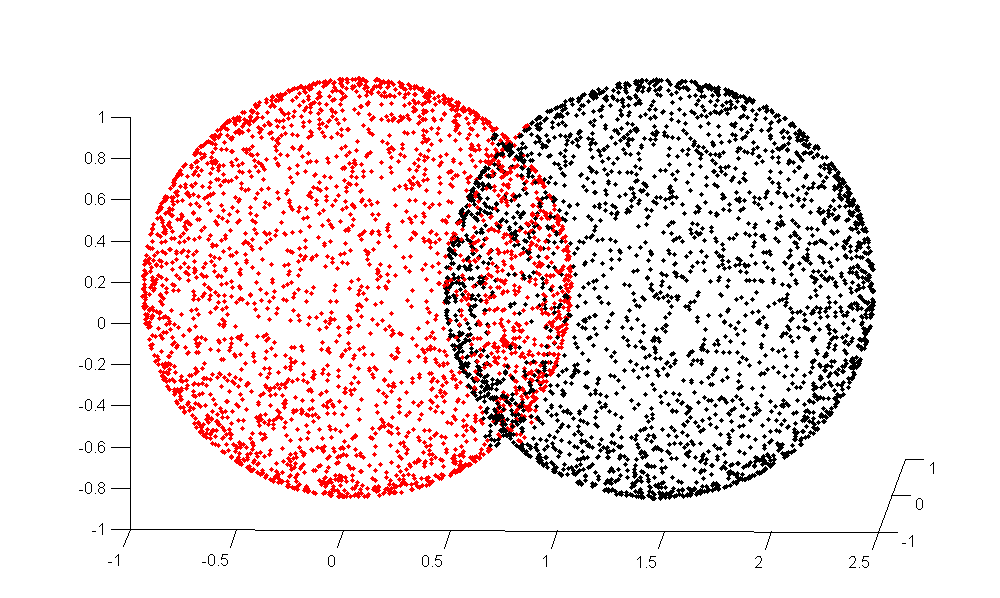}\\
\includegraphics[width=.45\columnwidth]{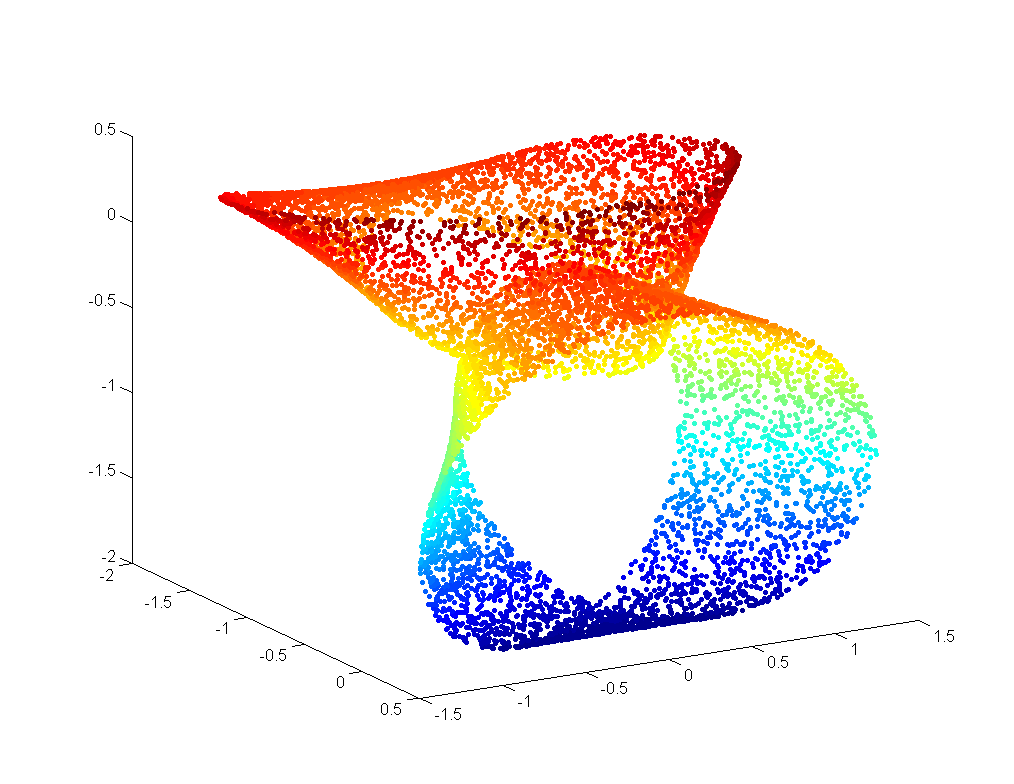}
\includegraphics[width=.45\columnwidth]{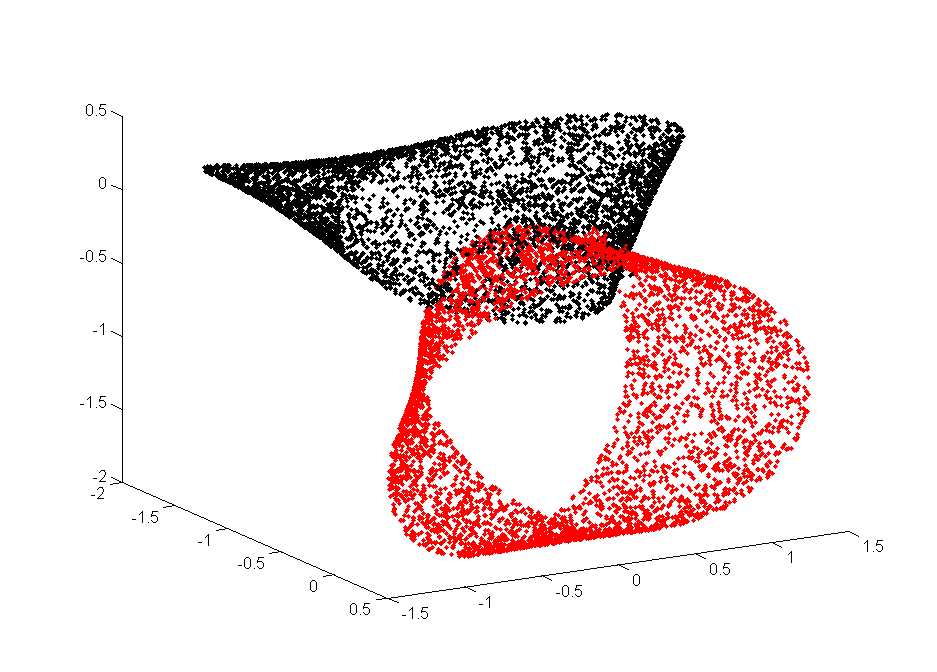}
\includegraphics[width=.45\columnwidth]{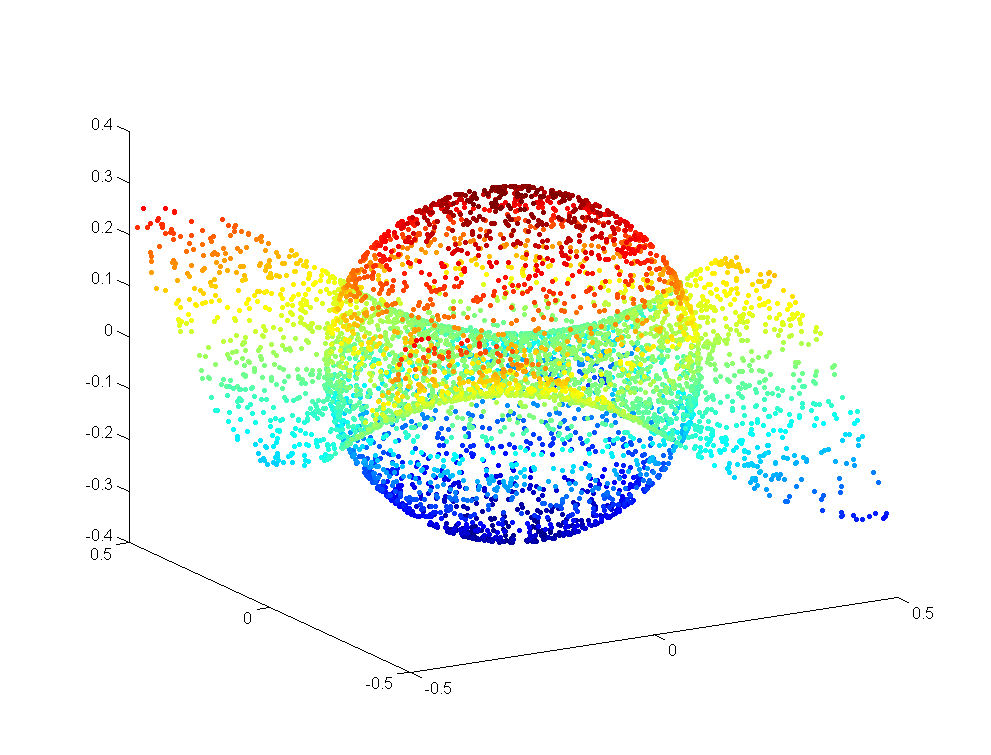}
\includegraphics[width=.45\columnwidth]{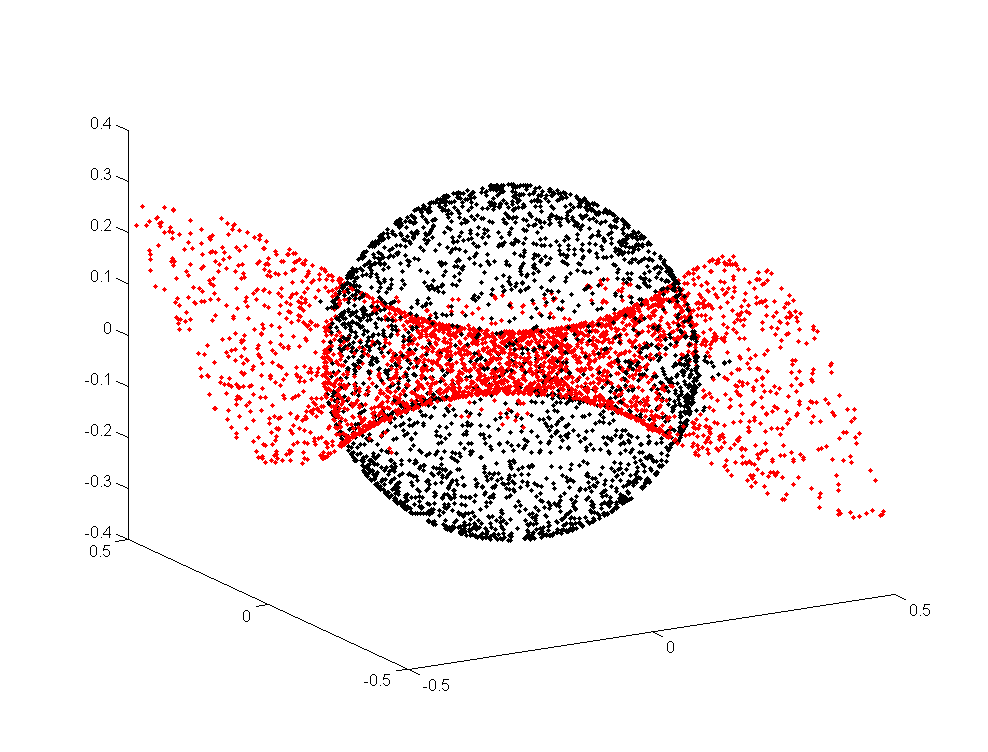}
\includegraphics[width=.45\columnwidth]{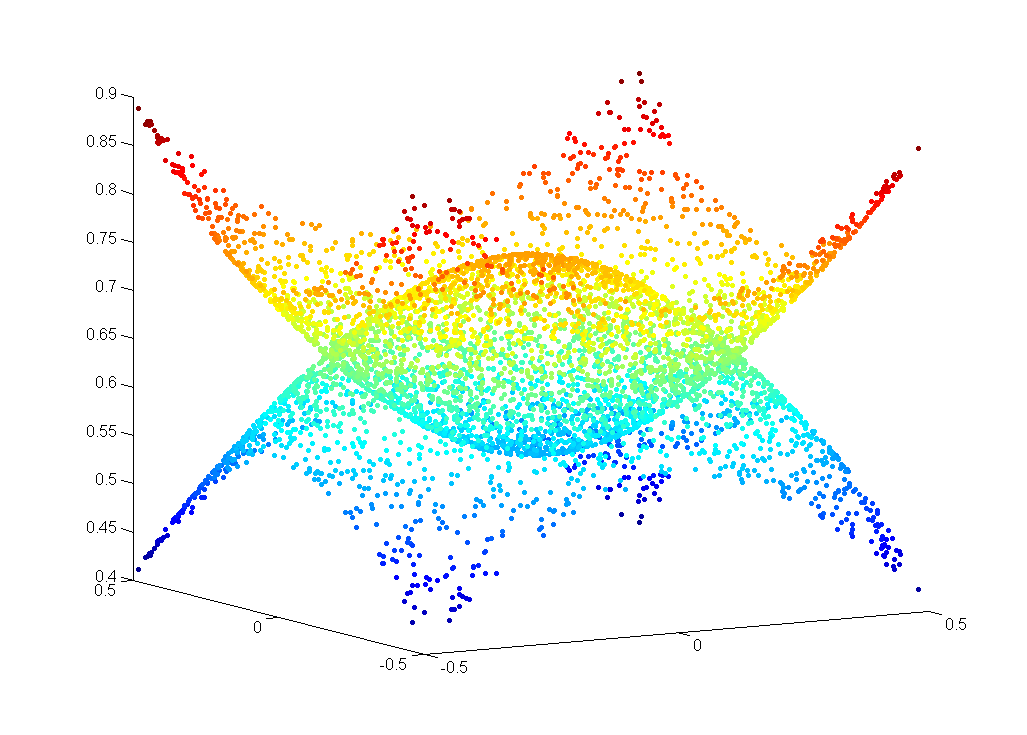}
\includegraphics[width=.45\columnwidth]{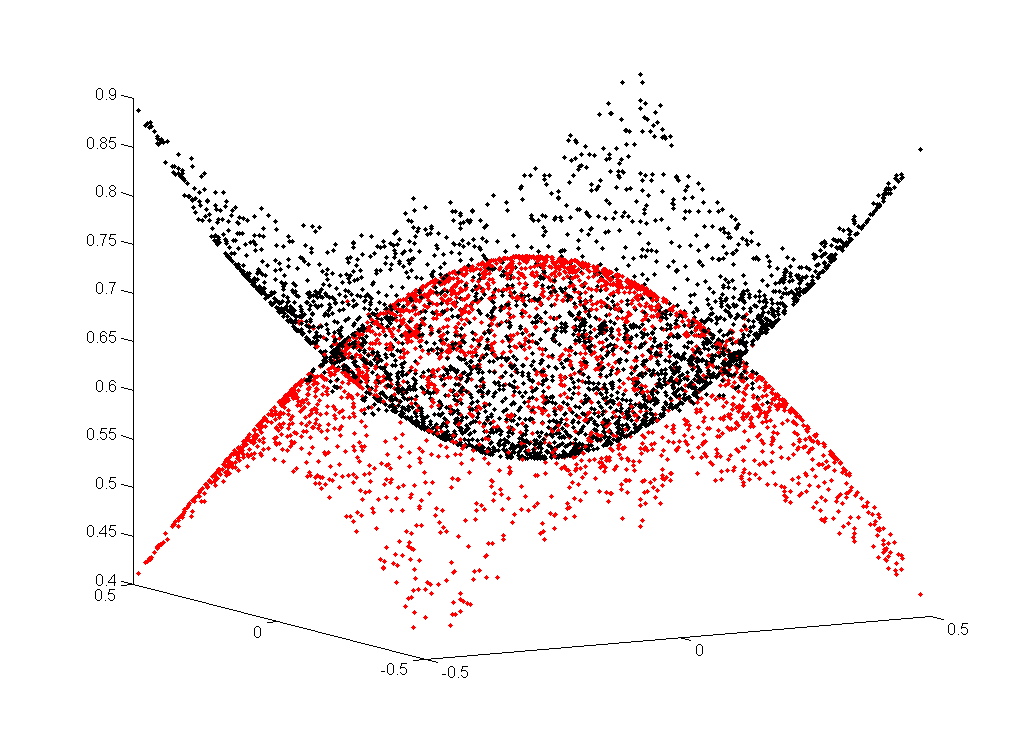}
\caption{Performance of Algorithm~\ref{alg:spectral} on data sets ``Two spheres'', ``Mobius strips'', ``Monkey saddle'' and ``Paraboloids''. Left column is the input data sets, and right column demonstrates the typical clustering.}
\label{fig:simulation_3D}
\end{center}
\end{figure}

In another simulation, we show the dependence of the success of our algorithm on the intersecting angle between curves in Table~\ref{table:2curves} and Figure~\ref{fig:2curves}.
Here, we fix two curves intersecting at a point, and gradually decrease the intersection angle by rotating one of them while holding the other one fixed.
The angles are $\pi/2$, $\pi/4$, $\pi/6$ and $\pi/8$.
From the table we can see that our algorithm performs well when the angle is $\pi/4$, but the performance deteriorates as the angle becomes smaller, and the algorithm almost always fails when the angle is $\pi/8$.

\begin{table}[h!]
\centering
\begin{tabular}{ l c c c c c}
\hline
Intersecting angle&$\rad$ & median misclustering rate & 5\% &10\% & 15\%\\
\hline
$\pi/2$& 0.02 (0.034$R$)&2.08\% & 98 & 98 & 98\\
$\pi/4$& 0.02 (0.034$R$)&3.33\% & 92 & 94 & 94\\
$\pi/6$& 0.02 (0.034$R$)& 5.53\% & 32 & 59 & 59\\
$\pi/8$& 0.02 (0.033$R$)&27.87\% & 0 & 2 & 2 \\
\hline
\end{tabular}
\caption{Choices for $r$ and misclustering statistics for the instances of two intersecting curves demonstrated in \figref{2curves}.
The statistics are based on $100$ repeats and include the median misclustering rate and number of repeats where the misclustering rate is smaller than $5\%$, $10\%$ and $15\%$.
\label{table:2curves}}
\end{table}

\begin{figure}[htbp]
\begin{center}
\includegraphics[width=.45\columnwidth]{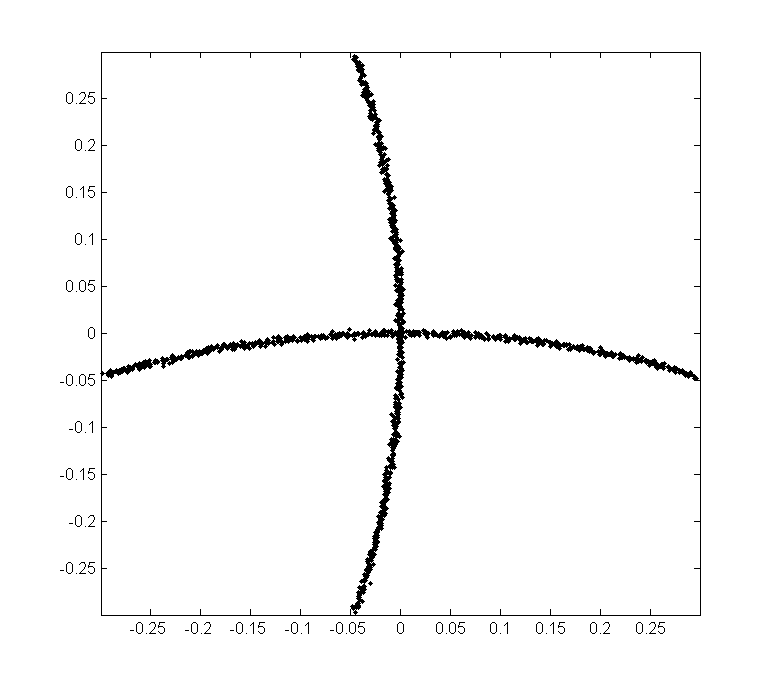}
\includegraphics[width=.45\columnwidth]{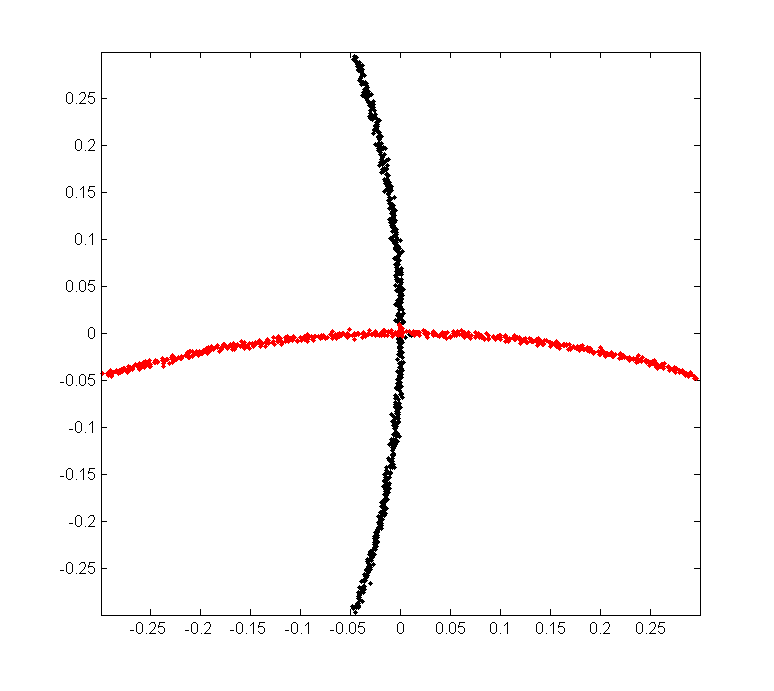}
\includegraphics[width=.45\columnwidth]{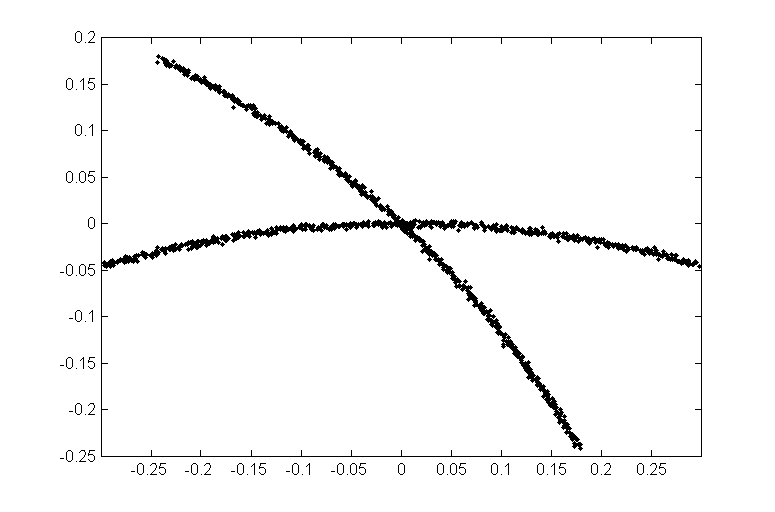}
\includegraphics[width=.45\columnwidth]{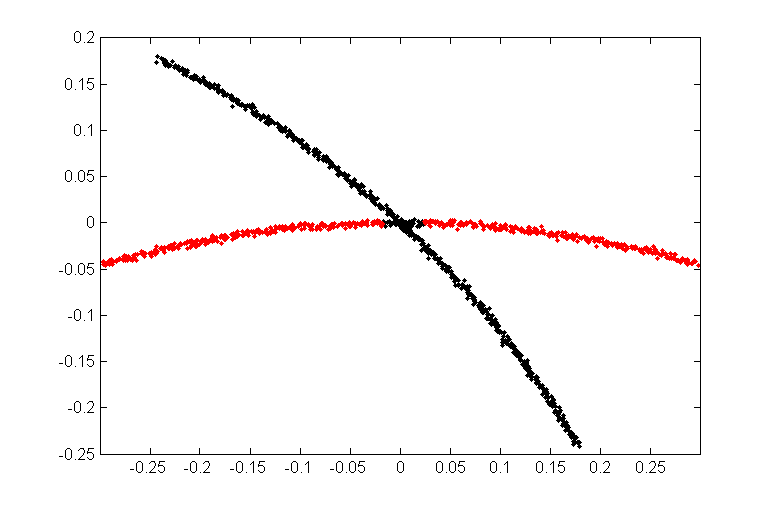}
\includegraphics[width=.45\columnwidth]{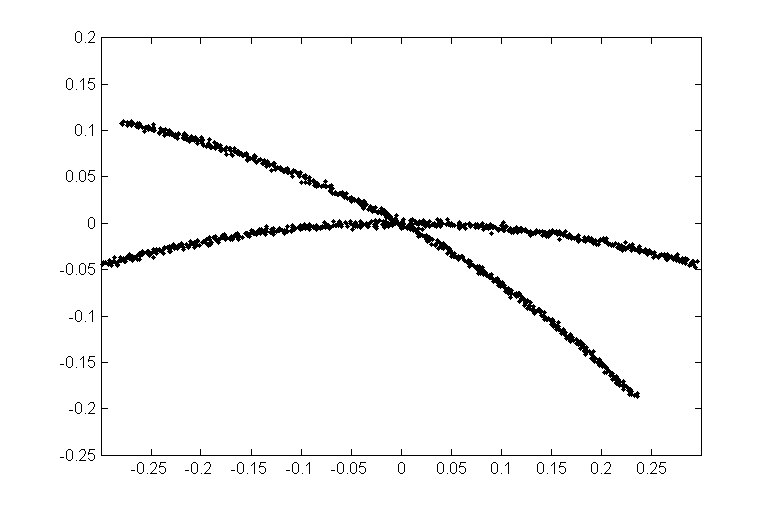}
\includegraphics[width=.45\columnwidth]{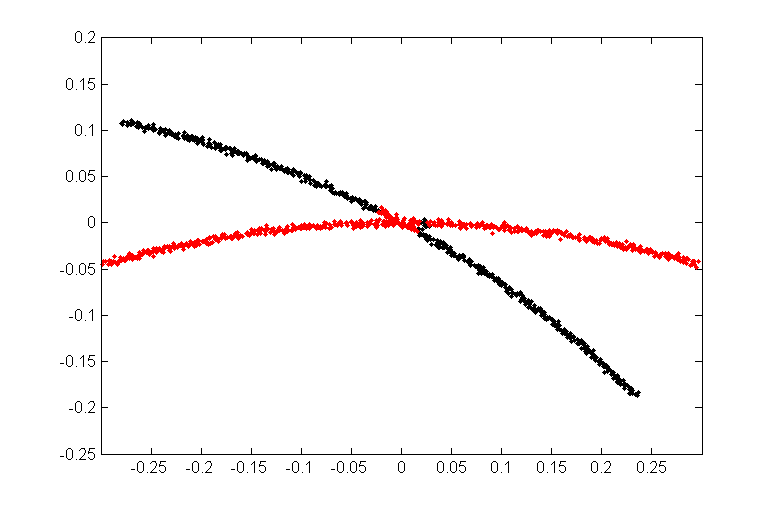}
\includegraphics[width=.45\columnwidth]{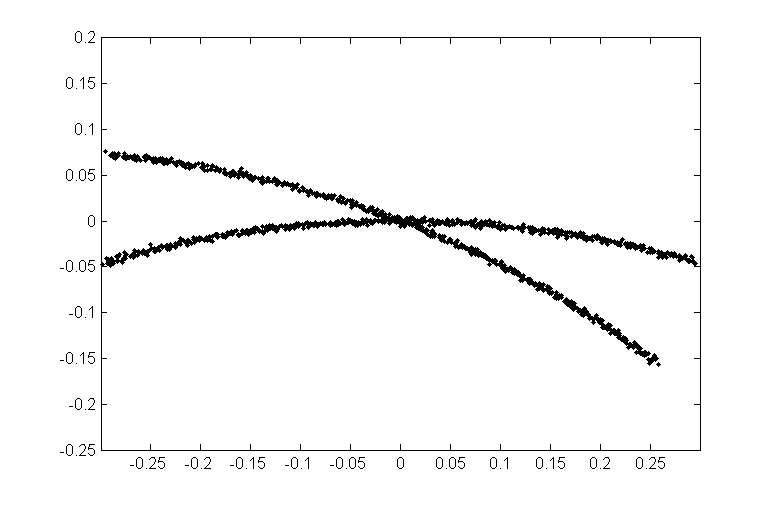}
\includegraphics[width=.45\columnwidth]{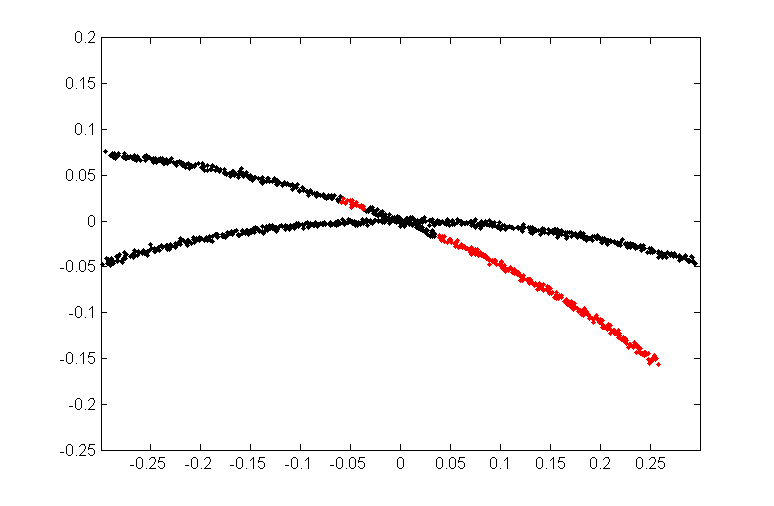}
\caption{Performance of Algorithm~\ref{alg:spectral} on two curves intersecting at various angles $\frac{\pi}{2}$, $\frac{\pi}{4}$, $\frac{\pi}{6}$, $\frac{\pi}{8}$.}
\label{fig:2curves}
\end{center}
\end{figure}

\section{Discussion}
\label{sec:discussion}

We distilled the ideas of~\cite{goldberg2009multi} and of~\cite{kushnir} to cluster points sampled near smooth surfaces.
The key ingredient is the use of local PCA to learn about the local spread and orientation of the data, so as to use that information in an affinity when building a neighborhood graph.

In a typical stylized setting for multi-manifold clustering, we established performance bounds for the simple variants described in \algref{cov} and \algref{proj}, which essentially consist of connecting points that are close in space and orientation, and then extracting the connected components of the resulting graph.
Both are shown to resolve general intersections as long as the incidence angle is strictly positive and the parameters are carefully chosen.
As is commonly the case in such analyzes, our setting can be generalized to other sampling schemes, to multiple intersections, to some features of the surfaces changing with the sample size, and so on, in the spirit of~\citep{higher-order,5714241,spectral_theory}.  We chose to simplify the setup as much as possible while retaining the essential features that makes resolving intersecting clusters challenging.
The resulting arguments are nevertheless rich enough to satisfy the mathematically thirsty reader.

We implemented a spectral version of \algref{proj}, described in \algref{spectral}, that assumes the intrinsic dimensionality and the number of clusters are known.
The resulting approach is very similar to what is offered by \cite{wang2011spectral} and \cite{Gong2012}, although it was developed independently of these works.
\algref{spectral} is shown to perform well in some simulated experiments, although it is somewhat sensitive to the choice of parameters.
This is the case of all other methods for multi-manifold clustering we know of and choosing the parameters automatically remains an open challenge in the field.

\section{Proofs}
\label{sec:proofs}

We start with some additional notation.
The ambient space is $\bbR^D$ unless noted otherwise.
For a vector $\bv \in \bbR^D$, $\|\bv\|$ denotes its Euclidean norm and for a real matrix $\bM \in \bbR^{D \times D}$, $\|\bM\|$ denotes the corresponding operator norm.
For a point $\bx \in \bbR^D$ and $r > 0$, $B(\bx, r)$ denotes the open ball of center $\bx$ and radius $r$, i.e., $B(\bx, r) = \{\by \in \bbR^D: \|\by - \bx\| < r\}$.
For a set $S$ and a point $\bx$, define $\dist(\bx, S) = \inf\{\|\bx - \by\| : \by \in S\}$.
For two points $\ba, \bb$ in the same Euclidean space, $\bb - \ba$ denotes the vector moving $\ba$ to $\bb$.  For a point $\ba$ and a vector $\bv$ in the same Euclidean space, $\ba + \bv$ denotes the translate of $\ba$ by $\bv$.
We identify an affine subspace $T$ with its corresponding linear subspace, for example, when saying that a vector belongs to $T$.

For two subspaces $T$ and $T'$, of possibly different dimensions, we denote by $0 \le \thetamax(T,T') \le \pi/2$ the largest and by $\thetamin(T, T')$ the smallest nonzero principal angle between $T$ and $T'$ \citep{MR1061154}.
When $\bv$ is a vector and $T$ is a subspace, $\angle(\bv, T) := \thetamax(\bbR \bv, T)$ this is the usual definition of the angle between $\bv$ and $T$.

For a subset $A \subset \bbR^D$ and positive integer $d$, $\vol_d(A)$ denotes the $d$-dimensional Hausdorff measure of $A$, and $\vol(A)$ is defined as $\vol_{\dim(A)}(A)$, where $\dim(A)$ is the Hausdorff dimension of $A$.
For a Borel set $A$, let $\lambda_A$ denote the uniform distribution on $A$.

For a set $S \subset \bbR^D$ with reach at least $1/\kappa$, and $\bx$ with $\dist(\bx, S) < 1/\kappa$, let $P_S (\bx)$ denote the metric projection of $\bx$ onto $S$, that is, the point on $S$ closest to $\bx$.
Note that, if $T$ is an affine subspace, then $P_T$ is the usual orthogonal projection onto $T$.
Let $\cS_d(\kappa)$ denote the class of connected, $C^2$ and compact $d$-dimensional submanifolds without boundary embedded in $\R^D$, with reach at least $1/\kappa$.
For a submanifold $S \in \R^D$, let $T_S(\bx)$ denote the tangent space of $S$ at $\bx \in S$.

We will often identify a linear map with its matrix in the canonical basis.
For a symmetric (real) matrix $\bM$, let $\beta_1(\bM) \ge \beta_2(\bM) \ge \cdots$ denote its eigenvalues in decreasing order.

We say that $f: \Omega \subset \bbR^D \to \bbR^D$ is $C$-Lipschitz if
$\|f(\bx) - f(\by)\| \le C \|\bx - \by\|, \forall \bx,\by \in \Omega.$

For two reals $a$ and $b$, $a \vee b = \max(a,b)$ and $a \wedge b = \min(a,b)$.
Additional notation will be introduced as needed.

\subsection{Preliminaries}
\label{sec:prelim}

This section gathers a number of general results from geometry and probability.
We took time to package them into standalone lemmas that could be of potential independent interest, particularly to researchers working in machine learning  and computational geometry.

\subsubsection{Smooth surfaces and their tangent subspaces}
The following result is on approximating a smooth surface near a point by the tangent subspace at that point.  It is based on \citep[Th.~4.18(2)]{MR0110078}.
\begin{lem}
\label{lem:S-approx}
For $S \in \cS_d(\kappa)$, and any two points $\bs, \bs' \in S$,
\beq \label{S-approx1}
\dist(\bs', T_S(\bs)) \leq \frac{\kappa}2 \|\bs' - \bs\|^2,
\eeq
and when $\dist(\bs', T_S(\bs)) \le 1/\kappa$,
\beq \label{S-approx2}
\dist(\bs', T_S(\bs)) \leq \kappa \|P_{T_S(\bs)}(\bs') - \bs\|^2.
\eeq
Moreover, for $\bt \in T_S(\bs)$ such that $\|\bs - \bt\| \le 7/(16 \kappa)$,
\beq \label{S-approx3}
\dist(\bt, S) \leq \kappa \|\bt - \bs\|^2.
\eeq
\end{lem}

\begin{proof}
Let $T$ be short for $T_S(\bs)$.
\citep[Th.~4.18(2)]{MR0110078} says that
\beq \label{4.18}
\dist(\bs' -\bs, T) \le \frac{\kappa}2 \|\bs' - \bs\|^2.
\eeq
Immediately, we have
\[
\dist(\bs' -\bs, T) = \|\bs' - P_{T}(\bs')\| = \dist(\bs', T),
\]
and \eqref{S-approx1} comes from that.
Based on that and Pythagoras theorem, we have
\[\dist(\bs', T) = \|P_{T}(\bs') - \bs'\| \leq \frac{\kappa}2 \|\bs' - \bs\|^2 = \frac{\kappa}2 \big(\|P_{T}(\bs') - \bs'\|^2 + \|P_{T}(\bs') - \bs\|^2\big),\]
so that
\[\dist(\bs', T) \big(1 - \frac\kappa2 \dist(\bs', T) \big) \le \frac{\kappa}2 \|P_{T}(\bs') - \bs\|^2,\]
and \eqref{S-approx2} follows easily from that.  For \eqref{S-approx3}, let $\bs' = P_{T}^{-1}(\bt)$, which is well-defined by \lemref{proj} below and belongs to $B(\bs, 1/(2\kappa))$.
We then apply \eqref{S-approx2} to get
\[\dist(\bt, S) \le \|\bt - \bs'\| = \dist(\bs', T) \le \kappa \|\bt - \bs\|^2.\]
\end{proof}

We need a bound on the angle between tangent subspaces on a smooth surface as a function of the distance between the corresponding points of contact.  This could be deduced directly from \citep[Prop.~6.2, 6.3]{1349695}, but the resulting bound is much looser --- and the underlying proof much more complicated --- than the following, which is again based on \citep[Th.~4.18(2)]{MR0110078}.
\begin{lem}
\label{lem:T-diff}
For $S \in \cS_d(\kappa)$, and any $\bs, \bs' \in S$,
\beq \label{T-diff}
\thetamax(T_S(\bs), T_S(\bs')) \le 2 \asin\left(\frac{\kappa}2 \|\bs' - \bs\| \wedge 1\right).
\eeq
\end{lem}

\begin{proof}
By \eqref{4.18} applied twice, we have
\[
\dist(\bs' -\bs, T_S(\bs)) \vee \dist(\bs -\bs', T_S(\bs'))\le \frac{\kappa}2 \|\bs' - \bs\|^2.
\]
Noting that
\beq \label{vT}
\dist(\bv, T) = \|\bv\| \sin \angle(\bv, T),
\eeq
for any vector $\bv$ and any linear subspace $T$, we get
\[
\sin \angle(\bs' -\bs, T_S(\bs)) \ \vee \ \sin \angle(\bs' -\bs, T_S(\bs')) \le \frac{\kappa}2 \|\bs' - \bs\|.
\]
Noting that the LHS never exceeds 1, and applying the arcsine function --- which is increasing --- on both sides, yields
\[
\angle(\bs' -\bs, T_S(\bs)) \vee \angle(\bs' -\bs, T_S(\bs')) \le \asin\left(\frac{\kappa}2 \|\bs' - \bs\| \wedge 1\right).
\]
We then use the triangle inequality
\[
\thetamax(T_S(\bs), T_S(\bs')) \le \angle(\bs' -\bs, T_S(\bs)) + \angle(\bs' -\bs, T_S(\bs')),
\]
and conclude.
\end{proof}

Below we state some properties of a projection onto a tangent subspace.  A result similar to the first part was proved in \citep[Lem.~2]{higher-order} based on results in \citep{1349695}, but the arguments are simpler here and the constants are sharper.
\begin{lem}
\label{lem:proj}
Take $S \in \cS_d(\kappa)$, $\bs \in S$ and $\rad \le \frac1{2\kappa}$, and let $T$ be short for $T_S(\bs)$.
$P_{T}$ is injective on $B(\bs, \rad) \cap S$ and its image contains $B(\bs, \rad') \cap T$, where $\rad' := (1 - \frac12 (\kappa\rad)^2) \rad$.  Moreover, $P_{T}^{-1}$ has Lipschitz constant bounded by $1 + \frac{64}{49} (\kappa\rad)^2$ over $B(\bs, \rad) \cap T$, for any $\rad \le \frac7{16 \kappa}$.
\end{lem}

\begin{proof}
Take $\bs', \bs'' \in S$ distinct such that $P_{T}(\bs') = P_{T}(\bs'')$.  Equivalently, $\bs'' - \bs'$ is perpendicular to $T_S(\bs)$.  Let $T'$ be short for $T_S(\bs')$.  By \eqref{4.18} and \eqref{vT}, we have
\[
\angle(\bs'' - \bs', T') \le \asin\left(\frac{\kappa}2 \|\bs'' - \bs'\| \wedge 1\right),
\]
and by \eqref{T-diff},
\[
\thetamax(T, T') \le 2 \asin\left(\frac{\kappa}2 \|\bs' - \bs\| \wedge 1\right).
\]
Now, by the triangle inequality,
\[
\angle(\bs'' - \bs', T') \ge \angle(\bs'' - \bs', T) - \thetamax(T, T') = \frac\pi2 - \thetamax(T, T'),
\]
so that
\[
\asin\left(\frac{\kappa}2 \|\bs'' - \bs'\| \wedge 1\right) \ge \frac\pi2 - 2 \asin\left(\frac{\kappa}2 \|\bs' - \bs\| \wedge 1\right).
\]
When $\|\bs' - \bs\| \le 1/\kappa$, the RHS is bounded from below by $\pi/2 - 2 \asin(1/2) $, which then implies that $\frac{\kappa}2 \|\bs'' - \bs'\| \ge \sin(\pi/2 - 2 \asin(1/2)) = 1/2$, that is, $\|\bs'' - \bs'\| \ge 1/\kappa$.
This precludes the situation where $\bs', \bs'' \in B(\bs, 1/(2\kappa))$, so that $P_{T}$ is injective on $B(\bs, \rad)$ when $\rad \le 1/(2\kappa)$.

The same arguments imply that $P_T$ is an open map on $R := B(\bs, \rad) \cap S$.  In particular, $P_T(R)$ contains an open ball in $T$ centered at $\bs$ and $P_T(\partial R) = \partial P_T(R)$, with $\partial R = S \cap \partial B(\bs, \rad)$ since $\partial S = \emptyset$.
Now take any ray out of $\bs$ within $T$, which is necessarily of the form $\bs + \bbR \bv$, where $\bv$ is a unit vector in $ T$.  Let $\bt_a = \bs + a \bv \in T$ for $a \in [0, \infty)$.  Let $a_*$ be the infimum over all $a > 0$ such that $\bt_a \in P_T(R)$.
Note that $a_* > 0$ and $\bt_{a_*} \in P_T(\partial R)$, so that there is $\bs_* \in \partial R$ such that $P_T(\bs_*) = \bt_{a_*}$. 
Let $\bs(a) = P_T^{-1}(\bs + a \bv)$, which is well-defined on $[0,a_*]$ by definition of $a_*$ and the fact that $P_T$ is injective on $R$.
We have that $\dot\bs(a) = D_{\bt_a} P_T^{-1} \bv$ is the unique vector in $T_a := T_S(P_T^{-1}(\bt_a))$ such that $P_T(\dot\bs(a)) = \bv$.  Elementary geometry shows that
\[
\|P_T (\dot\bs(a))\| = \|\dot\bs(a)\| \cos \angle(\dot\bs(a), T) \ge \|\dot\bs(a)\| \cos \thetamax(T_a, T),
\]
with
\[
\cos \thetamax(T_a, T) \ge \cos\left[2 \asin\left(\frac{\kappa}2 \|\bs(a) - \bs\|\right) \right] \ge \zeta := 1 - \frac12 (\kappa\rad)^2,
\]
by \eqref{T-diff}, $\|\bs(a) - \bs\| \le \rad$ and $\cos[2 \asin (x)] = 1 - 2x^2$ when $0 \le x \le 1$.
Since $\|P_T (\dot\bs(a))\| = \|\bv\| = 1$, we have
$\|\dot\bs(a)\| \le 1/\zeta,$
and this holds for all $a < a_*$.  So we can extend $\bs(a)$ to $[0,a_*]$ into a Lipschitz function with constant $1/\zeta$.  Together with the fact that $\bs_* \in \partial B(\bs, \rad)$, this implies that
\[
 \rad = \|\bs_* - \bs\| = \|\bs(a_*) - \bs(0)\| \le \frac1\zeta \|a_* \bv\| = \frac{ a_*}\zeta.
\]
Hence, $a_* \ge \zeta \rad$ and therefore $P_T(R)$ contains $B(\bs, \zeta \rad) \cap T$ as stated.

For the last part, fix $\rad < \frac{7}{16}\kappa$, so there is a unique $h < 1/(2\kappa)$ such that $\zeta h = \rad$, where $\zeta$ is redefined as $\zeta := 1 - \frac12 (\kappa h)^2$.
Take $\bt' \in B(\bs, \rad) \cap T$ and let $\bs' = P_T^{-1}(\bt')$ and $T' = T_S(\bs')$.
We saw that $P_T^{-1}$ is Lipschitz with constant $1/\zeta$ on any ray from $\bs$ of length $\rad$, so that $\|\bs' - \bs\| \le (1/\zeta) \|\bt' - \bs\| \le \rad/\zeta = h$.
The differential of $P_{T}$ at $\bs'$ is $P_T$ itself, seen as a linear map between $T'$ and $T$.  Then for any vector $\bu \in T'$, we have
\[
\|P_T (\bu)\| = \|\bu\| \cos \angle(\bu, T) \ge \|\bu\| \cos \thetamax(T', T),
\]
with
\[
\cos \thetamax(T', T) \ge \cos\left[2 \asin\left(\frac{\kappa}2 \|\bs' - \bs\|\right) \right] \ge 1 - \frac12 (\kappa h)^2 = \zeta,
\]
as before.
Hence, $\| D_{\bt'} P_T^{-1} \| \le 1/\zeta$, and we proved this for all $\bt' \in B(\bs, \rad) \cap T$.
Since that set is convex, we can apply Taylor's theorem and get that $P_T^{-1}$ is Lipschitz on that set with constant $1/\zeta$.
We then have
\[1/\zeta \le 1 + (\kappa h)^2 \le 1 + \frac{64}{49} (\kappa \rad)^2,\]
because $\kappa h \le 1/2$ and $\rad = \zeta h \ge 7h/8$.
\end{proof}

\subsubsection{Volumes and uniform distributions}

Below is a result that quantifies how much the volume of a set changes when applying a Lipschitz map.
This is well-known in measure theory and we only provide a proof for completeness.

\begin{lem} \label{lem:Lip-vol}
Suppose $\Omega$ is a measurable subset of $\bbR^D$ and $f: \Omega \subset \bbR^D \to \bbR^D$ is $C$-Lipschitz.  Then for any measurable set $A \subset \Omega$ and real $d > 0$, $\vol_d(f(A)) \le C^d \vol_d(A)$.
\end{lem}

\begin{proof}
By definition,
\[\vol_d(A) = \lim_{t \to 0} \ V_d^t(A), \qquad V_d^t(A) : = \inf_{(R_i) \in \cR^t(A)} \ \sum_{i \in \bbN} \diam(R_i)^d,\]
where $\cR^t(A)$ is the class of countable sequences $(R_i : i \in \bbN)$ of subsets of $\bbR^D$ such that $A \subset \bigcup_i R_i$ and $\diam(R_i) < t$ for all $i$.
Since $f$ is $C$-Lipschitz, $\diam(f(R)) \le C \diam(R)$ for any $R \subset \Omega$.
Hence, for any $(R_i) \in \cR^t(A)$, $(f(R_i)) \in \cR^{C t}(f(A))$.
This implies that
\[V_d^{Ct}(f(A)) \le \sum_{i \in \bbN} \diam(f(R_i))^d \le C^d \sum_{i \in \bbN} \diam(R_i)^d.\]
Taking the infimum over $(R_i) \in \cR^t(A)$, we get  $V_d^{Ct}(f(A)) \le C^d V_d^{t}(A)$, and we conclude by taking the limit as $t \to 0$, noticing that $V_d^{Ct}(f(A)) \to \vol(f(A))$.
\end{proof}

We compare below two uniform distributions.
For two Borel probability measures $P$ and $Q$ on $\bbR^D$, ${\rm TV}(P, Q)$ denotes their total variation distance, meaning,
\[
{\rm TV}(P, Q) = \sup\{|P(A) - Q(A)| : A \text{ Borel set}\}.
\]
Remember that for a Borel set $A$, $\lambda_A$ denotes the uniform distribution on $A$.

\begin{lem} \label{lem:2unif}
Suppose $A$ and $B$ are two Borel subsets of $\bbR^D$.
Then
\[
{\rm TV}(\lambda_A, \lambda_B) \le 4 \ \frac{\vol(A \symd B)}{\vol(A \cup B)}.
\]
\end{lem}

\begin{proof}
If $A$ and $B$ are not of same dimension, say $\dim(A) > \dim(B)$, then ${\rm TV}(\lambda_A, \lambda_B) = 1$ since $\lambda_A(B) = 0$ while $\lambda_B(B) = 1$.
And we also have
\[\vol(A \symd B) = \vol_{\dim(A)}(A \symd B) = \vol_{\dim(A)}(A) = \vol(A),\]
and
\[\vol(A \cup B) = \vol_{\dim(A)}(A \cup B) = \vol_{\dim(A)}(A) = \vol(A),\]
in both cases because $\vol_{\dim(A)}(B) = 0$.
So the result works in that case.

Therefore assume that $A$ and $B$ are of same dimension.
Assume WLOG that $\vol(A) \ge \vol(B)$.
For any Borel set $U$,
\[\lambda_A(U) - \lambda_B(U) = \frac{\vol(A \cap U)}{\vol(A)} -  \frac{\vol(B \cap U)}{\vol(B)},\]
so that
\beqn
|\lambda_A(U) - \lambda_B(U)|
&\le& \frac{|\vol(A \cap U) - \vol(B \cap U)|}{\vol(A)} + \vol(B \cap U) \left|\frac1{\vol(A)} - \frac1{\vol(B)}\right| \\
&\le&  \frac{\vol(A \symd B)}{\vol(A)} + \frac{\vol(B \cap U)}{\vol(B)} \frac{|\vol(A) - \vol(B)|}{\vol(A)} \\
&\le&  \frac{2 \vol(A \symd B)}{\vol(A)},
\eeqn
and we conclude with the fact that $\vol(A \cup B) \le \vol(A) + \vol(B) \le 2 \vol(A)$.
\end{proof}

We now look at the projection of the uniform distribution on a neighborhood of a surface onto a tangent subspace.
For a Borel probability measure $P$ and measurable function $f: \bbR^D \to \bbR^D$, $P^f$ denotes the push-forward (Borel) measure defined by $P^f(A) = P(f^{-1}(A))$.

\begin{lem} \label{lem:TV-map}
Suppose $A \subset \bbR^D$ is Borel and $f: A \to \bbR^D$ is invertible on $f(A)$, and that both $f$ and $f^{-1}$ are $C$-Lipschitz.  Then
\[
{\rm TV}(\lambda_A^f, \lambda_{f(A)}) \le 8 (C^{\dim(A)} -1).
\]
\end{lem}

\begin{proof}
First, note that $A$ and $f(A)$ are both of same dimension, and that $C \ge 1$ necessarily.
Let $d$ be short for $\dim(A)$.
Take $U \subset f(A)$ Borel and let $V = f^{-1}(U)$.
Then
\[\lambda_A^f(U) = \frac{\vol( A \cap V)}{\vol(A)}, \qquad \lambda_{f(A)}(U) = \frac{\vol( f(A) \cap U)}{\vol(f(A))},\]
\[
|\lambda_A^f(U) - \lambda_{f(A)}(U)|
\le \frac{|\vol(A \cap V) - \vol(f(A) \cap U)|}{\vol(A)} + \frac{|\vol(A) - \vol(f(A))|}{\vol(A)}.
\]
$f$ being invertible, we have $f(A \cap V) = f(A) \cap U$ and $f^{-1}(f(A) \cap U) = A \cap V$.
Therefore, applying \lemref{Lip-vol}, we get
\[
C^{-d} \le \frac{\vol(f(A) \cap U)}{\vol(A \cap V)} \le C^d,
\]
so that
\[|\vol(A \cap V) - \vol(f(A) \cap U)| \le (C^{d} -1) \vol(A \cap V) \le (C^{d} -1) \vol(A).\]
Similarly,
\[ |\vol(A) - \vol(f(A))| \le (C^d -1) \vol(A).\]
We then conclude with \lemref{2unif}.
\end{proof}

Now comes a technical result on the intersection of a smooth surface and a ball.

\begin{lem} \label{lem:TV}
There is a constant $\Clref{TV} \ge 3$ depending only on $d$ such that the following is true.
Take $S \in \cS_d(\kappa)$, $\rad < \frac1{\Clref{TV} \kappa}$ and $\bx \in \bbR^D$ such that $\dist(\bx, S) < \rad$.
Let $\bs = P_S(\bx)$ and $T = T_S(\bs)$.
Then
\[\vol\big(P_T(S \cap B(\bx, \rad)) \symd (T \cap B(\bx, \rad))\big) \le \Clref{TV} (\|\bx -\bs\| + \rad^2) \, \vol(T \cap B(\bx, \rad)).\]
\end{lem}

\begin{proof}
Let $A_\rad = B(\bs, \rad)$, $B_\rad = B(\bx, \rad)$ and $g = P_{T}$ for short.
Note that $T \cap B_\rad = T \cap A_{\rad_0}$ where $\rad_0 := (\rad^2 - \delta^2)^{1/2}$ and $\delta := \|\bx -\bs\|$.
Take $\bs_1 \in S \cap B_\rad$ such that $g(\bs_1)$ is farthest from $\bs$, so that $g(S \cap B_\rad) \subset A_{\rad_1}$ where $\rad_1 := \|\bs - g(\bs_1)\|$ --- note that $\rad_1 \le \rad$.
Let $\ell_1 = \|\bs_1 - g(\bs_1)\|$ and $\by_1$ be the orthogonal projection of $\bs_1$ onto the line $(\bx, \bs)$.
By Pythagoras theorem, we have $ \|\bx - \bs_1\|^2 = \|\bx - \by_1\|^2 + \|\by_1 - \bs_1\|^2$.
We have $ \|\bx - \bs_1\| \le \rad$ and $ \|\by_1 - \bs_1\| = \|\bs - g(\bs_1)\| = \rad_1$.
And because $\ell_1 \le \kappa \rad_1^2 < \rad$ by \eqref{S-approx2}, either $\by_1$ is between $\bx$ and $\bs$, in which case $\|\bx - \by_1\| = \delta - \ell_1$, or $\bs$ is between $\bx$ and $\by_1$, in which case $\|\bx - \by_1\| = \delta + \ell_1$.
In any case, $\rad^2 \ge \rad_1^2 + (\delta - \ell_1)^2$, which together with $\ell_1 \le \kappa \rad_1^2$ implies $\rad_1^2 \le \rad^2 - \delta^2 + 2 \delta \ell_1 \le \rad_0^2 + 2\kappa\rad_1^2 \delta$, leading to $\rad_1 \le (1-2\kappa\delta)^{-1/2} \rad_0 \le (1 + 4\kappa \delta) \rad_0$ after noticing that $\delta \le \rad < 1/(3\kappa)$.
From $g(S \cap B_\rad) \subset T \cap A_{\rad_1}$, we get
\beqn
\vol\big(g(S \cap B_\rad) \setminus (T \cap B_\rad)\big) &\le& \vol(T \cap A_{\rad_1}) - \vol(T \cap A_{\rad_0}) \\ &=& ((\rad_1/\rad_0)^d - 1) \vol(T \cap A_{\rad_0}).
\eeqn

We follow similar arguments to get a sort of reverse relationship.
Take $\bs_2 \in S \cap B_\rad$ such that $g(S \cap B_\rad) \supset T \cap A_{\rad_2}$, where $\rad_2 := \|\bs - g(\bs_2)\|$ is largest.
Assuming $\rad$ is small enough, by \lemref{proj}, $g^{-1}$ is well-defined on $T \cap A_\rad$, so that necessarily $\bs_2 \in \partial B_\rad$.
Let $\ell_2 = \|\bs_2 - g(\bs_2)\|$ and $\by_2$ be the orthogonal projection of $\bs_2$ onto the line $(\bx, \bs)$.
By Pythagoras theorem, we have $ \|\bx - \bs_2\|^2 = \|\bx - \by_2\|^2 + \|\by_2 - \bs_2\|^2$.
We have $ \|\bx - \bs_2\| = \rad$ and $ \|\by_2 - \bs_2\| = \|\bs - g(\bs_2)\| = \rad_2$.
And by the triangle inequality, $\|\bx - \by_2\| \le \|\bx - \bs\| + \|\by_2 - \bs\| = \delta + \ell_2$.
Hence, $\rad^2 \le \rad_2^2 + (\delta + \ell_2)^2$, which together with $\ell_2 \le \kappa \rad_2^2$ by \eqref{S-approx2}, implies $\rad_2^2 \ge \rad^2 - \delta^2 - 2 \delta \ell_2 - \ell_2^2 \ge \rad_0^2 - (2 \delta + \kappa \rad^2) \kappa\rad_2^2$, leading to $\rad_2 \ge (1+2\kappa\delta+\kappa^2 \rad^2)^{-1/2} \rad_0 \ge (1 - 2 \kappa \delta - \kappa^2 \rad^2) \rad_0$.
From $g(S \cap B_\rad) \supset T \cap A_{\rad_2}$, we get
\beqn
\vol\big((T \cap B_\rad) \setminus g(S \cap B_\rad)\big) &\le& \vol(T \cap A_{\rad_0}) - \vol(T \cap A_{\rad_2}) \\ &=& (1 - (\rad_2/\rad_0)^d) \vol(T \cap A_{\rad_0}).
\eeqn

All together, we have
\beqn
\vol\big(g(S \cap B_\rad) \symd (T \cap B_\rad)\big)
&\le& \big((\rad_1/\rad_0)^d - (\rad_2/\rad_0)^d\big) \  \vol(T \cap A_{\rad_0}) \\
&\le& \big( (1 + 4\kappa \delta)^d - (1 - 2 \kappa \delta - \kappa^2 \rad^2)^d\big) \vol((T \cap B_\rad)),
\eeqn
with $(1 + 4\kappa \rad)^d - (1 - 4\kappa \rad)^d \le C (\delta + \rad^2)$ when $\delta \le \rad \le 1/(3 \kappa)$, for a constant $C$ depending only on $d$ and $\kappa$.
The result follows from this.
\end{proof}

We bound below the $d$-volume of a the intersection of a ball with a smooth surface.
Though it could be obtained as a special case of \lemref{TV}, we provide a direct proof because this result is at the cornerstone of many results in the literature on sampling points uniformly on a smooth surface.

\begin{lem} \label{lem:ball-vol}
Suppose $S \in \cS_d(\kappa)$.  Then for any $\bs \in S$ and $\rad < \frac1{(d \vee 3)\kappa}$, we have
\[1 - 2d\kappa\rad \le \frac{\vol(S \cap B(\bs, \rad))}{\vol(T \cap B(\bs, \rad))} \le 1 + 2d\kappa\rad,\]
where $T := T_S(\bs)$ is the tangent subspace of $S$ at $\bs$.
\end{lem}

\begin{proof}
Let $T = T_S(\bs)$, $B_\rad = B(\bs, \rad)$ and $g = P_{T}$ for short.
By \lemref{proj}, $g$ is bi-Lipschitz with constants $(1+\kappa\rad)^{-1}$ and 1 on $S \cap B_\rad$, so by \lemref{Lip-vol} we have
\[
(1+\kappa\rad)^{-d} \le \frac{\vol(g(S \cap B_\rad))}{\vol(S \cap B_\rad)} \le 1.
\]
That $g^{-1}$ is Lipschitz with constant $1 + \kappa \rad$ on $g(S \cap B_\rad)$ also implies that $g(S \cap B_\rad)$ contains $T \cap B_{\rad'}$ where $\rad' := \rad/(1+\kappa\rad)$.  From this, and the fact that $g(S \cap B_\rad) \subset T \cap B_\rad$, we get
\beq \label{vol1}
1 \le \frac{\vol(T \cap B_\rad)}{\vol(g(S \cap B_\rad))} \le \frac{\vol(T \cap B_\rad)}{\vol(T \cap B_{\rad'})} = \frac{\rad^d}{{\rad'}^d} = (1+\kappa\rad)^{d}. 
\eeq
We therefore have
\[
\vol(S \cap B_\rad) \ge \vol(g(S \cap B_\rad)) \ge (1+\kappa \rad)^{-d} \vol(T \cap B_\rad),
\]
and
\[
\vol(S \cap B_\rad) \le (1+\kappa\rad)^{d} \vol(g(S \cap B_\rad)) \le (1+\kappa\rad)^{d} \vol(T \cap B_\rad).
\]
And we conclude with the inequality $(1+x)^d \le 1 + 2dx$ valid for any $x \in [0,1/d]$ and any $d \ge 1$.
\end{proof}

We now look at the density of a sample from the uniform on a smooth, compact surface.

\begin{lem} \label{lem:N-size}
There is a constant $\Clref{N-size}>0$ such that the following is true.  If $S \in \cS_d(\kappa)$ and we sample $n$ points $\bs_1, \dots, \bs_n$ independently and uniformly at random from $S$, and if $0 < \rad < 1/(\Clref{N-size} \kappa)$, then with probability at least $1 - \Clref{N-size} \rad^{-d} \exp(- n \rad^d/\Clref{N-size})$, any ball of radius $\rad$ with center on $S$ has between $n \rad^d/\Clref{N-size}$ and $\Clref{N-size} n \rad^d$ sample points.
\end{lem}

\begin{proof}
For a set $R$, let $N(R)$ denote the number of sample points in $R$.
For any $R$ measurable, $N(R) \sim \Bin(n, p_R),$ where $p_R := \vol(R \cap S)/\vol(S)$.
Let $\bx_1, \dots, \bx_m$ be an $(\rad/2)$-packing of $S$, and let $B_i = B(\bx_j, \rad/4) \cap S$.
For any $\bs \in S$, there is $j$ such that $\|\bs - \bx_j\| \le \rad/2$, which implies $B_i \subset B(\bs, \rad)$ by the triangle inequality.
Hence, $\min_{\bs \in S} N(B(\bs, \rad)) \ge \min_i N(B_i)$.

By the fact that $B_i \cap B_j = \emptyset$ for $i \ne j$,
\[
\vol(S) \ge \sum_{i=1}^m \vol(B_i) \ge m \min_i \vol(B_i),
\]
and assuming that $\rad$ is small enough that we may apply \lemref{ball-vol}, we have
\[
\min_i \vol(B_i) \ge \frac{\omega_d}2 (\rad/4)^d,
\]
where $\omega_d$ is the volume of the $d$-dimensional unit ball.
This leads to $m \le C \rad^{-d}$ and $p := \min_i \, p_{B_i} \ge \rad^d/C$, where $C > 0$ depends only on $S$.

Now, applying Bernstein's inequality to the binomial distribution, we get
\beq \label{Ni}
\pr{N(B_i) \le n p/2} \le \pr{N(B_i) \le n p_{B_i}/2} \le e^{-(3/32) n p_{B_i}} \le e^{-(3/32) n p}.
\eeq
We follow this with the union bound, to get
\[
\pr{\min_{\bs \in S} N(B(\bs, \rad)) \le n \rad^d/(2C)} \le m e^{-(3/32) np} \le\ C \rad^{-d} e^{-\frac{3}{32 C} n \rad^d}.
\]
From this the lower bound follows.
The proof of the upper bound is similar.
\end{proof}

Next, we bound the volume of the symmetric difference between two balls.
\begin{lem} \label{lem:2ball-vol}
Take $\bx, \by \in \bbR^d$ and $0 < \delta \le 1$.  Then
\[
\frac{\vol(B(\bx, \delta) \symd B(\by, 1))}{2 \vol(B(0, 1))} \le 1- (1 - \|\bx - \by\|)_+^d \wedge \delta^d.
\]
\end{lem}

\begin{proof}
It suffices to prove the result when $\|\bx - \by\| < 1$.
In that case, with $\gamma := (1 - \|\bx - \by\|) \wedge \delta$, we have $B(\bx, \gamma) \subset B(\bx, \delta) \cap B(\by, 1)$, so that
\beqn
\vol(B(\bx, \delta) \symd B(\by, 1))
&=& \vol(B(\bx, \delta)) + \vol(B(\by, 1)) - 2 \vol(B(\bx, \delta) \cap B(\by, 1)) \\
&\le& 2 \vol(B(\by, 1)) - 2 \vol(B(\bx, \gamma)) \\
&=& 2 \vol(B(\by,1)) (1 - \gamma^d).
\eeqn
\end{proof}

\subsubsection{Covariances}

The result below describes explicitly the covariance matrix of the uniform distribution over the unit ball of a subspace.
\begin{lem} \label{lem:unif-cov}
Let $T$ be a subspace of dimension $d$.  Then the covariance matrix of the uniform distribution on $T \cap B(0,1)$ (seen as a linear map) is equal to $c P_T$, where $c:=\frac1{d+2}$.
\end{lem}

\begin{proof}
Assume WLOG that $T = \bbR^d \times \{0\}$.
Let $X$ be distributed according to the uniform distribution on $T \cap B(0,1)$ and let $R = \|X\|$.
Note that
\[\pr{R \le r} = \frac{\vol(T \cap B(0,r))}{\vol(T \cap B(0,1))} = r^d, \quad \forall r \in [0,1].\]
By symmetry, $\E(X_i X_j) = 0$ if $i \ne j$, while
\[
\E(X_1^2) = \frac1d \E(X_1^2 + \dots + X_d^2) = \frac1d \E(R^2) = \frac1d \int_0^1 r^2 \cdot d r^{d-1} {\rm d}r = \frac1{d+2}.
\]
This is exactly the representation of $\frac1{d+2} P_T$ in the canonical basis of $\bbR^D$.
\end{proof}

We now show that a bound on the total variation distance between two compactly supported distributions implies a bound on the difference between their covariance matrices.
For a measure $P$ on $\bbR^D$ and an integrable function $f$, let $P(f)$ denote the integral of $f$ with respect to $P$, that is,
\[
P(f) = \int f(x) P(dx),
\]
and let $\E(P) = P(\bx)$ and $\Cov(P) = P(\bx \bx^\top) - P(\bx) P(\bx)^\top$ denote the mean and covariance matrix of $P$, respectively.

\begin{lem} \label{lem:TV-cov}
Suppose $\lambda$ and $\nu$ are two Borel probability measures on $\bbR^d$ supported on $B(0,1)$.  Then
\[
\|\E(\lambda) - \E(\nu)\| \le \sqrt{d} \, {\rm TV}(\lambda, \nu), \qquad \|\Cov(\lambda) - \Cov(\nu)\| \le 3 d \, {\rm TV}(\lambda, \nu).
\]
\end{lem}

\begin{proof}
Let $f_k(\bt) = t_k$ when $\bt = (t_1, \dots, t_d)$, and note that $|f_k(\bt)| \le 1$ for all $k$ and all $\bt \in B(0,1)$.
By the fact that \[
{\rm TV}(\lambda, \nu) = \sup\{\lambda(f) - \nu(f) : f: \bbR^d \to \bbR \text{ measurable with } |f| \le 1\},
\]
we have
\[|\lambda(f_k) - \nu(f_k)| \le {\rm TV}(\lambda, \nu), \quad \forall k = 1, \dots, d.\]
Therefore,
\[
\|\E(\lambda) - \E(\nu)\|^2 = \sum_{k=1}^d (\lambda(f_k) - \nu(f_k))^2 \le d \, {\rm TV}(\lambda, \nu)^2,
\]
which proves the first part.

Similarly, let $f_{k\ell}(\bt) = t_k t_\ell$.  Since $|f_{k\ell}(\bt)| \le 1$ for all $k, \ell$ and all $\bt \in B(0,1)$, we have
\[|\lambda(f_{k\ell}) - \nu(f_{k\ell})| \le {\rm TV}(\lambda, \nu), \quad \forall k, \ell = 1, \dots, d.\]
Since for any probability measure $\mu$ on $\bbR^d$,
\[\Cov(\mu) = \big(\mu(f_{k\ell}) - \mu(f_k)\mu(f_\ell) : k, \ell = 1, \dots, d\big),\]
we have
\beqn
\|\Cov(\lambda) - \Cov(\nu)\|
&\le& d \, \max_{k, \ell} \big( |\lambda(f_{k\ell}) - \nu(f_{k\ell})| + |\lambda(f_{k}) \lambda(f_{\ell}) - \nu(f_k)\nu(f_\ell)| \big) \\
&\le& d \max_{k, \ell} \big(|\lambda(f_{k\ell}) - \nu(f_{k\ell})| + |\lambda(f_{k})| |\lambda(f_{\ell}) - \nu(f_\ell)| + |\nu(f_{\ell})| |\lambda(f_{k}) - \nu(f_k)|\big) \\
&\le& 3 d \, {\rm TV}(\lambda, \nu),
\eeqn
using the fact that $|\lambda(f_k)| \le 1$ and $|\nu(f_k)| \le 1$ for all $k$.
\end{proof}

Next we compare the covariance matrix of the uniform distribution on a small piece of smooth surface with that of the uniform distribution on the projection of that piece onto a nearby tangent subspace.

\begin{lem} \label{lem:U-approx}
There is a constant $\Clref{U-approx}>0$ depending only on $d$ such that the following is true.
Take $S \in \cS_d(\kappa)$, $\rad < \frac1{\Clref{U-approx} \kappa}$ and $\bx \in \bbR^D$ such that $\dist(\bx, S) \le \rad$.
Let $\bs = P_S(\bx)$ and $T = T_S(\bs)$.
If $\bzeta$ and $\bxi$ are the means, and $\bM$ and $\bN$ are the covariance matrices of the uniform distributions on $S \cap B(\bx, \rad)$ and $T \cap B(\bx, \rad)$ respectively, then
\[\|\bzeta - \bxi\| \le \Clref{U-approx} \kappa \rad^{2},\qquad \|\bM - \bN\| \le \Clref{U-approx} \kappa \rad^{3}.\]
\end{lem}

\begin{proof}
We focus on proving the bound on the covariances, and leave the bound on the means --- whose proof is both similar and simpler --- as an exercise to the reader.
Let $T = T_S(\bs)$, $B_\rad = B(\bx, \rad)$ and $g = P_{T}$ for short.
Let $A = S \cap B_\rad$ and $A' = T \cap B_\rad$.
Let $X \sim \lambda_A$ and define $Y = g(X)$ with distribution denoted $\lambda_A^g$.  We have
\[\Cov(X) - \Cov(Y) = \frac12 \big( \Cov(X-Y, X+Y) + \Cov(X+Y, X-Y)\big),\]
where $\Cov(U, V) = \E((U-\bmu_U) (V -\bmu_V)^T)$ is the cross-covariance of random vectors $U$ and $V$ with respective means $\bmu_U$ and $\bmu_V$.  Note that by Jensen's inequality, the fact $\|\bu \bv^T\| = \|\bu\| \|\bv\|$ for any pair of vectors $\bu, \bv$, and then the Cauchy-Schwarz inequality
\[
\|\Cov(U, V)\| \le \E(\|U-\bmu_U\| \cdot \|V -\bmu_V\|) \le \E(\|U-\bmu_U\|^2)^{1/2} \cdot \E(\|V-\bmu_V\|^2)^{1/2}.
\]
Hence, letting $\bmu_X = \E X$ and $\bmu_Y = \E Y$, we have
\begin{eqnarray}
\|\Cov(\lambda_A) - \Cov(\lambda_A^g)\|
&\le& \|\Cov(X-Y, X+Y)\| \notag \\
&\le& \E\big[\|X - Y - \bmu_X + \bmu_Y\|^2\big]^{1/2} \E\big[\|X+Y - \bmu_X -\bmu_Y\|^2\big]^{1/2} \notag \\
&\le& \E\big[\|X - Y\|^2\big]^{1/2}  \left(\E\big[\|X - \bs\|^2\big]^{1/2}  + \E\big[\|Y - \bs\|^2\big]^{1/2} \right) \label{CovXY} \\
&\le& \frac{\kappa}2 \rad^2 \big(\rad + \rad) = \kappa \rad^3, \notag
\end{eqnarray}
where the third inequality is due to the triangle inequality and fact that the mean minimizes the mean-squared error, and the third to the fact that $X, Y \in B_\rad$ and \eqref{S-approx1}.

Assume $r < 1/((\Clref{TV} \vee d) \kappa)$.
Let $\lambda_{g(A)}$ denote the uniform distribution on $g(A)$.
$\lambda_A^g$ and $\lambda_{g(A)}$ are both supported on $B_\rad$, so that applying \lemref{TV-cov} with proper scaling, we get
\[\|\Cov(\lambda_A^g) - \Cov(\lambda_{g(A)})\| \le 3d \rad^2\, {\rm TV}(\lambda_A^g, \lambda_{g(A)}).\]
We know that $g$ is 1-Lipschitz, and by \lemref{proj} --- which is applicable since $\Clref{TV} \ge 3$ --- $g^{-1}$ is well-defined and is $(1 + \kappa \rad)$-Lipschitz on $B_\rad$.
Hence, by \lemref{TV-map} and the fact that $\dim(A) = d$, we have
\[
{\rm TV}(\lambda_A^g, \lambda_{g(A)}) \le 8 ((1 + \kappa \rad)^d - 1) \le 16 d \kappa \rad,
\]
using the inequality $(1+x)^d \le 1 + 2dx$, valid for any $x \in [0,1/d]$ and any $d \ge 1$.

Noting that $\lambda_{A'}$ is also supported on $B_\rad$, applying \lemref{TV-cov} with proper scaling, we get
\[\|\Cov(\lambda_{g(A)}) - \Cov(\lambda_{A'})\| \le 3d \rad^2\, {\rm TV}(\lambda_{g(A)}, \lambda_{A'}),\]
with
\[
{\rm TV}(\lambda_{g(A)}, \lambda_{A'}) \le 4 \frac{\vol(A \symd A')}{\vol(A')} \le C \kappa \rad,
\]
by \lemref{2unif} and \lemref{TV}, where $C$ depends only on $d,\kappa$.

By the triangle inequality,
\beqn
\|\bM - \bN\| &=& \|\Cov(\lambda_{A}) - \Cov(\lambda_{A'})\| \\
&\le& \|\Cov(\lambda_{A}) - \Cov(\lambda_{A}^g)\| + \|\Cov(\lambda_{A}^g) - \Cov(\lambda_{g(A)})\| + \|\Cov(\lambda_{g(A)}) - \Cov(\lambda_{A'})\| \\
&\le& \kappa \rad^3 + 48 d^2 \kappa \rad^3 + C \rad^{3}.
\eeqn
From this, we conclude.
\end{proof}

Next is a lemma on the estimation of a covariance matrix.  The result is a simple consequence of the matrix Hoeffding inequality of \cite{tropp}.  Note that simply bounding the operator norm by the Frobenius norm, and then applying the classical Hoeffding inequality \citep{hoeffding} would yield a bound sufficient for our purposes, but this is a good opportunity to use a more recent and sophisticated result.
\begin{lem}
\label{lem:cov}
Let $\bC_m$ denote the empirical covariance matrix based on an i.i.d.~sample of size $m$ from a distribution on the unit ball of $\bbR^d$ with covariance $\bSigma$.  Then 
$$\pr{\|\bC_m - \bSigma\| > t} \leq 4 d \exp\left(-\frac{mt}{16} \min\big(\frac{t}{32}, \frac{m}d\big) \right).$$
\end{lem}

\begin{proof}
Without loss of generality, we assume that the distribution has zero mean and is now supported on $B(0,2)$.  Let $\bx_1, \dots, \bx_m$ denote the sample, with $\bx_i = (x_{i,1}, \dots, x_{i,d})$.  We have
\[
\bC_m = \bC^\star_m - \frac1m \bar{\bx} \bar{\bx}^T,
\]
where
\[
\bC^\star_m := \frac1m \sum_{i=1}^m \bx_i \bx_i^T, \quad \bar{\bx} := \frac1m \sum_{i=1}^m \bx_i.
\]
Note that
\[
\|\bC_m - \bSigma\| \le \|\bC^\star_m - \bSigma\| + \frac1m \|\bar{\bx}\|^2.
\]
Applying the union bound and then Hoeffding's inequality to each coordinate --- which is in $[-2, 2]$ --- we get
\[
\P(\|\bar{\bx}\| > t) \le \sum_{j=1}^d \P(|\bar{x}_j| > t/\sqrt{d}) \le 2 d \exp\left(-\frac{m t^2}{8 d}\right).
\]
Noting that $\frac1m (\bx_i \bx_i^T - \bSigma), i=1,\dots,m,$ are independent, zero-mean, self-adjoint matrices with spectral norm bounded by $4/m$, we may apply the matrix Hoeffding inequality \citep[Th.~1.3]{tropp}, we get
\[
\pr{\|\bC^\star_m - \bSigma\| > t} \le 2 d \exp\left(-\frac{t^2}{8 \sigma^2}\right), \quad \sigma^2 := m (4/m)^2 = 16/m.
\]

Applying the union bound and using the previous inequalities, we arrive at
\beqn
\pr{\|\bC_m - \bSigma\| > t}
&\le& \pr{\|\bC^\star_m - \bSigma\| > t/2} + \pr{\|\bar{\bx}\| > \sqrt{m t/2}} \\
&\le& 2 d \exp\left(-\frac{m t^2}{512}\right) + 2 d \exp\left(-\frac{m^2 t}{16 d}\right) \\
&\le& 4 d \exp\left(-\frac{mt}{16} \min\big(\frac{t}{32}, \frac{m}d\big) \right).
\eeqn
\end{proof}

\subsubsection{Projections}
We relate below the difference of two orthogonal projections with the largest principal angle between the corresponding subspaces.
\begin{lem}
\label{lem:P-diff}
For two affine non-null subspaces $T,T'$,
\[\|P_{T} - P_{T'}\| = \begin{cases}
\sin \thetamax(T, T'), & \text{if } \dim(T) = \dim(T'), \\
1, & \text{otherwise}.
\end{cases}
\]
\end{lem}

\begin{proof}
For two affine subspaces $T, T' \subset \bbR^D$ of same dimension, let
$
\frac\pi2 \ge \theta_1 \ge \cdots \ge \theta_D \ge 0,
$
denote the principal angles between them.
By \citep[Th.~I.5.5]{MR1061154}, the singular values of $P_{T} - P_{T'}$ are $\{\sin \theta_j : j = 1, \dots, q\}$, so that $\|P_{T} - P_{T'}\| = \max_j \sin \theta_j = \sin \theta_1 = \sin \thetamax(T, T')$.
Suppose now that $T$ and $T'$ are of different dimension, say $\dim(T) > \dim(T')$.
We have $\|P_{T} - P_{T'}\| \le \|P_{T}\| \vee \|P_{T'}\| = 1$, since $P_T$ and $P_{T'}$ are orthogonal projections and therefore positive semidefinite with operator norm equal to 1.
Let $L = P_T(T')$.
Since $\dim(L) \le \dim(T') < \dim(T)$, there is $\bu \in T \cap L^\perp$ with $\bu \ne 0$.
Then $\bv^\top \bu = P_T(\bv)^\top \bu = 0$ for all $\bv \in T'$, implying that $P_{T'}(\bu) = 0$ and consequently $(P_{T} - P_{T'}) \bu = \bu$, so that $\|P_{T} - P_{T'}\| \ge 1$.
\end{proof}

The lemma below is a perturbation result for eigenspaces and widely known as the $\sin \Theta$ Theorem of \cite{MR0264450}.  See also \citep[Th.~7]{1288832} or \citep[Th.~V.3.6]{MR1061154}.
\begin{lem}[Davis and Kahan] \label{lem:davis}
Let $\bM$ be positive semi-definite with eigenvalues $\beta_1 \geq \beta_2 \geq \cdots$.  Suppose that $\Delta_d := \beta_d -\beta_{d+1} > 0$.  Then for any other positive semi-definite matrix $\bN$, 
\[
\|P^{(d)}_\bN -P^{(d)}_\bM\| \leq \frac{\sqrt{2} \|\bN -\bM\|}{\Delta_d},
\]
where $P^{(d)}_\bM$ and $P^{(d)}_\bM$ denote the orthogonal projections onto the top $d$ eigenvectors of $\bM$ and $\bN$, respectively.
\end{lem}

\subsubsection{Intersections}

We start with an elementary result on points near the intersection of two affine subspaces.

\begin{lem} \label{lem:angle}
Take any two linear subspaces $T_1, T_2 \subset \bbR^D$.  For any point $\bt_1 \in T_1 \setminus T_2$, we have
\[
\dist(\bt_1, T_2) \ge \dist(\bt_1, T_1 \cap T_2) \, \sin \thetamin(T_1, T_2).
\]
\end{lem}

\begin{proof}
We may reduce the problem to the case where $T_1 \cap T_2 = \{0\}$.  Indeed, let $\tilde{T}_1 = T_1 \cap T_2^\perp$, $\tilde{T}_2 = T_1^\perp \cap T_2$ and $\tilde{\bt}_1 = \bt_1 - P_{T_1 \cap T_2}(\bt_1)$.  Then
\[
\|\bt_1 - P_{T_2}(\bt_1)\| = \|\tilde{\bt}_1 - P_{\tilde{T}_2}(\tilde{\bt}_1)\|, \quad \|\bt_1 - P_{T_1 \cap T_2}(\bt_1)\| = \|\tilde{\bt}_1\|, \quad \sin \thetamin(T_1, T_2) = \sin \thetamin(\tilde{T}_1, \tilde{T}_2).
\]
So assume that $T_1 \cap T_2 = \{0\}$.  By \citep[Th.~10.1]{MR0094880}, the angle formed by $\bt_1$ and $P_{T_2}(\bt_1)$ is at least as large as the smallest principal angle between $T_1$ and $T_2$, which is $\thetamin(T_1, T_2)$ since $T_1 \cap T_2 = \{0\}$.     From this the result follows immediately.
\end{proof}

The following result says that a point cannot be close to two compact and smooth surfaces intersecting at a positive angle without being close to their intersection.  Note that the constant there cannot be solely characterized by $\kappa$, as it also depends on the separation between the surfaces away from their intersection.
\begin{lem} \label{lem:sep}
Suppose $S_1, S_2 \in \cS_d(\kappa)$ intersect at a strictly positive angle and that $\reach(S_1 \cap S_2) \ge 1/\kappa$.  Then there is a constant $\Clref{sep}$ such that 
\beq \label{sep}
\dist(\bx, S_1 \cap S_2) \le \Clref{sep} \max\big\{\dist(\bx, S_1), \dist(\bx, S_2) \big\}, \quad \forall \bx \in \R^D.
\eeq
\end{lem}

\begin{proof}
Assume the result is not true, so there is a sequence $(\bx_n) \subset \R^D$ such that $\dist(\bx_n, S_1 \cap S_2) > n \max_k \dist(\bx_n, S_k)$.  Because the surfaces are bounded, we may assume WLOG that the sequence is bounded.  Then $\dist(\bx_n, S_1 \cap S_2)$ is bounded, which implies $\max_k \dist(\bx_n, S_k) = O(1/n)$.  This also forces
$\dist(\bx_n, S_1 \cap S_2) \to 0$.  Indeed, otherwise there is a constant $C > 0$ and a subsequence $(\bx_{n'})$ such that $\dist(\bx_{n'}, S_1 \cap S_2) \ge C$.  Since $(\bx_{n'})$ is bounded, there is a subsequence $(\bx_{n''})$ that converges, and by the fact that $\max_k \dist(\bx_{n''}, S_k) = o(1)$, and by compactness of $S_k$, the limit is necessarily in $S_1 \cap S_2$, which is a contradiction.  So we have $\dist(\bx_n, S_1 \cap S_2) = o(1)$, implying $\max_k \dist(\bx_n, S_k) = o(1/n)$.

Assume $n$ is large enough that $\dist(\bx_n, S_1 \cap S_2) < 1/\kappa$ and let $\bs_n^k$ be the projection of $\bx_n$ onto $S_k$, and $\bs_n^\ddag$ the projection of $\bx_n$ onto $S_1 \cap S_2$.  Let $T_k = T_{S_k}(\bs_n^\ddag)$ and note that $\thetamin(T_1, T_2) \ge \theta$, where $\theta > 0$ is the minimum intersection angle between $S_1$ and $S_2$ defined in \eqref{theta}.  Let $\bt_n^k$ be the projection of $\bs_n^k$ onto $T_k$.  Assume WLOG that $\|\bt_n^1 - \bs_n^1\| \ge \|\bt_n^2 - \bs_n^2\|$.  Let $\bt_n$ denote the projection of $\bt_n^1$ onto $T_1 \cap T_2$, and then let $\bs_n = P_{S_1 \cap S_2}(\bt_n)$.  

By assumption, we have
\beq \label{sep1}
n \max_k \|\bx_n - \bs_n^k\| \le \|\bx_n - \bs_n^\ddag\| = o(1).
\eeq
We start with the RHS:
\beq \label{sep1-1}
\|\bx_n - \bs_n^\ddag\|
= \min_{\bs \in S_1 \cap S_2} \|\bx_n - \bs\|
\le \|\bx_n - \bs_n\|,
\eeq
and first show that $\|\bx_n - \bs_n\| = o(1)$ too.
We use the triangle inequality multiple times in what follows.  We have
\beq \label{sep2}
\|\bx_n - \bs_n\| \le \|\bx_n - \bs_n^1\| + \|\bs_n^1 - \bt_n^1\| + \|\bt_n^1 - \bt_n\| + \|\bt_n - \bs_n\|.
\eeq
From \eqref{sep1}, $ \|\bx_n - \bs_n^1\| = o(1)$ and $ \|\bx_n - \bs_n^\ddag\| = o(1)$, and so that by \eqref{S-approx1},
\beq \label{sep3}
\|\bs_n^1 -\bt_n^1\| \le \kappa \|\bs_n^1 - \bs_n^\ddag\|^2 \le 2 \kappa ( \|\bs_n^1 - \bx_n\|^2 + \|\bx_n - \bs_n^\ddag\|^2) = o(1).
\eeq
We also have
\beq \label{sep3-1}
\|\bt_n^1 - \bt_n\| = \min_{\bt \in T_1 \cap T_2} \|\bt_n^1 - \bt\| \le \|\bt_n^1 - \bs_n^\ddag\| \le \|\bt_n^1 -\bs_n^1\| + \|\bs_n^1 -\bx_n\| + \|\bx_n -\bs_n^\ddag\| = o(1),
\eeq
where the first inequality comes from $\bs_n^\ddag \in T_1 \cap T_2$.
Finally,
\[\|\bt_n - \bs_n\| = \min_{\bs \in S_1 \cap S_2} \|\bt_n - \bs\| \le \|\bt_n - \bs_n^\ddag\| \le \|\bt_n - \bt_n^1\| + \|\bt_n^1 -\bs_n^\ddag\| = o(1),\]
where the first inequality comes from $\bs_n^\ddag \in S_1 \cap S_2$.

We now proceed.  The last upper bound is rather crude.  Indeed, we use  \eqref{S-approx3} for $S = S_1 \cap S_2$ and $\bs = \bs_n^\ddag$, noting that $T_{S_1 \cap S_2}(\bs_n^\ddag) = T_1 \cap T_2$ and $\|\bt_n - \bs_n^\ddag\|  = o(1)$, and get
\[\|\bt_n - \bs_n\| \le \kappa \|\bt_n - \bs_n^\ddag\|^2 \le \kappa (\|\bt_n - \bs_n\| + \|\bs_n - \bx_n\| + \|\bx_n - \bs_n^\ddag\|)^2.
\]
We have $\|\bx_n - \bs_n^\ddag\| = \|\bx_n - P_{S_1 \cap S_2}(\bx_n)\| \le \|\bx_n - \bs_n\|$ because $\bs_n \in T_1 \cap T_2$.
This leads to
\beq \label{sep4}
\|\bt_n - \bs_n\| \le \kappa (\|\bt_n - \bs_n\| + 2\|\bs_n - \bx_n\|)^2 \le 4 \kappa \|\bx_n - \bs_n\|^2,
\eeq
eventually, since $\|\bt_n - \bs_n\| = o(1)$.

Combining \eqref{sep2}, \eqref{sep3} and \eqref{sep4}, we get
\[\|\bx_n - \bs_n\| \le \|\bx_n - \bs_n^1\| + O(\|\bx_n - \bs_n^1\|^2 + \|\bx_n - \bs_n\|^2) + \|\bt_n^1 - \bt_n\| + O(\|\bx_n - \bs_n\|^2),
\]
which leads to
\beq \label{sep5}
\|\bx_n - \bs_n\| \le 2 \|\bx_n - \bs_n^1\| + 2 \|\bt_n^1 - \bt_n\|,
\eeq
when $n$ is large enough.
Using this bound in \eqref{sep1} combined with \eqref{sep1-1}, we get
\[\|\bt_n^1 - \bt_n\| \ge \frac{n-2}2 \max_k \|\bx_n - \bs_n^k\|.\]
We then have
\beqn
\max_k \|\bx_n - \bs_n^k\| &\ge& \frac12 \|\bs_n^1 - \bs_n^2\| \\
&\ge& \frac12 (\|\bt_n^1 - \bt_n^2\| - \|\bs_n^1 - \bt_n^1\| - \|\bs_n^2 - \bt_n^2\|) \\
&\ge& \frac12 \dist(\bt_n^1, T_2) - \|\bs_n^1 - \bt_n^1\|,
\eeqn
with
\[\|\bs_n^1 - \bt_n^1\| = O(\|\bx_n - \bs_n^1\|^2 + \|\bx_n - \bs_n^\ddag\|^2) = O(\|\bx_n - \bs_n\|^2) = O(\|\bt_n^1 - \bt_n\|^2),\]
due (in the same order) to \eqref{sep3}, \eqref{sep1}-\eqref{sep1-1}, and \eqref{sep5}.
Recalling that $\|\bt_n^1 - \bt_n\| = \dist(\bt_n^1, T_1 \cap T_2)$, we conclude that
\[
\dist(\bt_n^1, T_2) = O(1/n) \dist(\bt_n^1, T_1 \cap T_2) + O(1) \dist(\bt_n^1, T_1 \cap T_2)^2.
\]
However, by \lemref{angle}, $\dist(\bt_n^1, T_2) \ge (\sin \theta) \dist(\bt_n^1, T_1 \cap T_2)$, so that dividing by $\dist(\bt_n^1, T_2)$ above leads to $1 = O(1/n) + O(1) \dist(\bt_n^1, T_2)$, which is in contradiction with the fact that $\dist(\bt_n^1, T_2) \le \|\bt_n^1 - \bt_n\| = o(1)$, established in \eqref{sep3-1}.
\end{proof}

\subsubsection{Covariances near an intersection}

We look at covariance matrices near an intersection.
We start with a continuity result.

\begin{lem} \label{lem:cov-cont}
Let $T_1$ and $T_2$ be two linear subspaces of same dimension $d$.
For $\bx \in T_1$, denote by $\bSigma(\bx)$ the covariance matrix of the uniform distribution over $B(\bx, 1) \cap (T_1 \cup T_2)$.
Then, for all $\bx, \by \in T_1$,
\[
\|\bSigma(\bx) - \bSigma(\by)\| \le
\begin{cases}
5 d \, \|\bx - \by\|, & \text{if } d \ge 2, \\
\sqrt{6 \|\bx - \by\|}, & \text{if } d = 1.
\end{cases}
\]
\end{lem}

\begin{proof}
Since, by \lemref{unif-cov}, $\bSigma(\bx) = c P_{T_1}$ for all $\bx \in T_1$ such that $\dist(\bx, T_2) \ge 1$, we may assume that $\dist(\bx, T_1) < 1$ and $\dist(\by, T_1) < 1$.
Let $d = \dim(T_1) = \dim(T_2)$ and $A^j_\bx = B(\bx, 1) \cap T_j$ for any $\bx$ and $j = 1,2$.
By \lemref{TV-cov} and then \lemref{2unif}, we have
\beqn
\|\bSigma(\bx) - \bSigma(\by)\|
&=& \|\Cov(\lambda_{A^1_\bx \cup A^2_\bx}) - \Cov(\lambda_{A^1_\by \cup A^2_\by})\| \\
&\le& {\rm TV}(\lambda_{A^1_\bx \cup A^2_\bx}, \lambda_{A^1_\by \cup A^2_\by}) \\
&\le& 4 \frac{\vol\big((A^1_\bx \cup A^2_\bx) \symd (A^1_\by \cup A^2_\by)\big)}{\vol\big((A^1_\bx \cup A^2_\bx) \cup (A^1_\by \cup A^2_\by)\big)} \\
&\le& 4 \frac{\vol(A^1_\bx \symd A^1_\by) + \vol(A^2_\bx \symd A^2_\by)}{\vol(A^1_\bx)}.
\eeqn

Note that $A^1_\bx$ is the unit-radius ball of $T_1$ centered at $\bx$, while $A^2_\bx$ is the ball of $T_2$ centered at $\bx_2 := P_{T_2}(\bx)$ and of radius $\eta := \sqrt{1 - \|\bx - \bx_2\|^2}$.
Similarly, $A^1_\by$ is the unit-radius ball of $T_1$ centered at $\by$, while $A^2_\by$ is the ball of $T_2$ centered at $\by_2 := P_{T_2}(\by)$ and of radius $\delta := \sqrt{1 - \|\by - \by_2\|^2}$.
Therefore, applying \lemref{2ball-vol}, we get
\[
\frac{\vol(A^1_\bx \symd A^1_\by)}{2 \vol(A^1_\bx)} \le 1 - (1 - t)_+^d,
\]
and assuming WLOG that $\delta \le \eta$, and after proper scaling, we get
\[\frac{\vol(A^2_\bx \symd A^2_\by)}{2 \vol(A^1_\bx)}
\le \zeta := \eta^d - (\eta - t_2)_+^d \wedge \delta^d,\]
where $t := \|\bx - \by\|$ and $t_2 := \|\bx_2 - \by_2\|$ --- note that $t_2 \le t$ by the fact that $P_{T_2}$ is 1-Lipschitz.

We have $1 - (1 - t)_+^d \le dt$.
This is obvious when $t \ge 1$, while when $t \le 1$ it is obtained using the fact that, for any $0 \le s < t \le 1$,
\beq \label{diff_d}
t^d - s^d = (t-s)(t^{d-1} + st^{d-2} + \cdots + s^{d-2}t + s^{d-1}) \le d t^{d-1} (t-s) \le d (t - s).
\eeq
For the second ratio, we consider several cases.
\bitem
\item When $\eta \le t_2$, then $\zeta = \eta^d \le \eta \le t_2 \le t$.
\item When $t_2 < \eta \le t_2 + \delta$, then $\zeta = \eta^d - (\eta - t_2)^d \le d t_2 \le d t$.
\item When $\eta \ge t_2 + \delta$ and $d \ge 2$, we have
\beqn
\zeta &=& \eta^d - \delta^d \le d \eta (\eta - \delta) \le d (\eta^2 - \delta^2) \\
&=& d (\|\by - \by_2\|^2 - \|\bx-\bx_2\|^2) \\
&=& d (\|\by - \by_2\|+ \|\bx-\bx_2\|)(\|\by - \by_2\|- \|\bx-\bx_2\|) \\
&\le& 2 d (t + t_2) \le 4d t,
\eeqn
where the triangle inequality was applied in the last inequality, in the form of
\[
\|\by - \by_2\| \le \|\by - \bx\| + \|\bx-\bx_2\| + \|\bx_2 - \bx\| = \|\bx-\bx_2\| + t + t_2.
\]
\item When $\eta \ge t_2 + \delta$ and $d = 1$, we have
\[
\zeta = \eta - \delta \le \sqrt{\|\by - \by_2\|- \|\bx-\bx_2\|} \le \sqrt{t + t_2} \le \sqrt{2t},
\]
using the same triangle inequality and the fact that, for any $0 \le s < t \le 1$,
\[
0 \le \sqrt{1-s} - \sqrt{1-t} = \frac{t-s}{\sqrt{1-s} + \sqrt{1-t}} \le \frac{t-s}{\sqrt{1-t + t-s}} \le \frac{t-s}{\sqrt{t-s}} = \sqrt{t-s}.
\]
\eitem
When $d\ge 2$, we can therefore bound $\|\bSigma(\bx) - \bSigma(\by)\|$ by $dt + 4dt = 5dt$, and when $d=1$, we bound that by $t + \sqrt{2t} \le \sqrt{6t}$.
\end{proof}

The following is in some sense a converse to \lemref{cov-cont}, in that we lower-bound the distance between covariance matrices near an intersection of linear subspaces.
Note that  the covariance matrix does not change when moving parallel to the intersection; however, it does when moving perpendicular to the intersection.

\begin{lem} \label{lem:intersect}
Let $T_1$ and $T_2$ be two linear subspaces of same dimension with $\thetamin(T_1,T_2) \ge \theta_0 > 0$.
Fix a unit norm vector $\bv \in T_1 \cap (T_1 \cap T_2)^\perp$.
With $\bSigma(h \bv)$ denoting the covariance of the uniform distribution over $B(h \bv, 1) \cap (T_1 \cup T_2)$, we have
\[
\inf_{h} \ \sup_{\ell} \|\bSigma(h \bv) - \bSigma(\ell \bv)\| \ge 1/\Clref{intersect},
\]
where the infimum is over $0 < h < 1/\sin \theta_0$ and the supremum over $\max(0, h -1/2) \le \ell \le \min(1/\sin \theta_0, h +1/2)$, and $\Clref{intersect} >0$ depends only on $d$ and $\theta_0$.
\end{lem}

\begin{proof}
If the statement of the lemma is not true, there are subspaces $T_1$ and $T_2$ of same dimension $d$, a unit length vector $\bv \in T_1  \cap (T_1 \cap T_2)^\perp$ and $0 \le h \le 1/\sin \theta_0$, such that
\beq \label{hyp-intersect}
\text{$\bSigma(\ell \bv) = \bSigma(h \bv)$ for all $\max(0, h -1/2) \le \ell \le \min(1/\sin \theta_0, h +1/2)$.}
\eeq
By projecting onto $(T_1 \cap T_2)^\perp$, we may assume that $T_1 \cap T_2 = 0$ without loss of generality.
Let $\theta = \angle(\bv,T_2)$ and note that $\theta \ge \theta_0$ since $T_1 \cap T_2 = 0$.
Define $\bu = (\bv - P_{T_2} \bv)/\sin\theta$ and also $\bw = P_{T_2} \bv/\cos\theta$ when $\theta < \pi/2$, and $\bw \in T_2$ is any vector perpendicular to $\bv$ when $\theta = \pi/2$.
$B(h \bv, 1) \cap T_1$ is the $d$-dimensional ball of $T_1$ of radius 1 and center $h \bv$, while --- using Pythagoras theorem --- $B(h \bv, 1) \cap T_2$ is the $d$-dimensional ball of $T_2$ of radius $t := (1 - (h \sin \theta)^2)^{1/2}$ and center $(h \cos \theta) \bw$.
Let $X$ be drawn from the uniform distribution over $B(h \bv, 1) \cap (T_1 \cup T_2)$, while $X_0$ and $X_0'$ are independently drawn from the uniform distributions over the unit balls of $T_1$ and $T_2$, respectively.
By \lemref{unif-cov}, $\Cov(X_0) = c P_{T_1}$ and $\Cov(X_0') = c P_{T_2}$ where $c := 1/(d+2)$.
Also, let $\xi$ be Bernoulli with parameter $\alpha$, where
\[
\alpha := \frac{\vol(B(h \bv, 1) \cap T_1)}{\vol(B(h \bv, 1) \cap (T_1 \cup T_2))} = \frac{\vol(B(h \bv, 1) \cap T_1)}{\vol(B(h \bv, 1) \cap T_1) + \vol(B((h \cos \theta) \bw, t) \cap T_2)} = \frac{1}{1 + t^d}.
\]
We have
\[
X \sim \xi \big(h \bv + X_0\big) + (1-\xi ) \big((h \cos \theta) \bw + t X_0'\big).
\]
A straightforward calculation, or an application of the law of total covariance, leads to
\beq \label{cov-mixture}
\Cov(X) = \E(\xi) \Cov(X_0) + \E(1-\xi) t^2 \Cov(X_0') + \Var(\xi) h^2 (\bv - (\cos \theta)\bw)(\bv - (\cos \theta)\bw)^\top,
\eeq
which simplifies to
\[
\bSigma(h \bv) = c \alpha P_{T_1} + c (1-\alpha) t^2 P_{T_2} + \alpha (1-\alpha ) (1 - t^2) \bu \bu^\top,
\]
using the fact that $\bv - (\cos \theta)\bw = (\sin\theta) \bu$ and the definition of $t$.
Let $\theta_1 = \thetamax(T_1,T_2)$ and let $\bv_1 \in T_1$ be of unit length and such that $\angle(\bv_1, T_2) = \theta_1$.
Then for any $0 \le h, \ell \le 1/\sin\theta_0$, we have
\beq \label{hyp-intersect2}
\|\bSigma(h \bv) - \bSigma(\ell \bv)\| \ge |\bv_1^\top \bSigma(h \bv)\bv_1 - \bv_1^\top \bSigma(\ell \bv)\bv_1| = |f(t_h) - f(t_\ell)|,
\eeq
where $t_h := (1 - (h\sin \theta)^2)^{1/2}$ and
\beqn
f(t)
&=& \frac{c}{1+t^d} + \frac{c t^{d+2} (\cos \theta_1)^2}{1+t^d} + \frac{t^d(1 - t^{2}) (\bu^\top \bv_1)^2}{(1+t^d)^2}.
\eeqn
It is easy to see that the interval
\[I_h = \{t_\ell : (h -1/2)_+ \le \ell \le (1/\sin \theta_0) \wedge (h +1/2) \}\]
is non empty.
Because of \eqref{hyp-intersect} and \eqref{hyp-intersect2}, $f(t)$ is constant over $t \in I_h$, but this is not possible since $f$ is a rational function not equal to a constant and therefore cannot be constant over an interval of positive length.
\end{proof}

We now look at the eigenvalues of the covariance matrix.
\begin{lem} \label{lem:eigen-inter}
Let $T_1$ and $T_2$ be two linear subspaces of same dimension $d$.
For $\bx \in T_1$, denote by $\bSigma(\bx)$ the covariance matrix of the uniform distribution over $B(\bx, 1) \cap (T_1 \cup T_2)$.
Then, for all $\bx \in T_1$,
\beq
c \big(1 - (1 - \delta^2(\bx))_+^{d/2}\big) \le \beta_d(\bSigma(\bx)), \qquad \beta_1(\bSigma(\bx)) \le c + \delta(\bx) (1 - \delta^2(\bx))_+^{d/2},  \label{eigen-inter1}
\eeq
\beq
\frac{c}{8} (1 - \cos \thetamax(T_1, T_2))^2 (1 - \delta^2(\bx))_+^{d/2+1} \le \beta_{d+1}(\bSigma(\bx)) \le (c + \delta^2(\bx)) (1 - \delta^2(\bx))_+^{d/2},  \label{eigen-inter2}
\eeq
where $c := 1/(d+2)$ and $\delta(\bx) := \dist(\bx, T_2)$.
\end{lem}

\begin{proof}
As in \eqref{cov-mixture}, we have
\beq \label{Sigma}
\bSigma(x) = \alpha c  P_{T_1} + (1-\alpha)  c t^2 P_{T_2} + \alpha (1-\alpha) (\bx - \bx_2) (\bx - \bx_2)^\top,
\eeq
where $\bx_2 := P_{T_2}(\bx)$ and $\alpha := (1+t^d)^{-1}$ with $t := (1 - \delta^2(\bx))_+^{1/2}$.
Because all the matrices in this display are positive semidefinite, we have
\[
\beta_d(\bSigma(\bx)) \ge \alpha c  \|P_{T_1}\| = \alpha c,
\]
with $\alpha \ge 1-t^d$.
And because of the triangle inequality, we have
\[
\beta_1(\bSigma(\bx)) \le  \alpha c  \|P_{T_1}\| + (1-\alpha)  c t^2 \|P_{T_2} \| + \alpha (1-\alpha) \|\bx - \bx_2\|^2 \le c + \alpha (1-\alpha) \delta^2(\bx),
\]
with $\alpha (1-\alpha) \le t^d$.
Hence, \eqref{eigen-inter1} is proved.

For the upper bound in \eqref{eigen-inter2}, by Weyl's inequality \citep[Cor.~IV.4.9]{MR1061154} and the fact that $\beta_{d+1}(P_{T_1}) = 0$, and then the triangle inequality, we get
\beqn
\beta_{d+1}(\bSigma(\bx)) &\le& \|\bSigma(\bx) - \alpha c P_{T_1}\| \\
&\le& c (1-\alpha) t^2 \|P_{T_2}\| + \alpha (1-\alpha) \delta^2(\bx) \\
&\le& (1-\alpha) ( c + \delta^2(\bx)),
\eeqn
and we then use the fact that $1 -\alpha \le t^d$.
For the lower bound, let $\theta_1 \ge \theta_2 \ge \cdots \ge \theta_d$ denote the principal angles between $T_1$ and $T_2$.
By definition of principal angles, there are orthonormal bases for $T_1$ and $T_2$, denoted $\bv_1, \dots, \bv_d$ and $\bw_1, \dots, \bw_d$, such that $\bv_j^\top \bw_k = \1_{j=k} \cdot \cos \theta_j$.
Take $\bu \in {\rm span}(\bv_1, \dots, \bv_d, \bw_1)$, that is, of the form $\bu = a \bv_1 + \bv + b \bw_1$, with $\bv \in {\rm span}(\bv_2, \dots, \bv_d)$.
Since $P_{T_1} = \bv_1 \bv_1^\top + \cdots + \bv_d \bv_d^\top$ and $P_{T_2} = \bw_1 \bw_1^\top + \cdots + \bw_d \bw_d^\top$, we have
\beqn
\frac1{c} \bu^\top \bSigma(\bx) \bu
&\ge& \alpha (a^2 + \|\bv\|^2 + 2 a b \cos \theta_1 + b^2 \cos^2 \theta_1 ) + (1-\alpha) t^2 (b^2 + 2 a b \cos \theta_1 + a^2 \cos^2 \theta_1) \\
&=& \alpha (a + b \cos \theta_1)^2 + (1-\alpha) t^2 (a \cos \theta_1 + b)^2 + \alpha (1 - a^2 - b^2),
\eeqn
assuming $\|\bu\|^2 = a^2 + \|\bv\|^2 + b^2 = 1$.
If $|a| \vee |b| \le 1/2$, then the RHS $\ge \alpha/2 \ge 1/4$.
Otherwise, the RHS $\ge (1-\alpha)t^2 (1-\cos\theta_1)^2/4$, using the fact that $\alpha \ge 1 - \alpha \ge (1-\alpha)t^2$.
Hence, by the Courant-Fischer theorem \citep[Cor.~IV.4.7]{MR1061154}, we have
\[
\beta_{d+1}(\bSigma(\bx)) \ge \frac{c}4 (1-\alpha)t^2(1-\cos\theta_1)^2,
\]
with $1-\alpha \ge t^d/2$.
This proves \eqref{eigen-inter2}.
\end{proof}

Below is a technical result on the covariance matrix of the uniform distribution on the intersection of a ball and the union of two smooth surfaces, near where the surfaces intersect.
It generalizes \lemref{U-approx}.

\begin{lem} \label{lem:cov-inter}
Let $S_1, S_2 \in \cS_d(\kappa)$ intersecting at a positive angle, with $\reach(S_1 \cap S_2) \ge 1/\kappa$.
Then there is a constant $\Clref{cov-inter} \ge 3$ such that the following holds.
Fix $\rad < 1/\Clref{cov-inter}$, and for $\bs \in S_1$ with $\dist(\bs, S_2) \le \rad$, let $\bC(\bs)$ and $\bSigma(\bs)$ denote the covariance matrices of the uniform distributions over $B(\bs, \rad) \cap (S_1 \cup S_2)$  and $B(\bs, \rad) \cap (T_1 \cup T_2)$, where $T_1 := T_{S_1}(\bs)$ and $T_2 := T_{S_2}(P_{S_2}(\bs))$.
Then
\beq \label{cov-inter}
\|\bC(\bs) - \bSigma(\bs)\| \le \Clref{cov-inter} \, \rad^{3}.
\eeq
\end{lem}

\begin{proof}
Below $C$ denotes a positive constant depending only on $S_1$ and $S_2$ that increases with each appearance.
We note that it is enough to prove the result when $\rad$ is small enough.
Take $\bs \in S_1$ such that $\delta := \dist(\bs, S_2) \le \rad$ and let $\bs_2 = P_{S_2}(\bs)$ --- note that $\|\bs - \bs_2\| = \delta$.
Let $B_\rad$ be short for $B(\bs, \rad)$ and define $A_k = B_\rad \cap S_k$, $\bmu_k = \E(\lambda_{A_k})$ and $\bD_k = \Cov(\lambda_{A_k})$, for $k=1,2$.
As in \eqref{cov-mixture}, we have
\[
\bC(\bs) = \alpha \bD_1 + (1-\alpha) \bD_2 + \alpha(1-\alpha) (\bmu_1 -\bmu_2)(\bmu_1 -\bmu_2)^\top,
\]
where
\[
\alpha := \frac{\vol(A_1)}{\vol(A_1) + \vol(A_2)}.
\]
Let $T_1 = T_{S_1}(\bs)$ and $T_2 = T_{S_2}(\bs_2)$, and define $A_k' = B_\rad \cap T_k$, so that $B_\rad \cap (T_1 \cup T_2) = A_1' \cup A_2'$.
Note that $\E(\lambda_{A_1'}) = \bs$ and $\E(\lambda_{A_2'}) = \bs_2$, and by \lemref{unif-cov}, $\bD_1' := \Cov(\lambda_{A_1'}) = c \rad^2 P_{T_1}$ and $\bD_2' := \Cov(\lambda_{A_2'}) = c (\rad^2 -\delta^2)P_{T_2}$, where $c := 1/(d+2)$.
As in \eqref{Sigma}, we have
\[
\bSigma(\bs) = \alpha' \bD_1' + (1- \alpha') \bD_2' + \alpha' (1-\alpha') (\bs -\bs_2)(\bs -\bs_2)^\top,
\]
where
\[
\alpha' := \frac{\vol(A_1')}{\vol(A_1') + \vol(A_2')}.
\]
Since $|\alpha' (1-\alpha') - \alpha (1-\alpha) | \le |\alpha' - \alpha|$, we have
\beqn
\|\bC(\bs) - \bSigma(\bs)\| &\le& |\alpha' - \alpha| \big(\|\bD_1'\| + \|\bD_2'\| + \|\bs -\bs_2\|^2 \big) \\
&& + \ \alpha \|\bD_1 - \bD_1'\| + (1- \alpha) \|\bD_2 - \bD_2'\| + \alpha (1-\alpha) 4 \rad \big(\|\bmu_1 - \bs\| + \|\bmu_2 - \bs_2\|\big) \\[.05in]
&\le& (2 c +1) \rad^2 |\alpha' - \alpha| \\
&& + \ \|\bD_1 - \bD_1'\| \vee \|\bD_2 - \bD_2'\| + 2 \rad \big(\|\bmu_1 - \bs\| \vee \|\bmu_2 - \bs_2\|\big),
\eeqn
using the triangle inequality multiple times, and in the first inequality we used the fact that
\[
\|\bv \bv^\top - \bw \bw^\top\| \le \|(\bv -\bw) \bv^\top\| + \|\bw (\bv - \bw)^\top\| \le (\|\bv\| + \|\bw\|) \|\bv - \bw\|,
\]
for any two vectors $\bv, \bw \in \bbR^D$.
Assuming that $\kappa \rad \le 1/\Clref{U-approx}$, by \lemref{U-approx}, we have $\|\bmu_1 - \bs\| \vee \|\bmu_2 - \bs_2\| \le \Clref{U-approx} \kappa \rad^{2}$ and $\|\bD_1 - \bD_1'\| \vee \|\bD_2 - \bD_2'\| \le \Clref{U-approx} \kappa \rad^{3}$.
Assuming that $\kappa \rad \le 1/3$, $P_{T_k}^{-1}$ is well-defined and $(1+\kappa \rad)$-Lipschitz on $S_k \cap B_\rad$.
And being an orthogonal projection, $P_{T_k}$ is 1-Lipschitz .
Hence, applying \lemref{Lip-vol}, we have
\[
1 \le \frac{\vol(A_k)}{\vol(P_{T_k}(A_k))} \le 1+\kappa \rad, \quad k = 1,2.
\]
Then by \lemref{TV},
\[
1 - \Clref{TV} \kappa \rad \le \frac{\vol(P_{T_k}(A_k))}{\vol(A_k')} \le 1 + \Clref{TV} \kappa \rad, \quad k = 1,2.
\]
So we get
\[
1 - C \rad \le \frac{\vol(A_k)}{\vol(A_k')} \le 1 + C \rad, \quad k = 1,2.
\]
Since for all $a,b,a',b' > 0$ we have
\begin{eqnarray}
\left| \frac{a}{a+b} - \frac{a'}{a'+b'} \right| &\le& \frac{|a - a'| \vee |b - b'|}{(a+b) \vee (a'+b')} \label{2frac} \\
&\le& \big|1 - a/a'| \vee |1 - b/b'|, \notag
\end{eqnarray}
we get
\[
|\alpha - \alpha'| \le C \rad.
\]
Hence,
\[\|\bC(\bs) - \bSigma(\bs)\| \le C \rad^{3},\]
so we are done with the proof.
\end{proof}

\subsection{Performance guarantees for \algref{cov}}

We deal with the case where there is no noise, that is, $\tau = 0$ in \eqref{data-point}, so that the data points are $\bs_1, \dots, \bs_N$, sampled exactly on $S_1 \cup S_2$ according to the uniform distribution.
We explain how things change when there is noise, meaning $\tau > 0$, in \secref{noise}.

Let $\Xi_i = \{j \neq i: \bs_j \in N_\rad(\bs_i)\}$, with (random) cardinality $N_i = |\Xi_i|$.
When there is no noise, $\bC_i$ is the sample covariance of $\{\bs_j : j \in \Xi_i\}$.
For $i \in [n]$, let $K_i = 1$ if $\bs_i \in S_1$ and $=2$ otherwise, and let $T_i = T_{S_{K_i}}(\bs_i)$, which is the tangent subspace associated with data point $\bs_i$.  Given $N_i$, $\{\bs_j : j \in \Xi_i\}$ are uniformly distributed on $S_{K_i} \cap B(\bs_i, \rad)$, and applying \lemref{cov} with rescaling, we get that for any $t > 0$
\[
\pr{\|\bC_i - \E \bC_i\| > \rad^2 t \, \big| \, N_i} \leq 4 d \exp\left(-\frac{N_i t}{\Clref{cov}} \min\big(t, \frac{N_i}d\big) \right),
\]
for an absolute constant $\Clref{cov} \ge 1$.
We may assume that $\rad < 1/(\Clref{N-size} \kappa)$ and let $n_\star := n \rad^d/\Clref{N-size}$.
We assume throughout that $\rad$ is large enough that $n_\star \ge d$, for otherwise the result is void since the probability lower bound stated in \thmref{main} is negative.
Using \lemref{N-size}, for any $t < 1$,
\beqn
\pr{\|\bC_i - \E \bC_i\| > \rad^2 t} &\leq& \pr{\|\bC_i - \E \bC_i\| > \rad^2 t \, \big| \, N_i \ge n_\star} + \pr{N_i < n_\star} \\
&\leq& 4 d \exp(- n_\star t^2/\Clref{cov}) + \Clref{N-size} n \exp(- n_\star) \\
&\le& (4 d + \Clref{N-size}) n \exp(- n_\star t^2/\Clref{cov}).
\eeqn

Define $\bSigma_i$ as the covariance of the uniform distribution on $T_i \cap B(\bs_i, \rad)$.  Let
\[
I_\star = \{i: K_j = K_i, \, \forall j \in \Xi_i\},
\]
or equivalently,
\[
I_\star^c = \{i: \exists j \text{ s.t. } K_j \ne K_i \text{ and } \|\bs_j - \bs_i\| \le \rad\}.
\]
By definition, $I_\star$ indexes the points whose neighborhoods do not contain points from the other cluster.
Applying \lemref{U-approx}, this leads to
\beq \label{cov-approx}
\|\E \bC_i - \bSigma_i\| \le \Clref{U-approx} \kappa \rad^3, \quad \forall i \in I_\star.
\eeq

Define the events
\[
\Omega_1 = \bigcup_{k = 1}^2 \big\{\forall \bs \in S_k: \, \# \{i: K_i = k \text{ and } \bs_i \in B(\bs, \rad/C_\Omega)\} > n_\star \big\},
\]
where $C_\Omega := 100 d^2 \Clref{intersect}^2$, and
\[
\Omega_2 = \left\{\|\bC_i - \E \bC_i\| \le \rad^2 t, \text{ for all } i \in [n]\right\},
\]
and their intersection $\Omega = \Omega_1 \cap \Omega_2$, where $t < 1$ will be determined later.  Note that, under $\Omega_1$, $N_i \ge n_\star$.
Applying the union bound,
\beqn
\P(\Omega^c) &\le& \P(\Omega_1^c) + \P(\Omega_2^c) \\
&\le& \Clref{N-size} n \exp(- n_\star) + n (4 d + \Clref{N-size}) \exp(- n_\star t^2/\Clref{cov}) \\
&\le& p_\Omega := (4 d + 2 \Clref{N-size}) n \exp(- n_\star t^2/\Clref{cov}).
\eeqn
Assuming that $\Omega$ holds, by the triangle inequality, \eqref{cov-noise} and \eqref{cov-approx}, we have
\beq \label{keybound}
\|\bC_i - \bSigma_i\| \le \|\bC_i - \E \bC_i\| + \|\E \bC_i - \bSigma_i\| \le \zeta \rad^2, \quad \forall i \in I_\star,
\eeq
where
\beq \label{zeta}
\zeta := t + \Clref{U-approx} \kappa \rad.
\eeq

The inequality \eqref{keybound} leads, via the triangle inequality, to the decisive bound
\beq \label{decisive}
\|\bC_i - \bC_j\| \le \|\bSigma_i - \bSigma_j\| + 2\zeta \rad^2, \qquad \forall i, j \in I_\star.
\eeq

Take $i, j \in I_\star$ such that $K_i = K_j$ and $\|\bs_i -\bs_j\| \le \eps$.  Then by \lemref{unif-cov} and \lemref{P-diff}, property \eqref{T-diff} and the fact that $\sin(2\theta) \le 2 \sin \theta$ for all $\theta$, and the triangle inequality, we have
\beq \label{eta-choice1}
\frac1{c \rad^2} \|\bSigma_i - \bSigma_j\| = \sin \thetamax(T_i, T_j) \le 2 \kappa \|\bs_i - \bs_j\| \le 2 \kappa\eps,
\eeq
where $c := 1/(d+2)$.
This implies that
\beq \label{eta-choice1S}
\frac1{\rad^2} \|\bC_i - \bC_j\| \le 2 c \kappa \eps + 2 \zeta.
\eeq
Therefore,  if $\eta > 2 c \kappa \eps + 2\zeta$, then any pair of points indexed by $i, j \in I_\star$ from the same cluster and within distance $\eps$ are direct neighbors in the graph built by \algref{cov}.

Take $i, j \in I_\star$ such that $K_i \ne K_j$ and $\|\bs_i -\bs_j\| \le \eps$.  By \lemref{sep},
\[
\max\big[ \dist(\bs_i, S_1 \cap S_2), \dist(\bs_j, S_1 \cap S_2) \big] \le \Clref{sep} \|\bs_i -\bs_j\|.
\]
Let $\bz$ be the mid-point of $\bs_i$ and $\bs_j$.  By convexity and the display above,
\[
\dist(\bz, S_1 \cap S_2) \le \frac12 \dist(\bs_i, S_1 \cap S_2) + \frac12 \dist(\bs_j, S_1 \cap S_2) \le \Clref{sep} \eps.
\]
Assuming $\Clref{sep} \eps < 1/\kappa$, let $\bs = P_{S_1 \cap S_2} (\bz)$.  Then, by the triangle inequality again,
\[
\max\big[ \|\bs - \bs_i\| , \|\bs - \bs_j\| \big] \le \dist(\bz, S_1 \cap S_2) + \frac12 \|\bs_i - \bs_j\| \le \Clref{sep} \eps + \frac12 \eps \le (\Clref{sep}+1) \eps.
\]
Let $T'_i$ denote the tangent subspace of $S_{K_i}$ at $\bs$ and let $\bSigma'_i$ be the covariance of the uniform distribution over $T'_i \cap B(\bs, \rad)$.  Define $T'_j$ and $\bSigma'_j$ similarly.  Then, as in \eqref{eta-choice1} we have
\[
\frac1{c \rad^2} \|\bSigma_i - \bSigma'_i\| \le \kappa \|\bs_i - \bs\| \le \kappa (\Clref{sep}+1) \eps,
\]
and similarly,
\[
\frac1{c \rad^2} \|\bSigma_j - \bSigma'_j\| \le \kappa (\Clref{sep}+1) \eps.
\]
\def\thetaS{\theta_{\scriptscriptstyle S}}
Moreover, by \lemref{unif-cov} and \lemref{P-diff},
\[
\frac1{c \rad^2} \|\bSigma'_i - \bSigma'_j\| = \sin \thetamax(T'_i, T'_j) \ge \sin \thetaS,
\]
where $\thetaS$ is short for $\theta(S_1, S_2)$.
Hence, by the triangle inequality,
\beq \label{eta-choice2}
\frac1{c \rad^2} \|\bSigma_i - \bSigma_j\| \ge \sin \thetaS - 2 \kappa (\Clref{sep}+1) \eps,
\eeq
and then
\beq \label{eta-choice2S}
\frac1{\rad^2} \|\bC_i - \bC_j\| \ge c \sin \thetaS - 2c \kappa (\Clref{sep}+1) \eps - 2 \zeta.
\eeq
Therefore, if $\eta < c \sin \thetaS - 2c \kappa (\Clref{sep}+1) \eps - 2 \zeta$, then any pair of points indexed by $i, j \in I_\star$ from different clusters are {\em not} direct neighbors in the graph built by \algref{cov}.

In summary, we would like to choose $\eta$ such that
\[
2 c \kappa \eps + 2\zeta < \eta < c\sin \thetaS - 2 c\kappa (\Clref{sep}+1) \eps - 2 \zeta.
\]
This holds when
\[
2 c \kappa \eps + 2\zeta < \eta < \frac{c \sin \thetaS}{\Clref{sep}+2},
\]
which is true when
\beq \label{choice}
\eps < \frac{(d+2) \eta}{6 \kappa}, \quad t \le \frac\eta6, \quad \rad \le \frac\eta{6\Clref{U-approx} \kappa}, \quad \eta < \frac{\sin \thetaS}{(\Clref{sep}+2)(d+2)},
\eeq
using the definition of $\zeta$ in \eqref{zeta} and that of $c = 1/(d+2)$.
We choose $t = \eta/6$ and get that $\P(\Omega^c) \le C n \exp(- n \rad^d \eta^2 / C)$, where $C$ depends only on $d$ and $\Clref{N-size}$.

\subsubsection{The different clusters are in different connected components} \label{sec:different-cc}
We show that Step~3 in \algref{cov} eliminates all points $i \notin I_\star$, implying by our choice of parameters in \eqref{choice} that after that step the two clusters are not connected to each other in the graph.
Hence, take $i \notin I_\star$ with $K_i = 1$ (say), so that $\dist(\bs_i, S_2) \le \rad$.
By \lemref{sep}, we have $\dist(\bs_i, S_1 \cap S_2) \le \Clref{sep} \rad < 1/\kappa$.
Assuming that $(\Clref{sep} + 1) \rad < 1/\kappa$, let $\bs^0 = P_{S_1 \cap S_2}(\bs_i)$ and  define $T^0_k = T_{S_k}(\bs^0)$

Below, $C > 0$ is a constant whose value increases with each appearance.
By \lemref{cov-inter} (and the notation there), for $\bs \in S_1$ such that $\dist(\bs, S_2) \le (\Clref{sep} + 1) \rad$,
\[\|\bC(\bs) - \bSigma(\bs)\| \le C \, \rad^{3}.\]
We now derive another approximation that involves $\bSigma^0(\bs)$, the covariance matrix of the uniform distribution on $B(P_{T_1^0}(\bs), \rad) \cap (T_1^0 \cup T_2^0)$.
For that, we continue with the notation used in the proof of \lemref{cov-inter} until \eqref{2nd-approx} below.
Define $\bt_1 = P_{T_1^0}(\bs)$ and $\bt_2 = P_{T^0_2}(\bt_1)$.
Let $\delta_0 = \|\bt_1 - \bt_2\|$, $\delta_1 = \|\bs - \bt_1\|$, $\delta_2 = \|\bs_2 - \bt_2\|$ and $A_k^0 = T_k^0 \cap B(\bt_1,\rad)$.
By \lemref{S-approx}, we have $\delta_1 \le C \rad^2$ and $\delta_2 \le C \rad^2$, because $\|\bs - \bs^0\| \le C \rad$ by \lemref{sep}, and
\[\|\bs_2 - \bs^0\| \le \|\bs_2 - \bs\| + \|\bs - \bs^0\|.\]
Hence, $|\delta_0 - \delta| \le \delta_1 + \delta_2 \le C \rad^2.$
We assume that $\rad$ is small enough that $C \rad^2 < \rad$, so that $A_1^0 \ne \emptyset$.
Note that $\E(\lambda_{A^0_k}) = \bt_k$ and $\bD_1^0 := \Cov(\lambda_{A_1^0}) = c \rad^2 P_{T_1^0}$, while $\bD_2^0 := \Cov(\lambda_{A_2^0}) = c (\rad^2 -\delta_0^2) P_{T_2^0}$  when $\delta_0 \le \rad$;  otherwise $A_2^0 = \emptyset$.
As in \eqref{Sigma}, we have
\beq \label{C0}
\bSigma^0(\bs) = \alpha^0 \bD_1^0 + (1- \alpha^0) \bD_2^0 + \alpha^0 (1-\alpha^0) (\bt_1 -\bt_2)(\bt_1 -\bt_2)^\top,
\eeq
where
\[
\alpha^0 := \frac{\vol(A_1^0)}{\vol(A_1^0) + \vol(A_2^0)}.
\]
This identity remains valid even when $A_2^0 = \emptyset$.
As in the proof of \lemref{cov-inter}, we have
\[\|\bSigma(\bs) - \bSigma^0(\bs)\| \le (2 c +1) \rad^2 |\alpha' - \alpha^0| \\
+ \ \|\bD_1' - \bD_1^0\| \vee \|\bD_2' - \bD_2^0\| + 2 \rad \big(\|\bt_1 - \bs\| \vee \|\bt_2 - \bs_2\|\big).\]
By the triangle inequality and the fact that $\|P_T\| \le 1$ for any subspace $T$,
\[
\|\bD_1' - \bD_1^0\| \le c \rad^2 \|P_{T_1} - P_{T^0_1}\|,
\]
and
\[
\|\bD_2' - \bD_2^0\| \le c \rad^2 \|P_{T_2} - P_{T_2^0}\| + c |\delta^2 - {\delta_0}^2|.
\]
By \lemref{T-diff} and \lemref{P-diff}, we have
\[
\|P_{T_1} - P_{T_1^0}\| \le \kappa \|\bs - \bs^0\| \le C \rad, \qquad \|P_{T_2} - P_{T_2^0}\| \le \kappa \|\bs_2 - \bs^0\| \le C \rad.
\]
And since $|\delta^2 - {\delta_0}^2| \le 2 \rad |\delta - {\delta_0}| \le C \rad^3$, we have
$\|\bD_k' - \bD_k^0\| \le C \rad^3$ for $k=1,2$.
Let $\omega_d$ denote the volume of the $d$-dimensional unit ball.
Then
\[
\vol(A_1') = \omega_d \rad^d, \quad \vol(A_2') = \omega_d (\rad^2 -\delta^2)^{d/2},
\quad \vol(A_1^0) = \omega_d \rad^2, \quad \vol(A_2^0) = \omega_d (\rad^2 -{\delta_0}^2)_+^{d/2},
\]
so that
\beqn
|\alpha' - \alpha^0|
&=& \big| \frac1{1 + (1 - \delta/\rad)^{d/2}} - \frac1{1 + (1 - \delta_0/\rad)_+^{d/2}}\big| \\
&\le& \big|(1 - \delta/\rad)^{d/2} - (1 - \delta_0/\rad)_+^{d/2}\big|.
\eeqn
Proceeding exactly as when we bounded $\zeta$ in the proof of \lemref{cov-cont}, we get
\[
|\alpha' - \alpha^0| \le d \sqrt{|\delta - \delta_0|/\rad} \le C \sqrt{\rad}.
\]
Hence, we proved that
\[\|\bSigma(\bs) - \bSigma^0(\bs)\| \le C\rad^{5/2}.\]
We conclude with the triangle inequality that
\beq \label{2nd-approx}
\|\bC(\bs) - \bSigma^0(\bs)\| \le \|\bC(\bs) - \bSigma(\bs)\| + \|\bSigma(\bs) - \bSigma^0(\bs)\| \le C \rad^{5/2}.
\eeq
Again, this holds for any $\bs \in S_1$ such that $\|\bs-\bs^0\| \le (\Clref{sep} + 1) \rad$.

Assuming that $\bs_i \ne \bs^0$ (which is true with probability one) and $\bs^0 = 0$, let $h = \|\bs_i - \bs^0\|$ and $\bv = (\bs_i - \bs^0)/h$.
Note that $\bs_i = h \bv$.
Because $\bv \perp T_1^0 \cap T_2^0$, and that $\thetamin(T_1^0, T_2^0) \ge \thetaS$, we apply \lemref{intersect} with scaling to find $\ell \in h \pm \rad/2$ such that $\|\bSigma^0(\ell \bv) - \bSigma^0(h \bv)\| \ge \rad^2 \Clref{intersect}$, where $\Clref{intersect} > 0$ depends only on $\thetaS$ and $d$.
Letting $\tilde{\bs} = \ell \bv$, we have $\|\tilde{\bs} - \bs_i\| = |h -\ell| \le \rad/2$, so that
\[\dist(\tilde{\bs}, S_1 \cap S_2) \le \dist(\bs_i, S_1 \cap S_2) + \rad/2 < (\Clref{sep} + 1/2) \rad < 1/\kappa,\]
and consequently, $P_{S_1 \cap S_2}(\tilde{\bs}) = \bs^0$, by \citep[Th 4.8(12)]{MR0110078}.
Hence, by the triangle inequality,
\beqn
\|\bC(\bs_i) - \bC(\tilde{\bs})\|
&\ge& \|\bSigma^0(\bs) - \bSigma^0(\tilde{\bs})\| - \|\bC(\bs_i) - \bSigma^0(\bs_i)\| - \|\bC(\tilde{\bs}) - \bSigma^0(\tilde{\bs})\| \\
&\ge& \rad^2/\Clref{intersect} - 2 C \rad^{5/2}.
\eeqn
We apply the same arguments but now coupled with \lemref{cov-cont} after applying a proper scaling, to get
\beqn
\|\bC(\bar{\bs}) - \bC(\tilde{\bs})\|
&\le& \|\bSigma^0(\bar{\bs}) - \bSigma^0(\tilde{\bs})\| + \|\bC(\bar{\bs}) - \bSigma^0(\bar{\bs})\| + \|\bC(\tilde{\bs}) - \bSigma^0(\tilde{\bs})\| \\
&\le& 5d \rad^{3/2} \sqrt{\|\bar{\bs} - \tilde{\bs}\|} + 2 C \rad^{5/2},
\eeqn
for any $\bar{\bs} \in S_1$ such that $\|\bar{\bs} - \tilde{\bs}\| \le \rad/2$, since this implies that $\|\bar{\bs}-\bs^0\| \le (\Clref{sep} + 1) \rad$.
Hence, when $\|\bar{\bs} - \tilde{\bs}\| \le \rad/C_\Omega$ and $\rad^{1/2} \le 1/(16 C \Clref{intersect})$, we have
\[
\|\bC(\bar{\bs}) - \bC(\bs_i)\| \ge \rad^2 \big(1/\Clref{intersect} - 5d \sqrt{1/C_\Omega} - 4 C \rad^{1/2}\big) \ge \rad^2/(4\Clref{intersect}).
\]
Now, under $\Omega$, there is $\bs_j \in S_1 \cap B(\bar{\bs}, \rad/C_\Omega)$, so that $\|\bC_j - \bC_i\| \ge \rad^2/(4\Clref{intersect})$.
Therefore, choosing $\eta$ such that $\eta < 1/(4\Clref{intersect})$, we see that $\|\bC_j - \bC_i\| > \eta$, even though $\|\bs_j - \bs_i\| \le \rad/2 + \rad/C_\Omega \le \rad$.

\subsubsection{Each cluster forms a connected component in the graph}  \label{sec:same-cc}
We show that the points that survive Step~3 and belong to the same cluster form a connected component in the graph, except for possible spurious points near the intersection.
Take, for example, the cluster generated from sampling $S_1$.  The danger is that Step~3 created a ``crevice" within this cluster wide enough to disconnect it.  We show this is not the case.  (Note that the crevice may be made of several disconnected pieces.)
Before we start, we recall that $I_\star^c$ was eliminated in Step~3, so that by our choice of $\eta$ in \eqref{choice}, to show that two points $\bs_i, \bs_j$ sampled from $S_1$ are neighbors it suffices to show that $\|\bs_i - \bs_j\| \le \eps$.

We first bound the width of the crevice.
Let $I_\circ = \{i \in I_\star: \Xi_i \subset I_\star\}$.
By our choice of parameters in \eqref{choice}, we see that $i \in I_\circ$ is neighbor with any $j \in \Xi_i$, so that $i$ survives Step~3.
Hence, the nodes removed at Step~3 are in $I_\ddag := \{i: \Xi_i \cap I_\star^c \ne \emptyset\}$, with the possibility that some nodes in $I_\ddag$ survive.
Now, for any $i \notin I_\star$, there is $j$ with $K_{j} \ne K_i$ such that $\|\bs_i - \bs_j\| \le \rad$, so by \lemref{sep},
\[
\dist(\bs_i, S_1 \cap S_2) \le \Clref{sep} \|\bs_i - \bs_j\| \le \Clref{sep} \rad.
\]
By the triangle inequality, this implies that $\dist(\bs_i, S_1 \cap S_2) \le \rad_1 := (\Clref{sep}+1) \rad$ for all $i \in I_\ddag$.
So the crevice is along $S_1 \cap S_2$ and of width bounded by $\rad_1$.
We will require that $\eps$ is sufficiently larger than $\rad_1$, which intuitive will ensure that the crevice is not wide enough to disconnect the subgraph corresponding to $I_\circ$.
Let $R := S_1 \setminus B(S_1 \cap S_2, \rad_2)$, $\rad_2 = \rad_1 + \rad = (\Clref{sep}+2) \rad$, so that any $\bs_i \in S_1$ such that $\dist(\bs_i, R) < \rad$ survives Step~3.

Take two adjacent connected components of $R$, denoted $R_1$ and $R_2$.
We show that there is at least one pair $j_1, j_2$ of direct neighbors in the graph such that $\bs_{j_m} \in R_m$.
Take $\bs$ on the connected component of $S_1 \cap S_2$ separating $R_1$ and $R_2$.
Let $T^k = T_{S_k}(\bs)$ and let $H$ be the affine subspace parallel to $(T^1 \cap T^2)^\perp$ passing through $\bs$.
Take $\bt^m \in P_{T^1}(R_m) \cap H \cap \partial B(\bs, \eps_1)$, where $\eps_1 := \eps/2$, and define $\bs^m = P_{T^1}^{-1}(\bt^m)$.
Note that here $\bt^1, \bt^2 \in T^1$ and $\bs^1, \bs^2 \in S_1$, and by \citep[Th 4.8(12)]{MR0110078}, $P_{S_1 \cap S_2}(\bt^m) = \bs$.
\lemref{proj} not only justifies this construction when $\kappa \eps_1 < 1/3$, it also says that $P_{T^1}^{-1}$ has Lipschitz constant bounded by $1+\kappa \eps_1$, which implies that
\[\|\bs^m - \bs\| \le (1 + \kappa \eps_1) \|\bt^m -\bs\| = (1 + \kappa \eps_1) \eps_1 \le \eps/3,\]
when $\eps$ is sufficiently small.
We also have
\beqn
\dist(\bs^m, S_1 \cap S_2) &\ge& \dist(\bt^m, S_1 \cap S_2) - \|\bs^m - \bt^m\| \\
&=& \|\bt^m - \bs\| - \|\bs^m - \bt^m\| \\
&\ge& \eps_1 - \frac\kappa2 \|\bs^m - \bs\|^2 \\
&\ge& \big(1 - \frac\kappa2(1 + \kappa \eps_1)^2 \eps_1\big) \eps_1 \\
&\ge& \eps/3,
\eeqn
when $\eps$ is sufficiently small.
We used \eqref{S-approx1} in the second inequality.
We assume $\rad/\eps$ is sufficiently small that $\eps/3 \ge \rad_2 + \rad$,
Then under $\Omega_1$, there are $j_1, j_2$ such that $\bs_{j_m} \in B(\bs^m, \rad) \cap S_1$.  By the triangle inequality, we then have that $\dist(\bs_{j_m}, S_1 \cap S_2) \ge \eps/3 - \rad \ge \rad_2$, so that $\bs_{j_m} \in R_m$, and
\beqn
\|\bs_{j_1} - \bs_{j_2}\| &\le& \|\bs_{j_1} - \bs^1\| + \|\bs^1 - \bs\| + \|\bs - \bs^2\| + \|\bs^2 - \bs_{j_2}\| \\
&\le& \rad + \eps/3 + \eps/3 + \rad \\
&=& \frac23 \eps + 2 \rad \le \eps,
\eeqn
because $6 \rad \le 3 (\rad + \rad_1) \le \eps$ by assumption.

Now, we show that the points sampled from $R_1$ form a connected subgraph.  ($R_1$ is any connected component of $R$.)
Take $\bs^1, \dots, \bs^M$ an $\rad$-packing of $R_1$, so that
\[\bigcup_m (R_1 \cap B(\bs^m, \rad/2)) \subset R_1 \subset \bigcup_m (R_1 \cap B(\bs^m, \rad)).\]
Because $R_1$ is connected, $\cup_m B(\bs^m, \rad)$ is necessarily connected.  Under $\Omega_1$, and $\Clref{cov-inter} \ge 2$, all the balls $B(\bs^m, \rad), m=1, \dots, M,$ contain at least one $\bs_i \in S_1$, and any such point survives Step~3 since $\dist(\bs_i, R_1) < \rad$ by the triangle inequality.
Two points $\bs_i$ and $\bs_j$ such that $\bs_i,\bs_j \in B(\bs^m, \rad)$ are connected, since $\|\bs_i - \bs_j\| \le 2 \rad \le \eps$.
And when $B(\bs^{m_1}, \rad) \cap B(\bs^{m_2}, \rad) \ne \emptyset$, $\bs_i \in B(\bs^{m_1}, \rad)$ and $\bs_j \in B(\bs^{m_2}, \rad)$ are such that $\|\bs_i - \bs_j\| \le 4 \rad \le \eps$.
Hence, the points sampled from $R_1$ are connected in the graph under $\Omega_1$.

%
We conclude that the nodes corresponding to $R$ that survive Step~3 are connected in the graph.

\subsubsection{Choice of parameters}
Aside from the constraints displayed in \eqref{choice}, we assumed in addition that
\[
\rad < \frac1{\Clref{N-size} \kappa}, \quad \rad < \frac1{\Clref{sep} \kappa + 1}, \quad \eta < \frac1{4\Clref{intersect}}, \quad 3 (\Clref{sep}+3) \rad \le \eps, \quad \eps < \frac7{8 \kappa},
\]
and that $\eps$ was sufficiently small.
Therefore, it suffices to choose the parameters as in \eqref{main} with a large-enough constant.


\subsection{Performance guarantees for \algref{proj}}

We keep the same notation and go a little faster here as the arguments are parallel.
Let $d_i$ denote the estimated dimensionality at point $\bs_i$, meaning the number of eigenvalues of $\bC_i$ exceeding $\sqrt{\eta} \, \|\bC_i\|$.
Recall that $\bQ_i$ denotes the orthogonal projection onto the top $d_i$ eigenvectors of $\bC_i$.
The arguments hinge on showing that, under $\Omega$, $d_i = d$ for all $i \in I_\star$ and that $d_i > d$ for $i$ such that $\dist(\bs_i, S_1 \cap S_2) \le \rad/C$, for some constant $C >0$.

Take $i\in I_\star$.
Under $\Omega$, \eqref{keybound} holds, and applying Weyl's inequality \citep[Cor.~IV.4.9]{MR1061154}, we have
\[
|\beta_m(\bC_i) - \beta_m(\bSigma_i) | \le \zeta \rad^2, \quad \forall m =1,\dots, D.
\]
By \lemref{unif-cov}, $\bSigma_i = c \rad^2 P_{T_i}$, so that $\beta_m(\bSigma_i) = c \rad^2$ when $m \le d$ and $\beta_m(\bSigma_i) = 0$ when $m > d$.
Hence,
\[
\beta_1(\bC_i) \le (c +\zeta) \rad^2, \qquad \beta_d(\bC_i) \ge (c -\zeta) \rad^2, \qquad \beta_{d+1}(\bC_i) \le \zeta \rad^2.
\]
This implies that
\[
\frac{\beta_d(\bC_i)}{\beta_1(\bC_i)} \ge \frac{c-\zeta}{c+\zeta} > \sqrt{\eta}, \qquad \frac{\beta_{d+1}(\bC_i)}{\beta_1(\bC_i)} \le \frac{\zeta}{c+\zeta} < \sqrt{\eta},
\]
when $\zeta \le \eta/2$ as in \eqref{choice} and $\eta$ is sufficiently small.
When this is so, $d_i = d$ by definition of $d_i$.

Note that the top $d$ eigenvectors of $\bSigma_i$ generate $T_i$.
Hence, applying the Davis-Kahan theorem, stated in \lemref{davis}, and \eqref{keybound} again, we get that
\[
\|\bQ_i - P_{T_i}\| \le \frac{\sqrt{2} \, \zeta \rad^2}{c \rad^2} = \zeta' := \sqrt{2} (d+2) \zeta, \quad \forall i \in I_\star.
\]
This is the equivalent of \eqref{keybound}, which leads to the equivalent of \eqref{decisive}:
\[
\|\bQ_i - \bQ_j\| \le \frac1{c\rad^2} \|\bSigma_i - \bSigma_j\| + 2\zeta', \qquad \forall i,j \in I_\star,
\]
using the fact that $\bSigma_i = c \rad^2 P_{T_i}$.
When $i,j \in I_\star$ are such that $K_i = K_j$, based on \eqref{eta-choice1}, we have
\[
\|\bQ_i - \bQ_j\| \le 2 \kappa \eps + 2\zeta'.
\]
Hence, when $\eta > 2 \kappa \eps + 2\zeta'$, two nodes $i, j \in I_\star$ such that $K_i = K_j$ and $\|\bs_i - \bs_j\| \le \eps$ are neighbors in the graph.
The arguments provided in \secref{same-cc} now apply in exactly the same way to show that nodes $i \in I_\star$ such that $K_i = 1$ belong to a single connected component in the graph, except for possible nodes near the intersection.
The same is true of nodes $i \in I_\star$ such that $K_i = 2$.

Therefore, it remains to show that these two sets of nodes are not connected.
When we take $i, j \in I_\star$ such that $K_i \ne K_j$, we have the equivalent of \eqref{eta-choice2S}, meaning,
\[
\|\bQ_i - \bQ_j\| \ge \sin \thetaS - 2 \kappa (\Clref{sep}+1) \eps - 2 \zeta'.
\]
We choose $\eta$ smaller than the RHS, so that these nodes are not neighbors in the graph.

We next prove that a node $i \in I_\star$ is not neighbor to a node near the intersection because of different estimates for the local dimension.
Take $\bs \in S_1$ such that $\delta(\bs) := \dist(\bs, S_2) < \rad$.
We apply \lemref{cov-inter} and use the notation there until \eqref{delta} below, with the exception that we use $\bs^2 = P_{S_2}(\bs)$ and $T_{2(1)}$ to denote $T_{S_2}(\bs^2)$.
Together with Weyl's inequality, we have
\[
\beta_{d+1}(\bC(\bs)) \ge \beta_{d+1}(\bSigma(\bs)) -  \Clref{cov-inter} \rad^{3}, \qquad \beta_{1}(\bC(\bs)) \le \beta_1(\bSigma(\bs)) +  \Clref{cov-inter} \rad^{3},
\]
which together with \lemref{eigen-inter} (and proper scaling), implies that
\[
\frac{\beta_{d+1}(\bC(\bs))}{\beta_1(\bC(\bs))} \ge \frac{\frac{c}{8} (1-\cos \thetamax(T_1,T_{2(1)}))^2 (1 - (\delta(\bs)/\rad)^2)_+^{d/2+1} -\Clref{cov-inter} \rad^{3}}{c + (\delta(\bs)/\rad) (1 - (\delta(\bs)/\rad)^2)_+^{d/2} + \Clref{cov-inter} \rad^{3}}.
\]
Define $\bs^0$, $T_1^0$ and $T_2^0$ as in \secref{different-cc}.
Then, by the triangle inequality,
\[
\thetamax(T_1,T_{2(1)}) \ge \thetamax(T_1^0,T_2^0) - \thetamax(T_1,T_1^0) -\thetamax(T_{2(1)},T_2^0).
\]
By definition, $\thetamax(T_1^0,T_2^0) \ge \thetaS$, and by \lemref{T-diff},
\[
\thetamax(T_1,T_1^0) \le 2 \asin\left(1 \wedge \frac\kappa2 \|\bs - \bs^0\|\right) \le C \rad,
\]
and similarly,
\[
\thetamax(T_{2(1)},T_2^0) \le 2 \asin\left(1 \wedge \frac\kappa2 \|\bs^2 - \bs^0\| \right) \le C \rad,
\]
because $\|\bs - \bs^0\| \le C \rad$ and $\|\bs^2 - \bs^0\| \le \rad + \|\bs - \bs^0\| \le C \rad$.
Hence, for $\rad$ small enough, $\thetamax(T_1,T_{2(1)}) \ge \thetaS/2$, and furthermore,
\beq \label{delta}
\frac{\beta_{d+1}(\bC(\bs))}{\beta_1(\bC(\bs))} \ge \sqrt{\eta} \quad \text{ when } \quad 1 - \frac{\delta(\bs)^2}{\rad^2} \ge \xi^{2/(d+2)},
\eeq
where
\[
\xi := \frac{9 (1 + 1/c) \sqrt{\eta}}{(1-\cos (\thetaS/2))^2},
\]
by the fact that $\eta \ge \rad$ in \eqref{choice}.
The same is true for points on $\bs \in S_2$ if we redefine $\delta(\bs) = \dist(\bs, S_1)$.
Hence, for $\bs_i$ close enough to the intersection that $\delta(\bs_i)$ satisfies \eqref{delta}, $d_i > d$.
Then, by \lemref{P-diff}, $\|\bQ_i - \bQ_j\| = 1$ for any $j \in I_\star$.
By our choice of $\eta < 1$, this means that $i$ and $j$ are not neighbors.

So the only way $\{i \in I_\star : K_i = 1\}$ and $\{i \in I_\star : K_i = 2\}$ are connected in the graph is if there are $\bs_i \in S_1$ and $\bs_j \in S_2$ such that $\|\bs_i - \bs_j\| \le \eps$ and both $\delta(\bs_i)$ and $\delta(\bs_j)$ fail to satisfy \eqref{delta}.
We now show this is not possible.
By \lemref{cov-inter}, we have
\[
\|\bC_i - \bSigma_i\| \le \Clref{cov-inter} \, \rad^{3}.
\]
By \eqref{Sigma} (and using the corresponding notation) and the triangle inequality
\beqn
\|\bSigma_i - \alpha_i c \rad^2 P_{T_i}\| &\le& c (1-\alpha_i) t_i^2 \rad^2 + \alpha_i (1-\alpha_i) \delta^2(\bs_i) \le 2 (1 - \alpha_i) \rad^2 \\
&\le& 2 (1 - (\delta(\bs_i)/\rad)^2)_+^{d/2} \rad^2  \le 2 \xi^{d/(d+2)} \rad^2 ,
\eeqn
where the very last inequality comes from $\delta(\bs_i)$ not satisfying \eqref{delta}.
Hence,
\[
\|\bC_i - \alpha_i c \rad^2 P_{T_i}\| \le 2 \xi^{d/(d+2)} \rad^2 + \Clref{cov-inter} \, \rad^{3},
\]
and since the $\beta_{d+1}(P_{T_i}) = 0$, by the Davis-Kahan theorem, we have
\[
\|\bQ_i - P_{T_i}\| \le \frac1{\alpha_i c \rad^2} \big[\xi^{d/(d+2)} \rad^2 + \Clref{cov-inter} \, \rad^{3} \big] \le C (\xi^{d/(d+2)} + \rad),
\]
and similarly,
\[
\|\bQ_j - P_{T_j}\| \le C (\xi^{d/(d+2)} + \rad).
\]
By \lemref{P-diff}, $\|P_{T_i} - P_{T_j}\| = \sin \thetamax(T_i, T_j)$.
Let $\bs^0 = P_{S_1 \cap S_2}(\bs_i)$, and define $T_1^0$ and $T_2^0$ as before.
We have
\[
\thetamax(T_i, T_j) \ge \thetamax(T_1^0, T_2^0) - \thetamax(T_i, T_1^0) - \thetamax(T_j, T_2^0) \ge \thetaS - C \eps,
\]
calling in \lemref{T-diff} as before, coupled with the fact that $\|\bs_i - \bs^0\| \le C \eps$ and $\|\bs_j - \bs^0\| \le C \eps$, since $\dist(\bs_i, S_2) \le \|\bs_i - \bs_j\| \le \eps$ and \lemref{sep} applies, and then $\|\bs_j - \bs^0\| \le \|\bs_i - \bs^0\| + \|\bs_j - \bs_i\|$.
Hence, assuming $\eps$ is small enough,
\beqn
\|\bQ_i - \bQ_j\|
&\ge& \|P_{T_i} - P_{T_j}\| -\|\bQ_i - P_{T_i}\| - \|\bQ_j - P_{T_j}\| \\
&\ge& \sin(\thetaS/2) -C (\xi^{d/(d+2)} + \rad) > \eta,
\eeqn
when $\rad$ and $\eta$ (and therefore $\xi$) are small enough.
Therefore $i$ and $j$ are not neighbors, as we needed to show.

We conclude by remarking that, by choosing $C$ large enough in \eqref{main}, the resulting choice of parameters fits all our (often implicit) requirements.

\subsection{Noisy case} \label{sec:noise}

So far we only dealt with the case where $\tau = 0$ in \eqref{data-point}.
When $\tau > 0$, a sample point $\bx_i$ is in general different than its corresponding point $\bs_i$ sampled from one of the surfaces.
However, when $\tau/\rad$ is small, this does not change things much.
For one thing, the points are close to each other, since we have $\|\bx_i - \bs_i\| \le \tau$ by assumption, and $\tau$ is small compared to $\rad$.
And the corresponding covariance matrices are also close to each other.
To see this, redefine $\Xi_i = \{j \neq i: \bx_j \in N_\rad(\bx_i)\}$ and $\bC_i$ as the sample covariance of $\{\bx_j: j \in \Xi_i\}$.
Let $\bD_i$ denote the sample covariance of $\{\bs_j: j \in \Xi_i\}$.
Let $X$ be uniform over $\{\bx_j : j \in \Xi_i\}$ and define $Y = \sum_j \bs_j \1_{\{X = \bx_j\}}$.  As in \eqref{CovXY}, we have
\begin{eqnarray}
\|\bD_i - \bC_i\| &=& \|\Cov(X) - \Cov(Y)\| \notag \\
&\le& \E\big[\|X - Y\|^2\big]^{1/2} \cdot \left(\E\big[\|X - \bx_i\|^2\big]^{1/2}  + \E\big[\|Y - \bx_i\|^2\big]^{1/2}\right) \notag \\
&\le& \tau \cdot (\rad + \rad + \tau) = \rad^2 \big(2 \tau/\rad + (\tau/\rad)^2\big),  \label{cov-noise}
\end{eqnarray}
which is small compared to $\rad^2$, which is the operating scale for covariance matrices in our setting.

Using these facts, the arguments are virtually the same, except for some additional terms due to triangle inequalities, for example, $\|\bs_i - \bs_j\| - 2 \tau \le \|\bx_i - \bx_j\| \le \|\bs_i - \bs_j\| + 2 \tau$.
In particular, this results in $\zeta$ in \eqref{zeta} being now redefined as $\zeta = \frac{3\tau}\rad + t + \Clref{U-approx} \kappa \rad$.
We omit further technical details.

\subsection*{Acknowledgements}

We would like to thank Jan Rataj for hints leading to \lemref{T-diff}, which is much sharper than what we knew from \citep{1349695}.
We would also like to acknowledge support from the Institute for Mathematics and its Applications (IMA).
For one thing, the authors first learned about the research of \cite{goldberg2009multi} there, at the {\em Multi-Manifold Data Modeling and Applications} workshop in the Fall of 2008, and this was the main inspiration for our paper.
Also, part of our work was performed while TZ was a postdoctoral fellow at the IMA, and also while EAC and GL were visiting the IMA.
This work was partially supported by grants from the National Science Foundation (DMS-09-15160, DMS-09-15064, DMS-09-56072).
\bibliography{refs}
\bibliographystyle{chicago}

\end{document}